\documentclass[11pt]{article}
\usepackage{jmlr2e}

\usepackage{preamble}

\usepackage{lastpage}
\jmlrheading{26}{2025}{1-\pageref{LastPage}}{9/24; Revised 4/25}{7/25}{24-1523}{Etienne Boursier and Nicolas Flammarion}

\ShortHeadings{Early Alignment is a Two-Edged Sword}{Boursier and Flammarion}
\firstpageno{1}

\title{Early Alignment in Two-Layer Networks Training is a Two-Edged Sword}

\author{%
 \name{Etienne Boursier} \email{etienne.boursier@inria.fr}\\
 \addr Université Paris-Saclay, CNRS, Inria, Laboratoire de mathématiques d'Orsay, 91405, Orsay, France
 \AND
  \name{Nicolas Flammarion} \email{nicolas.flammarion@epfl.ch}\\
 \addr TML Lab, EPFL, Switzerland
}

\editor{Brian Kulis}

\begin{document}
\setlength{\abovedisplayskip}{5pt}
\setlength{\topsep}{0.5em}

\setcounter{tocdepth}{3}

\doparttoc 
\faketableofcontents 

\maketitle

\begin{abstract}
Training neural networks with first order optimisation methods is at the core of the empirical success of deep learning. 
The scale of initialisation is a crucial factor, as small initialisations are generally associated to a feature learning regime, for which gradient descent is implicitly biased towards \textit{simple} solutions. This work provides a general and quantitative description of the early alignment phase, originally introduced by Maennel et al. (2018). For small initialisation and one hidden ReLU layer networks, the early stage of the training dynamics leads to an alignment of the neurons towards key directions. This alignment induces a sparse representation of the network, which is directly related to the implicit bias of gradient flow  at convergence. 
This sparsity inducing alignment however comes at the expense of difficulties in minimising the training objective: we also provide a simple data example for which overparameterised networks fail to converge towards global minima and only converge to a spurious stationary point instead.
\end{abstract}

\begin{keywords}
Implicit Bias, Gradient Flow, ReLU Networks, Training Dynamics
\end{keywords}

\section{Introduction}\label{sec:intro}

Artificial neural networks are nowadays used in numerous applications \citep{he2016deep,jumper2021highly}. A part of their success originates from the ability of optimisation methods to find global minima, despite the non-convexity of the training losses; as well as the good generalisation performances obtained despite a large overparameterisation and an interpolation of the data \citep{zhang2021understanding,geiping2021stochastic,liu2020bad}. The understanding of these two generally admitted reasons yet remains very limited in the machine learning community. Recently, different lines of work advanced our comprehension of the empirical success of neural networks.
First, \citet{mei2018mean,chizat2018global,wojtowytsch2020convergence,rotskoff2022}
proved convergence of first order optimisation methods towards global minima of the training loss for idealised infinite width architectures.
On the other hand, benign overfitting occurs in many different statistical models, i.e., the learnt estimator yields a small generalisation error despite interpolating the training data \citep{belkin2018overfitting,bartlett2020benign,frei2022benign,tsigler2023benign}.

This surprisingly good generalisation performance of neural networks is often attributed to the \textit{implicit bias} of the used optimisation algorithms, that do select a specific global minimum of the objective function~\citep{neyshabur2014search,zhang2021understanding}. For linear neural networks, implicit bias has been thoroughly characterised \citep{ji2018gradient,arora2019fine,yun2021unifying,min2021on,varre2023spectral}. In the presence of non-linear activations, e.g., ReLU, the implicit bias is much harder to characterise \citep{vardi2021implicit}. In the classification setting, the learnt estimator is proportional to the min-norm margin classifier \citep{lyu2019gradient,chizat2020implicit}. It is yet much more unclear in the regression case, for which it has only been characterised for specific data examples. In particular, \citet{boursier2022gradient} suggest that gradient flow is biased towards minimal norm interpolators; while other works suggest it induces sparsity in the representation of the network, in the sense it could be represented by a small number of neurons \citep{shevchenko2022mean,safran2022effective,chistikov2023learning}. Although different \citep{chistikov2023learning}, both notions of minimal norm and sparsity seem closely related \citep{parhi2021banach,stewart2022regression,boursier2023penalising}.
A similar implicit bias towards low rank solutions has been conjectured for matrix factorisation \citep{gunasekar2017implicit,arora2019implicit,razin2020implicit}.

\medskip

Due to the non-convexity of the considered loss, the convergence point of training crucially depends on the choice of initialisation. Large initialisation is known to lead to the Neural Tangent Kernel (NTK) regime, for which gradient descent provably converges exponentially towards a global minimum of the training loss~\citep{jacot2018neural,du2018gradient,arora2019fine}. 
Unfortunately, this regime is also associated with lazy training, where the weights parameters only slightly change~\citep{chizat2019lazy}. As a consequence, features are not learnt during training and can lead to a poor generalisation performance \citep{arora2019fine}.

On the other hand, smaller initialisations, such as in the mean field regime, yield a favorable implicit bias~\citep{chizat2020implicit,jacot2021saddle,boursier2022gradient}, while still having some convergence guarantees towards global minima~\citep{chizat2018global,wojtowytsch2020convergence}. However, this regime is more intricate to analyse and still lacks strong results on both convergence guarantees and implicit bias. In this objective, a recent part of the literature focuses on a complete description of the training dynamics, for both classification and regression, with specific data assumptions such as orthogonally separable~\citep{phuong2020inductive,wang2021convex}, symmetric linearly separable~\citep{lyu2021gradient}, orthogonal~\citep{boursier2022gradient}, positively correlated~\citep{wang2023understanding,chistikov2023learning,min2023early} or XOR-type data~\citep{glasgow2024sgd}.

In particular, these works rely on a first early alignment phase, during which the neurons' weights all align towards a few key directions, while remaining small in norm and having no or little impact on the estimator prediction. In their seminal paper, \citet{maennel2018gradient} first described this key phenomenon, providing general heuristics for infinitesimal initialisation scale.
This phase, specific to small initialisation and homogeneous activations, thus already induces some sparsity in the directions represented by the network and seems key to a final implicit bias with similar sparsity induced properties. 
This early alignment is not specific to the one hidden layer \citep{jacot2021saddle} and has been empirically observed with more complex architectures and real data~\citep{gur2018gradient,atanasov2021neural,ranadive2023special}.

\paragraph{Contributions.} Our contribution is twofold. First, we characterise the early alignment phenomenon for small initialisation, one hidden (leaky) ReLU layer networks trained with gradient flow in a general setup covering both classification and regression. 
A general study, i.e., holding for general datasets, of the so-called alignment has only been proposed by \citet{maennel2018gradient}. This previous study however fails at precisely quantifying this alignment, since it follows heuristic arguments, for infinitely small initialisations. In opposition, we provide a finite time, macroscopic initialisation and rigorous analysis of the early alignment phase. As a consequence, our result (\cref{thm:alignment}) can be directly applied to numerous previous works characterising the training dynamics of one hidden layer neural networks, to describe their first phase of dynamics. 

Second, we apply this result to analyse the complete dynamics of training on a specific data example, for which gradient flow converges towards a spurious stationary point. In particular, our result implies that for small initialisation scales, interpolation might not happen at the end of training, even with an infinite number of neurons and infinite training time. This failure of convergence for largely overparameterised networks is suprising, as it goes in the opposite direction of previous works~\citep{chizat2018global,wojtowytsch2020convergence}. This negative result highlights the importance of weights' omnidirectionality to reach global minima, a property that can be lost by early alignment for non-differentiable activations.

Overall, our work provides a general description of the early alignment phenomenon for small initialisations. This early phase presents clear benefits as it induces some sparsity of the network representation, that can be preserved along the whole trajectory. However, this benefit also comes at the expense of minimising the training loss, even on relatively simple datasets.

\medskip

The concurrent works of \citet{kumar2024directional,tsoy2024simplicity} also provide a mathematical description of the early alignment, with the main differences that these results i) hold for a different set of assumptions (e.g. smooth activations or assuming a unique solution of the gradient flow); ii) do not provide a quantitative bound on the initialisation scale at which early alignment happens. As a consequence of the second point, their results do not hold in the limit of infinite width network. In contrast, our main \cref{thm:alignment} holds for this \textit{mean field} limit, which is a key feature of our no-convergence result in \cref{thm:noconvergence}.

\section{Setting}\label{sec:setting}

We consider $n$ data points $(x_k,y_k)_{k\in[n]}$ with features $x_k\in\R^d$ and labels $y_k\in\R$. We also denote by $X = [x_1, \ldots, x_n]^{\top}\in \R^{n\times d}$ the matrix whose rows are given by the input vectors.
A two-layer neural network is parameterised by $\theta = \left(w_j, a_j\right)_{j\in[m]}\in\R^{m\times(d+1)}$, corresponding to the estimated function
\begin{equation}\label{eq:param}
h_{\theta}:x \mapsto \sum_{j=1}^m a_j \sigma(\langle w_j,x\rangle),
\end{equation}
where $\sigma$ is the (leaky) ReLU activation defined as $\sigma(x)\coloneqq\max(x,\gamma x)$ with $\gamma\in[0,1]$.
The parameters $w_j$ and $a_j$ respectively account for the hidden and output layer of the network. 
Note that \cref{eq:param} does not account for any bias term, as a simple reparameterisation of the features $\tilde{x}=(x,1)$ allows to do so. 

Training aims at minimising the empirical loss over the training dataset defined, for some loss function $\ell:\R^2\to\R_+$, by
\begin{equation*}\label{eq:loss}
L(\theta) \coloneqq \frac{1}{n}\sum_{k=1}^n \ell(h_{\theta}(x_k),y_k).
\end{equation*}
As the limiting dynamics of gradient descent with infinitesimal learning rate, we study a solution of the following differential inclusion, for almost any $t\in\R_+$,
\begin{gather}\label{eq:ODE1}
\frac{\df \theta^t}{\df t} \in -{\partial}_{\theta}L(\theta^t),
\end{gather}
where $\partial_{\theta}L(\theta)$ is the Clarke subdifferential of $L$ at $\theta$ \citep{clarke1990optimization}. Although the loss is not differentiable, the chain rule can still be applied for ReLU networks \citep{bolte2020mathematical, bolte2021conservative}, which is crucial to our analysis.

In the following, we consider a general loss function with minimal properties, covering both classical choices of square loss, $\ell(\hat{y},y)=(\hat{y}-y)^2$, and logistic loss with binary labels, $\ell(\hat{y},y)=\ln(1+e^{-\hat{y}y})$.
\begin{assumption}\label{ass:loss}
For any $y\in\R$, the function $\hat{y}\mapsto\ell(\hat{y},y)$ is differentiable. Moreover its derivative, denoted by $\partial_1 \ell(\cdot,y)$, is $1$-Lipschitz and verifies $\partial_1 \ell(0,y)\neq 0$ for any $y\neq 0$.
\end{assumption}
%
%
The existence of a global solution to \cref{eq:ODE1} is guaranteed \citep[see, e.g.,][Chapter 2]{aubin2012differential}. However, such solutions can be non-unique since the loss function is not continuously differentiable in $\theta$. In the following, \textit{any} global solution of \cref{eq:ODE1} is considered.

\paragraph{Initialisation.}
The choice of initialisation is crucial when training two-layer neural networks, since the considered optimisation problem is non-convex. 
The $m$ neurons of the neural network are here initialised as
\begin{equation}\label{eq:initscaling}
(a_j^0, w_j^0) = \frac{\lambda}{\sqrt{m}} (\tilde{a}_j,\tilde{w}_j),
\end{equation}
where $\lambda>0$ is the scale of initialisation (independent of $m$) and $(\tilde{a}_j,\tilde{w}_j)$ are drawn i.i.d. such that almost surely\footnote{\cref{eq:initialisation} is stated for any $k$ (not just $k=m$), since we aim at stating results with constants that do not depend on the width $m$.}
\begin{gather}
|\tilde{a}_j| \geq \|\tilde{w}_j\| \text{ for any }j\in \N^*,\label{eq:balancedinit}\\
\text{and }\frac{1}{k}\sum_{j=1}^k \tilde{a}_j^2 \leq 1 \text{ for any }k\in\N^*.\label{eq:initialisation}
\end{gather}
The $\frac{1}{\sqrt{m}}$ factor in \cref{eq:initscaling} characterises the feature learning (or mean field) regime. In absence of this $\frac{1}{\sqrt{m}}$ term, no relevant feature would be learnt as it corresponds to the lazy regime \citep{chizat2019lazy}. 
Since the scale of $a_j^0$ can be controlled through $\lambda$, \cref{eq:initialisation} is compatible with any classical initialisation. In particular, it holds almost surely when $\tilde{a}_j$ is bounded and it holds with high probability if $\tilde{a}_j$ is sub-Gaussian. 
As an example, any initialisation where 
\begin{align*}
a_j^0\in\left\{-\frac{\lambda}{\sqrt{m}},\frac{\lambda}{\sqrt{m}}\right\}\quad\text{and}\quad
w_j^0 \sim \cU\left(B(0,\frac{\lambda}{\sqrt{m}})\right),
\end{align*}
satisfies the above description. 

\subsection{Notations}

We note $f(\lambda,t) = \bigO{g(\lambda,t)}$ if there exists a constant $c>0$ that \textbf{only depends on the dataset}\footnote{It is yet independent of the number of neurons $m$.} $(x_k,y_k)_{k\in[n]}$ such that $|f(\lambda,t)| \leq c g(\lambda,t)$ on the considered set for $\lambda$ and~$t$. Occasionally, we note $f(\lambda,t) = \mathcal{O}_{\alpha}(g(\lambda,t))$ if the constant $c>0$ depends on the dataset and an extra parameter $\alpha$.
Conversely, we note $f = \Omega(g)$ if $g = \bigO{f}$. 
We also note $f = \Theta(g)$ if both $f = \bigO{g}$ and $f = \Omega(g)$.

\medskip

Along the paper, the detailed proofs are postponed to the appendix, for sake of readability.

\section{Weight Alignment Phenomenon}
This section aims at precisely quantifying the early alignment phenomenon in the setting of \cref{sec:setting}. 
For each individual neuron, \cref{eq:ODE1} rewrites
\begin{equation}\label{eq:ODEs}
\frac{\df w_j^t}{\df t} \in a_j^t \D_j^t \quad \text{and} \quad
\frac{\df a_j^t}{\df t} \in \langle w_j^t, \D_j^{t}\rangle,
\end{equation}
\begin{gather*}
\text{where }\D_j^t = \D(w_j^t,\theta^t)\coloneqq \Big\{ -\frac{1}{n}\sum_{k=1}^n \eta_k \partial_1 \ell(h_{\theta^t}(x_k),y_k)x_k \ \Big|\ \forall k\in[n], \eta_k \begin{cases} =1 \text{  if }\langle w_j^t,x_k\rangle>0 \\
=\gamma \text{  if }\langle w_j^t,x_k\rangle<0\\
 \in [\gamma,1] \text{ otherwise}
\end{cases}\Big\}.
\end{gather*}
In the following, we also note by $D(w,\theta)$ the minimal norm subgradient, which is uniquely defined and happens to be of particular interest:
\begin{equation*}
D(w,\theta) = \argmin_{D\in\D(w,\theta)}\|D\|_2 \ .
\end{equation*}
Note that the set $\D_j^{t}$ only depends on the parameters through the estimated function $h_{\theta^t}$ and the activations of the neuron $A(w_j^t)$ where $A$ is defined by
\begin{equation*}
A : \begin{array}{l} \R^{d} \to \{-1,0,1\}^n \\w \mapsto\big(\sign(\langle w, x_k\rangle)\big)_{k\in[n]}\end{array},
\end{equation*}
where $\sign(0)\coloneqq 0$ by convention.
We thus also note in the following for any $u\in\{-1,0,1\}^n$: $\D_u \coloneqq \D(w,\mathbf{0}) \text{ for any }w\in A^{-1}(u)$, since its definition does not depend on the choice of $w$. If $u\not\in A(\R^d)$, we note $\D_u=\emptyset$ by convention.

Our whole analysis relies on a first well known result, corresponding to the balancedness property \citep{arora2019fine,boursier2022gradient}.
\begin{lem}[Balancedness]\label{lemma:balanced}
For any $j\in[m]$ and $t\geq0$, $(a_j^t)^2 - \|w_j^t\|^2 = (a_j^0)^2 - \|w_j^0\|^2$.
\end{lem}
Continuity of the neuron weights and \cref{eq:balancedinit} then ensure that the sign of $a_j^t$ remains constant during the whole training. We also make the following assumption on the data.
\begin{assumption}\label{ass:Dj}
The data points $(x_k,y_k)$ for any $k\in[n]$ are generated independently, following a distribution that is absolutely continuous with respect to the Lebesgue distribution on $\R^{d+1}$. 
\end{assumption}
\cref{ass:Dj} allows to avoid degenerate situations due to data. In particular, it is required to ensure the following \cref{lemma:Dj}, as well as \cref{lemma:alphamin,lemma:delta0,lemma:interior} in the appendix.
\begin{lem}\label{lemma:Dj}
If \cref{ass:loss,ass:Dj} hold, then almost surely, for any $u\in \{-1,0,1\}^n$ 
\begin{equation*}
\D_u \cap  \left(\partial\bar{A^{-1}(u)} \cup-\partial \bar{A^{-1}(u)}\right)=\emptyset \text{ or } \mathbf{0}\in \D_u,
\end{equation*}
where $\bar{A^{-1}(u)}$ denotes the closure of $A^{-1}(u)$ and $\partial \bar{A^{-1}(u)}$ is the boundary \textbf{of the manifold} $\bar{A^{-1}(u)}$. 
Also, any family $(\partial_1\ell(0,y_k) x_k)_{k}$ with at most $d$ vectors is linearly independent.
\end{lem}
The boundary of the manifold $\partial \bar{A^{-1}(u)}$, which is different from the topological boundary \citep[see, e.g.,][Section~22]{tu2011manifolds}, is here given by
\begin{equation*}\label{eq:boundary}
\partial \bar{A^{-1}(u)} = \Big\{ w\in \R^d \mid \forall k\in[n], \sign(\langle w,x_k \rangle) \begin{cases} \geq 0 \text{ if } u_k = 1\\
= 0 \text{ if } u_k =0\\
\leq 0 \text{ if } u_k=-1
\end{cases} \text{and } A(w)\neq u \Big\}.
\end{equation*}
\cref{lemma:Dj} proves useful in our analysis, as it ensures that all the neuron dynamics at the end of the early alignment phase occur in the interior of the manifolds $A^{-1}(u)$, enabling to control these neurons.
Following \citet{maennel2018gradient}, \cref{def:extremal} introduces extremal vectors, which are key to the early alignment.
\begin{defin}\label{def:extremal}
For any $u\in\{-1,0,1\}^n$, the vector $D\in\D_u$ is said \textbf{extremal} if $D\neq \mathbf{0}$ and $D \in -A^{-1}(u)\cup A^{-1}(u)$.
\end{defin}
Extremal vectors actually correspond to the critical points (up to rescaling) of the following function on the unit sphere
\begin{equation}
\label{eq:G}
G: \begin{array}{l} \bS_{d} \to \R \\w \mapsto\langle w,D(w,\mathbf{0})\rangle\end{array}.
\end{equation}
The function $G$ is piecewise linear. It is indeed linear on each activation cone $A^{-1}(u)\subset \R^d$. As a consequence, it has at most one critical point per cone. In general, the number of extremal vectors is even much smaller than the number of activation cones, since some of the cones do not include any critical point. 
Understanding the function $G$ is crucial, since it is at the core of the early alignment phase. 
\begin{lem}\label{lemma:extremal}
If \cref{ass:loss,ass:Dj} hold, there exists almost surely at least one extremal vector.
\end{lem}
In the early alignment phase, the parameters norm remains small so that $h_{\theta^t}\approx 0$. Meanwhile, the vectors $\frac{w_j^t}{a_j^t}$ approximately follow an ascending sub-gradient flow of $G$ on the unit ball of $\R^d$ (see \cref{eq:ODEalignsketch} in the proof sketch of \cref{thm:alignment}). 
Since this directional movement happens much faster than the norm growth of the parameters, the vectors $w_j^t$ end up being aligned in direction to the critical points of $G$, \ie the extremal vectors.
\cref{thm:alignment} precisely quantifies this phenomenon for every neuron satisfying \cref{cond:neurons2} below.
\begin{condition}\label{cond:neurons2}
The neuron $j\in[m]$ satisfies \cref{cond:neurons2} for $\alpha_0>0$ if both
\begin{enumerate}
\item $\langle D(w_j^0,\mathbf{0}),\frac{w_j^0}{a_j^0}\rangle>-\sqrt{1-\alpha_0^2}\| D(w_j^0,\mathbf{0})\|$;
\item for any $t\in \R_+$: \hspace{0.1cm} $w_j^t = \mathbf{0} \implies w_j^{t'} = \mathbf{0}\text{ for all } t'\geq t.$
\end{enumerate}
\end{condition}
The meaning and necessity of this individual neuron condition is discussed further in \cref{rem:neurons2a,rem:neurons2b} below.
\begin{thm}\label{thm:alignment}
If \cref{ass:loss,ass:Dj} hold and the function $G$ defined in \cref{eq:G} does not admit a saddle point, then the following holds for any constant $\varepsilon\in(0,\frac{1}{3})$, $\alpha_0>0$ and initialisation scale $\lambda < \lambda^*_{\alpha_0}$ where $\lambda_{\alpha_0}^*>0$ only depends\footnote{The exact value of $\lambda_{\alpha_0}^*$ is given by \cref{eq:lambdastaralpha} in \cref{app:alignment}.} on the data $(x_k,y_k)_k$, $\alpha_0$ and the activation parameter $\gamma$; with $D_{\max} \coloneqq \max\limits_{w\in\R^d}\|D(w,\mathbf{0})\|$ and $\tau\coloneqq -\frac{\varepsilon\ln(\lambda)}{D_{\max}}$,
\begin{enumerate}[itemsep=-0.5em, topsep=0.5em, label=(\roman*),leftmargin=1cm]
\item output weights do not grow large until $\tau$:
\begin{gather*}
\forall t \leq \tau,\forall j\in[m], |a_j^0|\lambda^{2\varepsilon} \leq |a_j^t| \leq |a_j^0|\lambda^{-2\varepsilon} \quad \text{ and }\quad \|w_j^t\|\leq |a_j^t|.
\end{gather*}
\item Moreover,  for any neuron $j$ satisfying \cref{cond:neurons2} for $\alpha_0$, $D(w_j^\tau,\mathbf{0})$ is either an extremal vector or $\mathbf{0}$, along which $w_j^\tau$ is aligned:
\begin{gather*}
\|D(w_j^\tau,\mathbf{0})\|\geq\langle D(w_j^\tau,\mathbf{0}),\frac{w_j^\tau}{a_j^\tau}\rangle\geq \|D(w_j^\tau,\mathbf{0})\| -\mathcal{O}_{\alpha_0}\left(\lambda^{\frac{\|D(w_j^\tau,\mathbf{0})\|}{D_{\max}}\varepsilon}\right),\\
\text{or}\quad \frac{w_j^{\tau}}{\|w_j^{\tau}\|} = -\frac{D(w_j^\tau,\mathbf{0})}{\|D(w_j^\tau,\mathbf{0})\|}.
\end{gather*}
Also, the direction towards which $w_j^{\tau}$ is aligned corresponds to a local maximum (resp. minimum) of $G$ if $a_j^0>0$ (resp. $a_j^0<0$).
\end{enumerate} 
\end{thm}
\cref{thm:alignment} describes for a small enough scale of initialisation the early alignment phase, which happens during a time of order $\varepsilon\ln(\frac{1}{\lambda})$ at the beginning of the training dynamics. First, the neurons all remain of small norm during this phase---while the term $\lambda^{-2\varepsilon}$ grows large as $\lambda \to 0$, the other term $|a_j^0|$ also scales in $\lambda$, making their product bounded and arbitrarily small as $\lambda \to 0$.

Second, neurons end up aligned towards a few key directions, given by extremal vectors. There are indeed few such directions: as a first observation, there is at most one extremal vector $D(w,\mathbf{0})$ per activation cone, and there are at most $\bigO{\min(3^{n},n^d)}$ such cones \citep[see e.g.,][Theorem 4]{cover1965geometrical}. 
In general, the number of extremal vectors is even much smaller. For example, studies describing the complete parameters dynamics \citep{phuong2020inductive,lyu2021gradient,boursier2022gradient,chistikov2023learning,min2023early,wang2023understanding}
 all count either one or two extremal vectors---with the exception of \citet{glasgow2024sgd}, where the population loss counts $4$ extremal vectors (we refer to \cref{app:glasgow} for more details about this fact). The subsequent work of \citet{boursier2024simplicity} even goes beyond, by studying the extremal vectors of $G$ when the number of data points grows to infinity. In particular, they showed that for some linear data model, there are only two extremal vectors for large number of training samples. 
A general understanding of the number of extremal vectors yet remains open. 
We believe this quantization of represented directions to be closely related to the implicit bias of first order optimisation methods and elaborate further on this aspect in \cref{sec:discussion}.

Note that some neurons $w_j^{\tau}$ are not aligned towards extremal vectors but instead have ${D(w_j^\tau,\mathbf{0})=\mathbf{0}}$. Thanks to the first point of \cref{lemma:interior} in \cref{app:alignment}, this means that the neuron is deactivated with all data points $x_k$. As a consequence, it does not move anymore during training and has no impact at all on the estimated function $h_{\theta^t}$, i.e., these neurons can be ignored after the early alignment phase.

\medskip

The complete proof of \cref{thm:alignment} is given in \cref{app:alignment} and is sketched below, at the end of this section. 
%
\begin{rem}\label{rem:neurons2a}
The first point of \cref{cond:neurons2} only bounds away from $-1$ the alignment of $w_j^t$ with $D(w_j^t,\mathbf{0})$. When $\alpha_0\to 0$, this covers all the neurons with probability $1$. However, when fixing $\alpha_0>0$ and let $m$ go to infinity, a (small) fraction of neurons does not satisfy this condition. These neurons are hard to control, as their alignment speed is arbitrarily slow at the beginning of the procedure (see \cref{eq:ODEalignsketch}): they can then take an arbitrarily large time before being aligned to some extremal vector.
\end{rem}
\begin{rem}\label{rem:neurons2b}
Neurons such that $w_i(t)=\mathbf{0}$ at some time can spontaneously leave $\mathbf{0}$ in a way that cannot be controlled---both in the time at which it leaves, and in the direction at which it does so---due to the multiplicity of subgradient flow solutions of \cref{eq:ODE1}. The second point in \cref{cond:neurons2} then restricts the analysis to natural solutions, by assuming that as soon as a neuron is $\mathbf{0}$, it does not move anymore. This is what generally happens in practice for ReLU activations, where we fix $\sigma'(0)=0$ in common implementations. Another way to ensure this point is to consider a balanced initialisation $|a_j^0|=\|w_j^0\|$. In that case, a simple Grönwall argument allows to show that both $a_j^t$ and $w_j^t$ never cancel, automatically guaranteeing the second point of \cref{cond:neurons2}. 
Adding weight decay to the optimisation scheme (with any choice of regularisation parameter) would also ensure the second point of \cref{cond:neurons2}.
\end{rem}

A significant assumption in \cref{thm:alignment} is that $G$ has no saddle point. While $G$ can have saddle points in complex data cases, we stress that the absence of saddles for $G$ is also a significant possibility and covers all previous works describing the complete parameters dynamics \citep{phuong2020inductive,lyu2021gradient,boursier2022gradient,chistikov2023learning,min2023early,wang2023understanding,glasgow2024sgd}, as well as any $2$-dimensional data.\footnote{\label{foot:saddlesdim}We refer to \cref{lemma:saddlesdim} for a proof of this fact.} As a consequence, the first phase of training in these works is fully grasped by \cref{thm:alignment}. The additional technical contribution of \citet{glasgow2024sgd} for the first phase  lies in further bounding the difference between gradient flow and SGD during this early alignment phase and we detail further in \cref{app:glasgow} how our results can be applied to the XOR data setting studied by \citet{glasgow2024sgd}. 
We provide in \cref{app:alignmentgeneral} an adapted version of \cref{thm:alignment} that holds in the presence of saddle points, but requires a stronger condition on the neurons.

The presence of saddle points is much harder to handle in general, as neurons can evolve arbitrarily slowly near saddle points. Providing a finite time convergence for such neurons then does not seem possible. Additionally, non trivial phenomena can happen in the presence of saddle points that are due to the non-continuity of the loss gradient. Notably some non-zero neurons can spontaneously leave their activation regions in a way that cannot be uniformly controlled, due to the multiplicity of gradient flow solutions of \cref{eq:ODE1}. In \cref{app:alignmentgeneral}, we again consider specific gradient flow solutions to control such phenomena, via \cref{cond:neurons}.

\medskip

Recall that \cref{ass:Dj} is only required so that \cref{lemma:Dj,lemma:alphamin,lemma:delta0,lemma:interior} hold. A deterministic version of both \cref{thm:alignment,thm:alignmentgeneral} is then possible without \cref{ass:Dj}, if we ensure that \cref{lemma:Dj,lemma:alphamin,lemma:delta0,lemma:interior} all simultaneously hold.

\medskip

The concurrent works of \citet{tsoy2024simplicity,kumar2024directional} provide a complementary characterisation of this early alignment phase. Notably \citet{tsoy2024simplicity} characterize, for smooth activations, the early alignment as well as the first growth of neurons cluster occurring after this alignment.  As an artifact of their analysis, they also require an odd data dimension~$d$. On the other hand, \citet{kumar2024directional} describe the early alignment for more general network architectures, assuming a unique gradient flow solution. 

In contrast to our work, their results do not quantify the initialisation scale $\lambda_{\alpha_0}^*$ at which this early alignment phase happens. \cref{thm:alignment} indeed provides a quantitative bound on such a scale, given by \cref{eq:lambdastaralpha} in \cref{app:alignment}. Importantly, this bound does not depend on the network width $m$. This point is crucial in \cref{sec:noconvergence} as it leads to results that hold for any value of $m$ and thus remain valid in the mean field limit when $m\to\infty$. Our \cref{thm:noconvergence} in \cref{sec:noconvergence} being valid in the mean field limit is a crucial point, since it implies that the seminal result by \citet{chizat2018global} of global convergence of overparameterised networks does not extend to the non-smooth case. This relation to \citet{chizat2018global} is discussed further in \cref{sec:discussion}.

\medskip

Although our scale threshold $\lambda_{\alpha_0}^*$ is independent of $m$, it still depends in the data. It is unclear how this threshold scales with quantities of interest, such as the number of training samples $n$ or the data dimension $d$. Given the generality of the data considered here (\cref{ass:Dj}), we cannot expect a more precise description of this dependence. However, we believe such a description to be both of interest and feasible for typical data models. In particular, the subsequent work of \citet{boursier2024simplicity} studies the evolution of this scale threshold with both $n$ and $d$ for a linear data model with Gaussian inputs.

\paragraph{Sketch of proof.} The first point of the proof follows from a simple Gr\"onwall inequality argument on \cref{eq:ODEs}, using the fact that $|a_j^t|\geq \|w_j^t\|$ for any $t$ by balancedness.

From the first point, $h_{\theta^t}(x_k)=\bigO{\lambda^{2-4\varepsilon}}$ during the early alignment phase. As a consequence, we can derive the following approximate ODE almost everywhere:
\begin{equation}\label{eq:ODEalignsketch}
\frac{\df \langle \w_j^t, D(w_j^t,\mathbf{0}) \rangle}{\df t}=  \|D(w_j^t,\mathbf{0})\|^2 - \langle \w_j^t, D(w_j^t,\mathbf{0})\rangle^2 - \bigO{\lambda^{2-4\varepsilon}},
\end{equation}
where $\w_j^t=\frac{w_j^t}{a_j^t}$ is the direction of the neuron $j$. If $a_j^t>0$ (resp. $a_j^t<0$), this equality corresponds to an approximate projected gradient ascent (resp. descent) of $G$ on the unit ball.

In a first time, thanks to \cref{cond:neurons2}, $|\langle \w_j^t, D(w_j^t,\mathbf{0})\rangle|$ is bounded away from $\|D(w_j^t,\mathbf{0})\|$, which ensures that $\langle \w_j^t, D(w_j^t,\mathbf{0})\rangle$ is increasing at a rate $\Omega_{\alpha_0}(1)$, until either $D(w_j^{t_2},\mathbf{0})=\mathbf{0}$ or $|\langle \w_j^{t_2}, D(w_j^{t_2},\mathbf{0})\rangle|$ is close to $\|D(w_j^t,\mathbf{0})\|$ for some time $t_2=\mathcal{O}_{\alpha_0}(1)$. 

The former case implies that the neuron is deactivated for the remaining of the training, and thus $D(w_j^{\tau},\mathbf{0})=\mathbf{0}$. 
The second case implies $\w_j^t$ is close in direction to $D(w_j^t,\mathbf{0})$. Moreover since $\langle \w_j^t, D(w_j^t,\mathbf{0}) \rangle$ is increasing at the start of training and $G$ has no saddle points, $D(w_j^t,\mathbf{0})$ corresponds to a local maximal direction of $G$ if $a_j^0>0$ (and minimal if $a_j^0<0$). 

As a consequence, we can show by studying the local maxima (or minima) of $G$ that if $\w_j^{t_2}$ is negatively correlated with $D(w_j^{t_2}, \mathbf{0})$, it is positively proportional to $-D(w_j^{t_2},\mathbf{0})$. From there,  $w_j^t$ stays aligned with $-D(w_j^{t_2},\mathbf{0})$ until $\tau$ and only changes in norm.
If instead $\w_j^{t_2}$ is positively correlated with $D(w_j^{t_2}, \mathbf{0})$, then we have a stability result (\cref{lemma:saddles} in \cref{app:globalstability}) showing that
\begin{equation}\label{eq:manifoldstability}
A(w_j^{t})=A(w_j^{t_2}) \quad \text{for any }t\in[t_2,\tau].
\end{equation}
In other words, $D(w_j^t,\mathbf{0})$ is constant on $[t_2,\tau]$, so that \cref{eq:ODEalignsketch} is close to an ODE of the form $f'(t)=c-f(t)^2$ on this interval. It then implies, by Gr\"onwall comparison, exponential convergence of $\langle \w_j^t, D(w_j^t,\mathbf{0}) \rangle$ towards  $\|D(w_j^t,\mathbf{0})\|$ and allows to conclude the second point of \cref{thm:alignment}. \qed

\medskip

Although these arguments are rather intuitive, their rigorous proof is tedious. Indeed, the estimated function $h_{\theta^t}$ is not exactly $0$. It is only close to $0$ during the early alignment phase, so that the neurons' dynamics are not exactly controlled by the vectors $D(w_j^t,\mathbf{0})$, but small perturbations of these vectors. As a consequence, we have to control carefully all the perturbed dynamics. In particular, showing the stability of the critical manifolds, i.e., \cref{eq:manifoldstability}, is quite technical and is done in \cref{app:localstability,app:globalstability}. The fact that stability is still preserved when slightly perturbing the vectors $D(w_j^t,\mathbf{0})$ largely relies on \cref{ass:Dj}.

\section{Convergence Towards Spurious Stationary Points}\label{sec:noconvergence}

As explained in \cref{sec:discussion} below, the early alignment phenomenon has an interesting impact, especially on the implicit bias of gradient descent.
However, it can also lead to undesirable, counter-intuitive results. This section aims to demonstrate that, in simple data examples, the early alignment phenomenon can be the cause of largely overparameterised neural networks converging to spurious stationary points. 
Such a result is somewhat surprising, as most of the current literature suggests that with a large enough number of neurons in the mean field regime, the learnt estimator converges towards a zero training loss function. This discrepancy with previous results is discussed further in \cref{sec:discussion}.

\medskip

We consider in this section the particular case of regression with ReLU activation: $\sigma(x)=\max(0,x)$ and
\begin{equation}\label{eq:squareloss}
\ell(\hat{y},y)=\frac{1}{2}(\hat{y}-y)^2.
\end{equation}
We consider the following $3$ points data example ($n=3$ in this section).
\begin{assumption}\label{ass:noconvergenceall}
The data is given by $3$ points $(x_k,y_k) \in \R^3$, for some $\eta>0$,
\begin{gather*}x_1\in (-1,-1+\eta]\times[1,1+\eta] \text{ and } y_1 \in[1,1+\eta]; \\
x_2\in[-\eta,\eta]\times[1-\eta,1+\eta] \text{ and } y_2\in(0,\eta];\\
x_3 \in [1-\eta,1)\times[1,1+\eta]\text{ and } y_3\in[1,1+\eta].
\end{gather*}
\end{assumption}
This assumption corresponds to a non-zero measure set of $\R^{3\times 3}$, so that it cannot be considered as a specific degenerate case. Also setting all the second coordinates of the $x_k$ to $1$ is possible and would correspond to the univariate dataset shown in \cref{fig:noconvdata} below, with bias terms in the hidden layer of the network.

\begin{figure}[ht]
\centering
\captionsetup{width=.8\linewidth}
\includegraphics[width=0.75\linewidth,trim=0cm 0cm 0cm 0cm, clip]{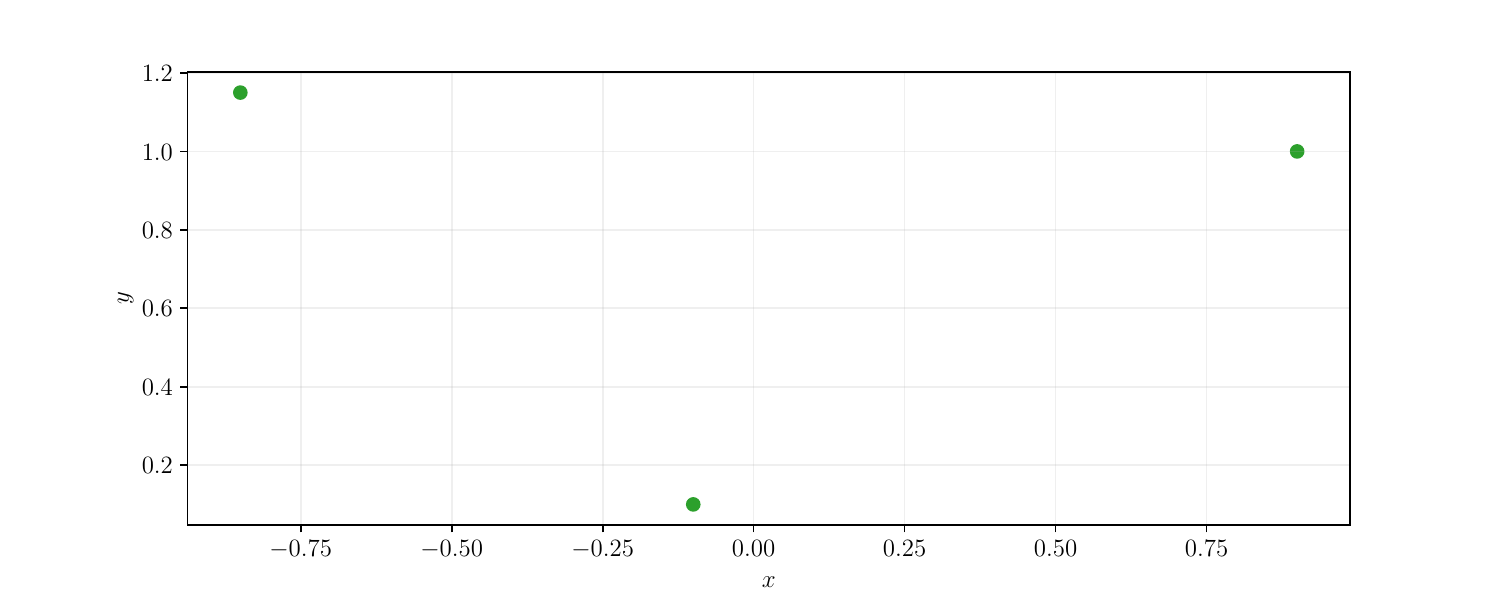}
\caption{example of univariate data verifying \cref{ass:noconvergenceall} with $\eta=\frac{1}{6}$. The second coordinate of the features $x_k$ is here fixed to $1$ to take into account bias terms of the neural network.}
\label{fig:noconvdata}
\end{figure}

\medskip

\begin{thm}\label{thm:noconvergence}
Let $\beta^* \coloneqq  \left(X^{\top} X \right)^{-1}X^{\top}y$ be the ordinary least squares estimator of the data.
If \cref{ass:noconvergenceall} holds with $\eta<\frac{1}{6}$, then there exists some constant $\tilde\lambda=\Theta(1)$ such that for any $\lambda < \tilde\lambda$ and $m \in \N$, the parameters $\theta^t$ converge to some $\theta_{\infty}$ such that
\begin{equation*}
h_{\theta_{\infty}}(x_k)= x_k^{\top}\beta^* \text{ for any } k\in [n]. 
\end{equation*}
In particular, it satisfies $\lim_{t\to\infty}L(\theta^t)>0$. 
\end{thm}
A more refined version of \cref{thm:noconvergence}, stated in \cref{app:proofnoconvergence}, even states that the learnt estimator is very close the positive part of the linear, ordinary least squares estimator, when taking $\lambda\to 0$. 
With the data example given by \cref{ass:noconvergenceall}, the linear estimator corresponding to $\beta^*$ does not fit all the data, while a simple two-neurons network can. In conclusion, although the data still seems easy to fit, \cref{thm:noconvergence} implies that for an initialisation scale smaller than $\tilde{\lambda}$, the learnt parameters converge towards a spurious local stationary point.

\medskip

The scale of initialisation $\lambda$ does not depend on the width $m$, but solely on the training data. In particular, \cref{thm:noconvergence} even holds in the limit $m\to\infty$. Even though this result might seem to contradict known global convergence results in the infinite width limit at first hand \citep{chizat2018global,wojtowytsch2020convergence}, it is actually compatible with these results and highlights the importance of some assumptions that appear to be technical in nature (differential activation or infinite data) to get these global convergence results. A further discussion on this aspect is given in \cref{sec:discussion}.

\medskip

The proof of \cref{thm:noconvergence} describes the complete dynamics of the parameters $\theta$ during training. This dynamics happens in three distinct phases, described in detail in \cref{sec:sketch}.
In particular, the first phase of the proof is a direct consequence of our early alignment result given by \cref{thm:alignment}. 
Because of this alignment, the positive neurons then nearly behave as a single neuron, which is positively correlated with all data points.
From there, adding a neuron with negative output weights only increases the training loss, for any direction of the weight. As a consequence, it becomes impossible to activate a new direction while training and the network remains equivalent to a single neuron one until the end of training.

\begin{rem}
The condition $\eta<\frac{1}{6}$ does not seem to be tight and is only needed for analytical reasons here. In \cref{sec:expe}, a similar result is empirically observed for a larger $\eta$. We provide in \cref{app:assumptions} more general data assumptions for which \cref{thm:noconvergence} still holds.
\end{rem} 

\subsection{Sketch of Proof}\label{sec:sketch}

This section provides a sketch of proof for \cref{thm:noconvergence}. In particular, the proof relies on the description of a three phases training dynamics as follows.
\begin{enumerate}
\item The first phase corresponds to the early alignment one, after which all the neurons with positive output weights are aligned with the single extremal vector.
\item During the second phase, the neurons with positive output weights all grow in norm, while staying aligned. After this phase, the estimated function $h_{\theta^t}$ is already very close to the single neuron network represented by $\beta^*$.
\item During the last phase, a local Polyak-Łojasiewicz (PL) argument allows to state that the parameters~$\theta$ do not significantly move and converge at an exponential speed towards the linear estimator~$\beta^*$.
\end{enumerate}
The first and second phase are a mere consequence of the fact that for our data example, there is a unique extremal vector given by $\frac{1}{n}\sum_{k=1}^n y_k x_k$. The third phase on the other hand is a consequence of both that i)  all positive neurons are aligned and behave as a single neuron at the end of the early alignment phase; ii) the data is designed so that adding any negative ReLU neuron to the linear estimator will only increase the training loss. In consequence, no additional neuron will be ``created'' during this third phase and the estimator will remain equivalent to $\beta^*$.

\subsubsection{Phase 1}

The first phase of the training dynamics is the early alignment one and can be fully described by \cref{thm:alignment}. 
First define the sets
\begin{gather*}
\cI \coloneqq \left\lbrace i\in[m] \mid a_i^0>0 \text{ and } \exists k\in[n], \langle w_i^0, x_k\rangle >0 \right\rbrace,\\
\cN \coloneqq \left\lbrace i\in[m] \mid a_i^0<0 \right\rbrace.
\end{gather*}
Also define the vectors
\begin{gather*}
D^t = -\frac{1}{n}\sum_{k=1}^n (h_{\theta^t}(x_k)-y_k)x_k \quad \text{and} \quad
D^* = \frac{1}{n}\sum_{k=1}^n y_k x_k.
\end{gather*}
An important property of the considered dataset is that there is a single extremal vector, given by~$D^*$, and all data points are pairwise positively correlated. During the first alignment phase, all the neurons with positive output weights then end up aligned with $D^*$ as stated by \cref{lemma:phase1} below.
\begin{lem}[Phase 1]\label{lemma:phase1}
Under \cref{ass:noconvergenceall} with $\eta<\frac{1}{6}$, for any $\varepsilon\in(0,\frac{1}{3})$, $\lambda<\tilde\lambda$ and $\tau=\frac{-\varepsilon\ln(\lambda)}{D_{\max}}$.
\begin{enumerate}
\item Positive neurons are aligned with $D^*$: $\forall i\in \cI, \langle D^*, \frac{w_i^{\tau}}{a_i^{\tau}} \rangle = \| D^*\| - \bigO{\lambda^{\varepsilon}}$ and $a_i^0 \leq a_i^{\tau} \leq a_i^{0}\lambda^{-2\varepsilon}$.
\item Negative neurons' norm decreased: $\forall i \in \cN, a_{i}^0 \leq a_i^{\tau} \leq 0$.
\item Remaining neurons do not move during the whole training: $\forall i \not\in \cI\cup \cN, \forall t \in \R_+, w_i^t=w_i^0$ and $a_i^t=a_i^0$.
\end{enumerate}
\end{lem}
\cref{lemma:phase1} is a direct application of \cref{thm:alignment}, up to additional minor remarks, that are specific to \cref{ass:noconvergenceall}.
It does not allow to control the directions of the neurons in $\cN$. Such a control is actually not needed since the norm of the neurons in $\cN$ stays close to $0$ in the following phases.

\subsubsection{Phase 2}\label{sec:phase2}

At the end of the first phase, the neurons in $\cI$ are aligned with $D^*$ and positively correlated with all the points $x_k$. From there, their norm grows during the second phase, while the norm of neurons in $\cN$ remains close to $0$. While the neurons in $\cI$ grow in norm, their direction also changes. 
Controlling both their direction and norm simultaneously is intricate and split into two subphases (2a and 2b) detailed below. At the end of this norm growth and change of direction during the second phase, the group of positive neurons almost behave as the optimal linear regressor $\beta^*$. In other words with $\tau_3$ the end of the second phase, we have
\begin{gather*}
\sum_{i\in \cI} a_i^{\tau_3}w_i^{\tau_3} \approx \beta^*\quad
\text{ and } \quad \forall i\in \cI, \forall k\in [n], \langle w_i^{\tau_3}, x_k\rangle >0.
\end{gather*}
The second inequality implies that each neuron $i\in \cI$ behaves linearly with the data points despite the ReLU activation.

\paragraph{Phase 2a.} 
The subphase 2a is the first part of the phase 2. During  this part, the norms of the neurons in $\cI$ grow until reaching a small, but $\Theta(1)$ threshold given by $\varepsilon_2$. We choose $\varepsilon_2$ small enough, so that during this phase $D^t = D^* + \bigO{\varepsilon_2}$. As a consequence, the dynamics of the neurons during this part is not much different than in the early alignment phase. This behavior leads to \cref{lemma:phase2a} below.
\begin{lem}[Phase 2a]\label{lemma:phase2a}
Consider $\varepsilon,\tau$ fixed by \cref{lemma:phase1}.
Under \cref{ass:noconvergenceall} with $\eta<\frac{1}{6}$, there exist constants $c',\varepsilon_2^*=\Theta(1)$ such that for any $\varepsilon_2\in[c'  \lambda^{\frac{\varepsilon}{2}},\varepsilon_2^*]$ and $\lambda<\tilde\lambda$, there is some $\tau_2 \in (\tau, \infty)$ such that
\begin{enumerate}
\item positive neurons norm reach the threshold: $\sum_{i\in \cI} (a_i^{\tau_2})^2 = \varepsilon_2$;
\item positive neurons are positively correlated with all data points: $\forall i\in \cI, k\in [n], \langle w_i^{\tau_2},x_k \rangle =\Omega(1)$;
\item positive neurons are nearly aligned: $\forall i,j \in \cI, \langle \w_i^{\tau_2},\w_j^{\tau_2}\rangle= 1 - \bigO{\lambda^{1-\varepsilon}}$;
\item negative neurons' norm decreased: $\forall i \in \cN, a_{i}^0 \leq a_i^{\tau_2} \leq 0$.
\end{enumerate}
\end{lem}
Note that the third point of \cref{lemma:phase2a} only states that the positive neurons are pairwise aligned, while \cref{lemma:phase1} stated that they were aligned with $D^*$. This is because at the time $\tau_2$, the positive neurons might not be aligned anymore with $D^*$ at a precision level $\lambda$. The direction of the positive neurons might indeed have slightly moved away from $D^*$, but all the positive neurons kept a common direction during this phase. This property proves useful in the following, as we want the neurons in $\cI$ to behave as a single neuron, \ie we want them all pairwise aligned.

\paragraph{Phase 2b.} 

At the end of the first part, given by the Phase 2a, the positive neurons' have a norm $\varepsilon_2$ and are on the verge of exploding in norm. The Phase 2b corresponds to this explosion until the positive neurons almost correspond to the optimal linear regressor $\beta^*$. More precisely for some $\varepsilon_3>0$, the Phase 2b ends when
\begin{equation*}
\left\|\beta^* - \sum_{i\in \cI}
a_i^{\tau_3} w_i^{\tau_3} \right\| \leq \varepsilon_3.
\end{equation*}
\cref{lemma:phase2b} below states that this condition is reached at some point, while the negative neurons remain small in norm and the positive neurons are still positively correlated with all data points.
\begin{lem}[Phase 2b]\label{lemma:phase2b}
Consider $\varepsilon, \varepsilon_2, \tau_2$ fixed by \cref{lemma:phase2a}. Under \cref{ass:noconvergenceall} with $\eta<\frac{1}{6}$, there exist constants $c_3=\Theta(1)$ and $\tilde\varepsilon_2^*=\Theta(1)$ depending only on the data, such that if $\varepsilon_2\leq \tilde\varepsilon_2^*$ and $\varepsilon_3\geq\lambda^{c_3\varepsilon \varepsilon_2}$ and $\lambda<\tilde\lambda$,
after some time $\tau_3 \in [\tau_2,+\infty)$,
\begin{enumerate}
\item positive neurons behave as the optimal linear regressor: $\left\|\beta^* - \sum_{i\in \cI}
a_i^{\tau_3} w_i^{\tau_3} \right\| \leq \varepsilon_3$;
\item positive neurons are still nearly aligned: $\forall i,j \in \cI, \langle \w_i^{\tau_2},\w_j^{\tau_2}\rangle= 1 - \bigO{\lambda^{\varepsilon}}$;
\item positive neurons are positively correlated with all data points: $\forall i\in \cI, \forall k\in[n],$ $\langle \frac{w_i^{\tau_3}}{a_i^{\tau_3}},x_k\rangle=\Omega(1)$;
\item negative neurons' norm is small: $\forall i\in \cN, a_{i}^0\lambda^{-\varepsilon} < a_i^{\tau_3} \leq 0$.
\end{enumerate}
\end{lem}
The main argument in the proof of \cref{lemma:phase2b} is that the Phase 2b happens in a time of order\footnote{The vector $\sum_{i\in \cI} a_i^{t} w_i^{t}$ is indeed approximating the gradient flow of the linear regression of the data, with a learning rate lower bounded by $\varepsilon_2$.} $\bigO{\frac{-\ln(\varepsilon_3)}{\varepsilon_2}}$. Since we choose $\lambda$ very small with respect to $\varepsilon_2,\varepsilon_3$, this time can be seen as very short. In comparison, the phase 2a ends after a time of order $-\ln(\lambda)$. In this short amount of time, the norm of the negative neurons cannot grow significantly enough and the positive neurons cannot significantly disalign. Thanks to the latter and the specific structure of the data, the positive neurons remain positively correlated with all data points during the whole second phase.

\subsubsection{Phase 3}\label{sec:phase3}

The third phase is only theoretical since it describes infinitesimal movement towards the convergence point $\beta^*$. For this phase, we define for $\varepsilon_4, \delta_4>0$,
\begin{align*}
\tau_4 \coloneqq \inf\Big\{t\geq \tau_3 \ \Big| \   \|\theta^t - \theta^{\tau_3}\|_{2} \geq \varepsilon_4 \text{ or } \exists i\in \cI, k\in[n] \langle\w_i^t,x_k \rangle\leq \delta_4 \|x_k\| \Big\}.
\end{align*}

\begin{lem}[Phase 3]\label{lemma:phase3}
If \cref{ass:noconvergenceall} holds with $\eta<\frac{1}{6}$, consider $\varepsilon, \varepsilon_2, \varepsilon_3, \tau_3$ fixed by \cref{lemma:phase2b}. There are positive constants $\tilde\lambda,\varepsilon_3^*,\varepsilon_4^*,\delta^*_4=\Theta(1)$ such that if we also have $\lambda<\tilde\lambda$, $\varepsilon_3 < \varepsilon^*_3$, $\varepsilon_4 =\varepsilon_4^*$ and $\delta_4 \leq \delta^*_4$, then $\tau_4 = \infty$.
Moreover, there then exists $\theta^{\infty}$ such that 
\begin{enumerate}
\item the parameters converge towards some limit: $\displaystyle\lim_{t\to\infty}\theta^t = \theta^{\infty}$;
\item all the neurons of this limit behave linearly with the training data: \\$\displaystyle\forall i \in [m], \left(\forall k \in [n], \langle w_i^{\infty}, x_k \rangle \geq 0  \right) \text{ or } \left(\forall k \in [n], \langle w_i^{\infty}, x_k \rangle \leq 0  \right)$;
\item the active neurons correspond to the optimal linear estimator: ${\!\!\!\!\!\!\!\!\displaystyle\sum_{\substack{i\in[m]\\\forall k\in[m], \langle  w_i^{\infty}, x_k\rangle \geq 0}}\!\!\!\!\!\!\!\! a_i^{\infty} w_i^{\infty} = \beta^*}$.
\end{enumerate}
\end{lem}
Similarly to \cite{chatterjee2022convergence}, the proof of \cref{lemma:phase3} relies on a local PL condition. The condition is a bit trickier here, since it is with respect to the loss obtained by a spurious stationary point instead of a global minimum. Getting this local PL condition requires to also control the negative neurons. Controlling them is intricate but possible by showing that they either act as a linear operator or that they decrease in norm.

Once we derive the local PL condition (see \cref{lemma:localPL} in \cref{app:phase3}), the remaining of the proof encloses the dynamics of the parameters in a small compact set and then shows exponential convergence towards some parameter $\theta^\infty$ satisfying the conditions of \cref{lemma:phase3}, which implies that
\begin{gather*}
\forall k \in [n], h_{\theta^\infty}(x_k) = \langle \beta^*,x_k \rangle ,\\
\forall x\in \R^d, h_{\theta^\infty}(x) = \langle \beta^*,x \rangle_+ + \bigO{\varepsilon_3+\varepsilon_4}.
\end{gather*}

\section{Experiments}\label{sec:expe}

This section empirically illustrates the results of \cref{thm:alignment,thm:noconvergence}. The considered dataset does not exactly fit the conditions of \cref{thm:noconvergence} to illustrate that \cref{ass:noconvergenceall} with $\eta<\frac{1}{6}$ is only needed for analytical purposes. The dataset is however similar to datasets satisfying \cref{ass:noconvergenceall} (see \eg \cref{fig:noconvdata}) in the sense that all three data points are positively correlated, with positive labels; and the middle point is below the optimal linear regressor. 
The data points are here represented as unidimensional, since their second coordinate is fixed to $1$ to take into account bias terms in the hidden layer of the neural network.

\begin{figure}[b!] 
\centering
\begin{subfigure}{0.5\linewidth}
\includegraphics[width=\linewidth, trim=1cm 0.9cm 1.5cm 1cm, clip]{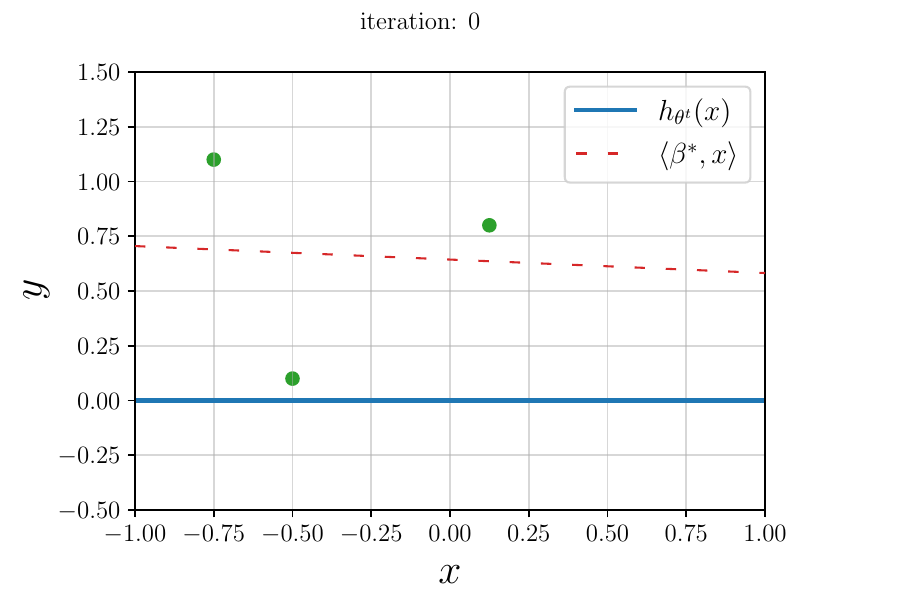}
\caption{\label{fig:estim0}Estimated function at initialisation (iteration $0$)}
\end{subfigure}
\hfill
\begin{subfigure}{0.4\linewidth}
      \includegraphics[width=\linewidth,trim=8cm 1.5cm 8cm 1.5cm, clip]{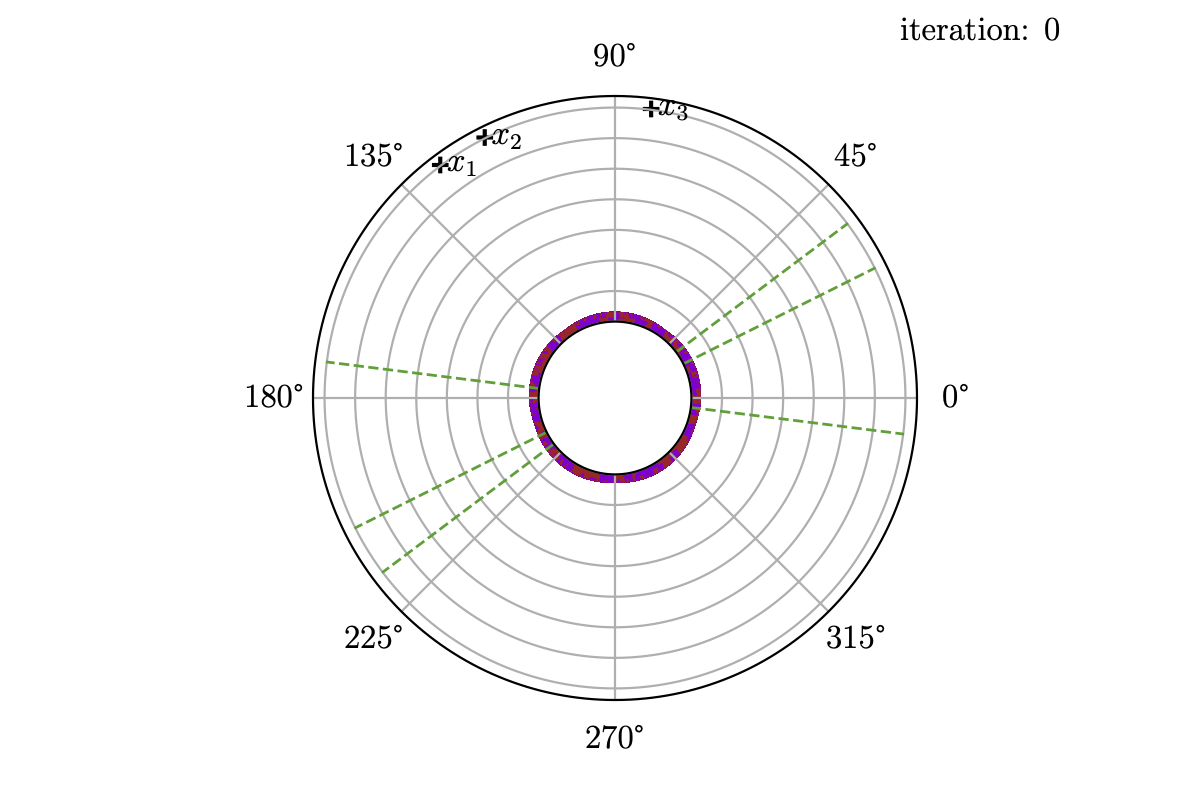}
\caption{\label{fig:alignment0}Weights' repartition at initialisation}
\end{subfigure}
\begin{subfigure}{0.5\linewidth}
\includegraphics[width=\linewidth,trim=1cm 0.9cm 1.5cm 1cm, clip]{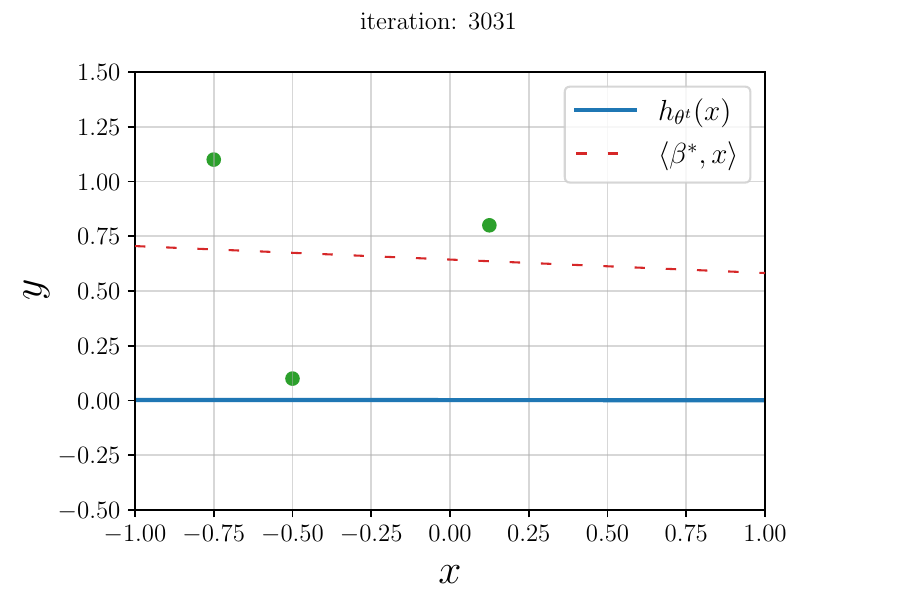}
\caption{\label{fig:estim1}Estimated function after early alignment (iteration $3000$)}
\end{subfigure}
 \hfill
 \begin{subfigure}{0.4\linewidth}
      \includegraphics[width=\linewidth,trim=8cm 1.5cm 8cm 1.5cm, clip]{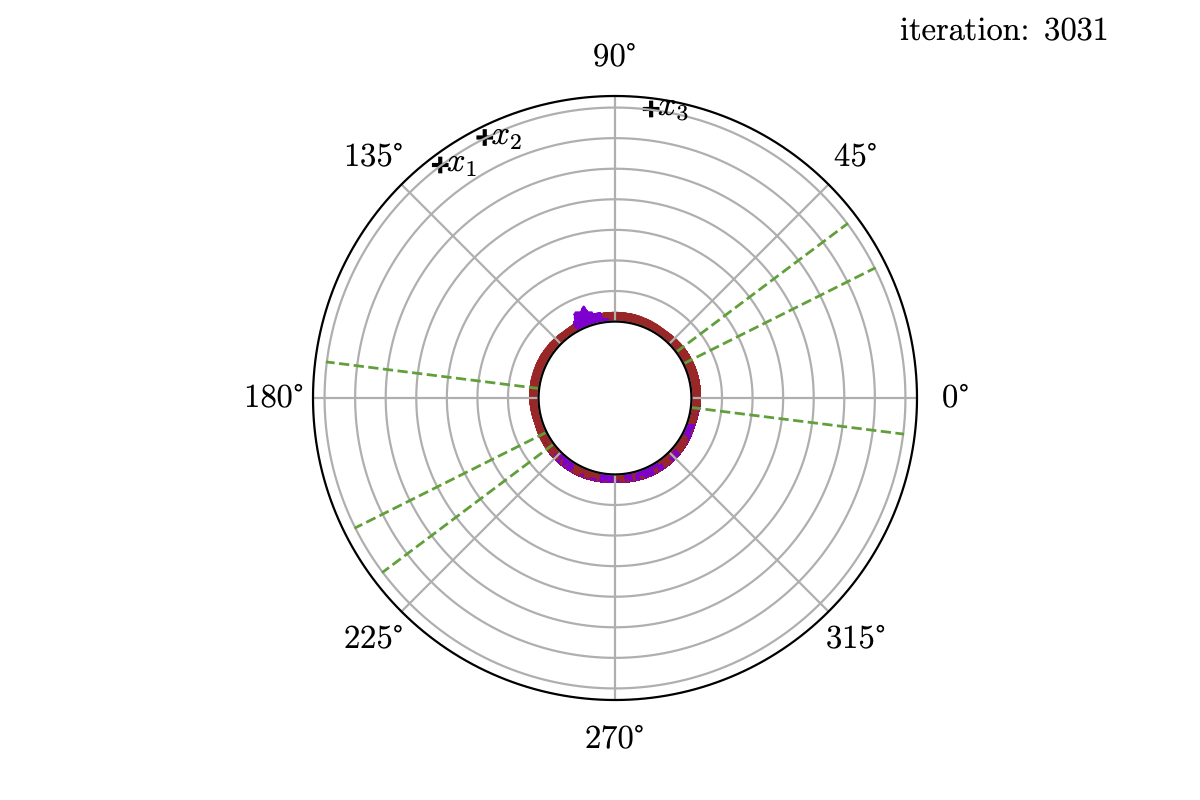}
\caption{\label{fig:alignment1}Weights' repartition after early alignment}
\end{subfigure}
%
   %
  \caption{\label{fig:dynamicsmain}Training dynamics (part 1/2).}
\end{figure}

    \begin{figure}[t]
    \addtocounter{figure}{-1}
    \centering
  \begin{subfigure}{0.5\linewidth}
    \addtocounter{subfigure}{4}
\includegraphics[width=\linewidth,trim=1cm 0.9cm 1.5cm 1cm, clip]{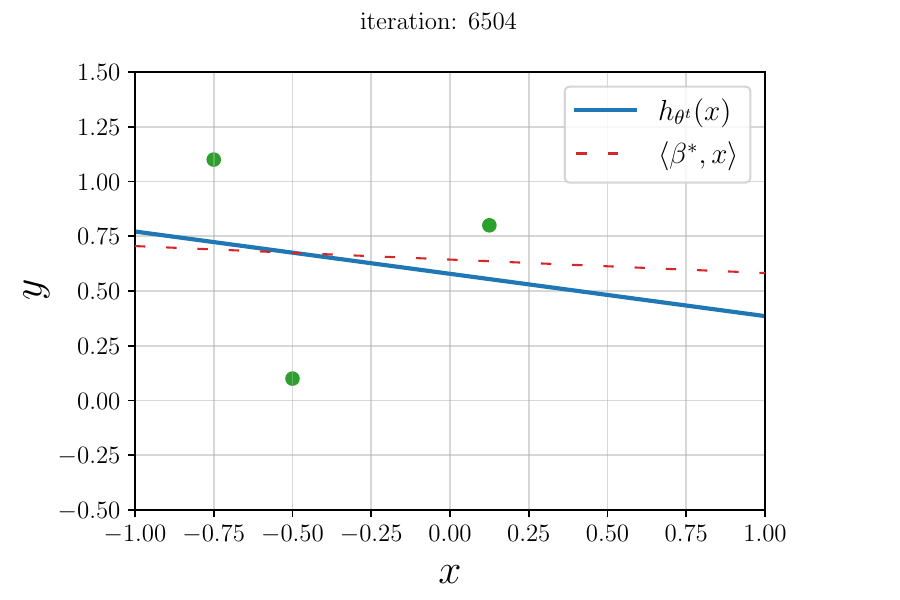}
\caption{\label{fig:estim2}Estimated function during neurons' growth (iteration $6500$)}
\end{subfigure}
 \hfill
 \begin{subfigure}{0.4\linewidth}
      \includegraphics[width=\linewidth,trim=8cm 1.5cm 8cm 1.5cm, clip]{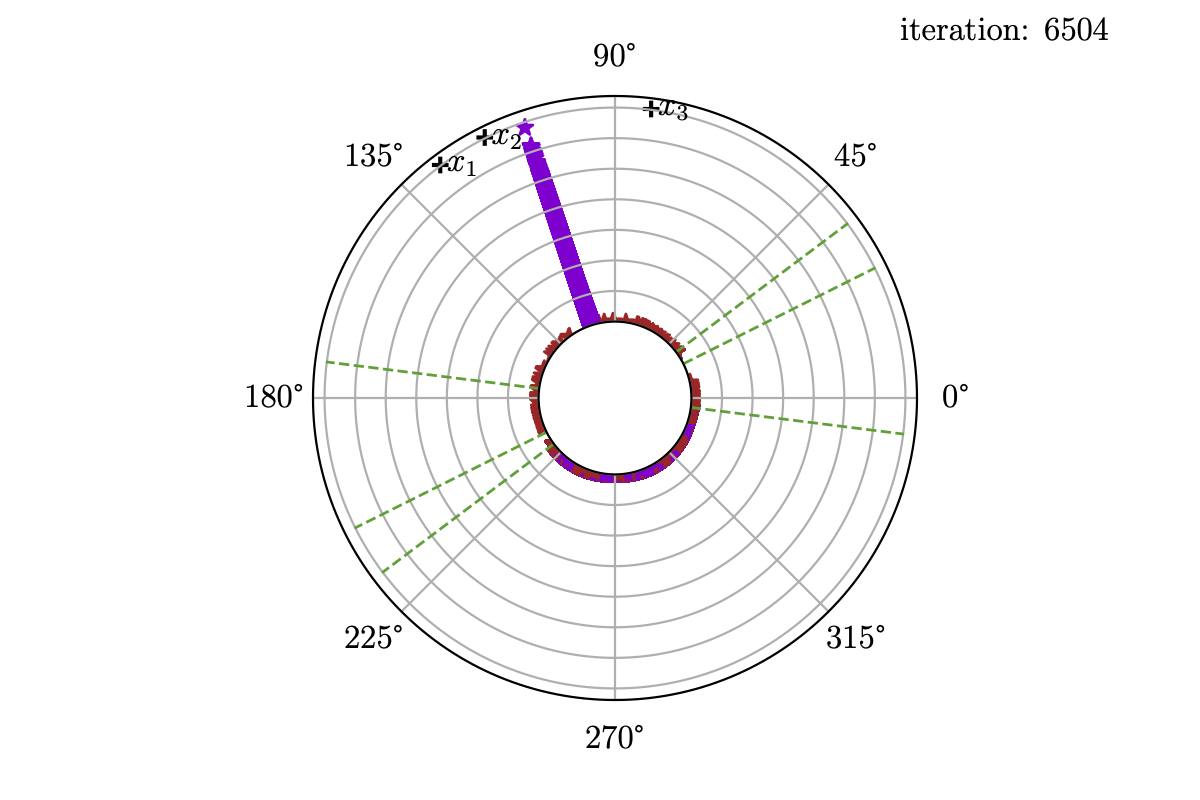}
\caption{\label{fig:alignment2}Weights' repartition during neurons' growth}
\end{subfigure}
\begin{subfigure}{0.5\linewidth}
\includegraphics[width=\linewidth,trim=1cm 0.9cm 1.5cm 1cm, clip]{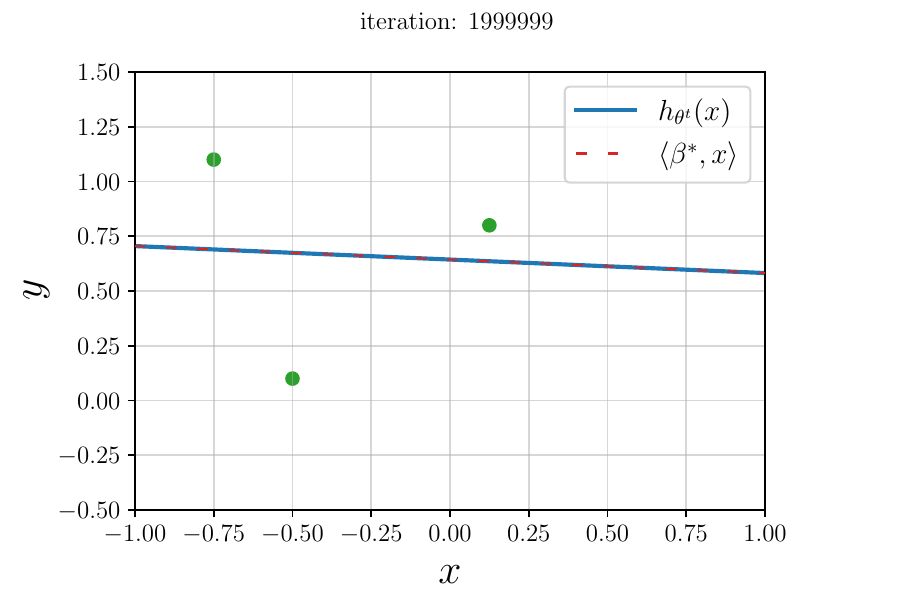}
\caption{\label{fig:estim3}Estimated function at convergence (iteration $2\times 10^6$)}
\end{subfigure}
 \hfill
 \begin{subfigure}{0.4\linewidth}
      \includegraphics[width=\linewidth,trim=8cm 1.5cm 8cm 1.5cm, clip]{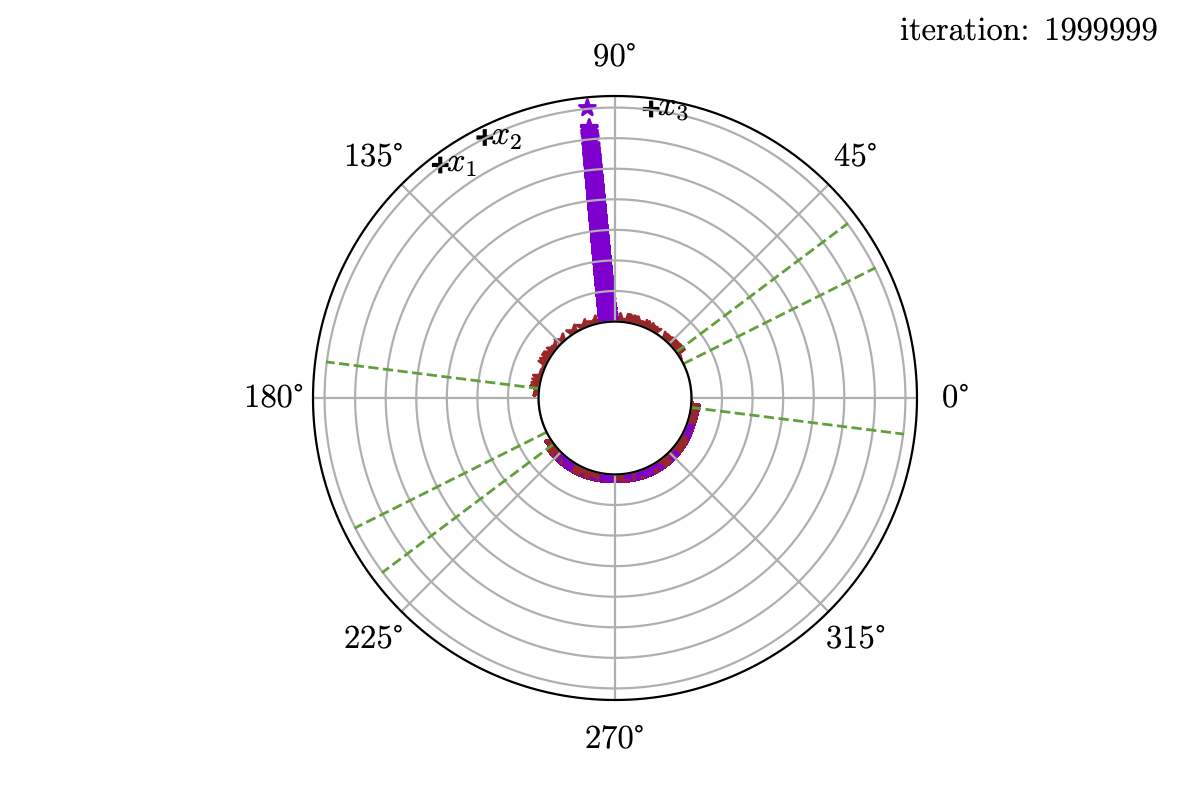}
\caption{\label{fig:alignment3}Weights' repartition at convergence}
\end{subfigure}
\caption{\label{fig:dynamicsmain2}Training dynamics (part 2/2).}
\end{figure}

For the given set of data points, we train a one-hidden ReLU network with gradient descent over the mean square loss given by \cref{eq:squareloss}. Experimental details and additional experiments are provided in \cref{sec:additionalexpe}. The code and animated versions of the figures are also available at \url{github.com/eboursier/early_alignment}.


\cref{fig:dynamicsmain} illustrates the training dynamics over time. The left column represents the estimated function $h_{\theta^t}(x)$ at different training times; while the right column represents the $2$-dimensional repartition of the network weights $w_j^t$ in polar coordinates. 
In the latter, the inner circle corresponds to $0$ norm and is shifted away from the origin to accurately observe the early alignment phenomenon. Each star corresponds to a single neuron $w_j^t\in \R^2$, represented in purple (resp. red) if its associated output weight $a_j^t$ is positive (resp. negative). The data points directions' $x_k$ are represented by black crosses and the dotted green lines delimit the different activation cones $A^{-1}(u)$. 
We recall that the weights $w_j^t$ are two-dimensional, as their second coordinate corresponds to bias terms.

\medskip

\cref{fig:estim0,fig:alignment0} show the parameters at initialisation. Given the small initialisation scale, the neurons are all nearly $0$ norm and the estimated function is close to the zero function. Moreover, the neurons' directions are uniformly spread over the whole space. 

\cref{fig:estim1,fig:alignment1} then show the state of the network after the early alignment phase. In particular, all the positive neurons (in purple) are either aligned with the single extremal vector $\frac{1}{n}\sum_{k=1}^n y_k x_k$ or deactivated with all data points. The latter then never move from their initial position. Meanwhile, the neurons' norm is still close to $0$ so that the estimated function is still nearly $0$. Actually, \cref{fig:estim1,fig:alignment1} precisely show the end of Phase 2a (see \cref{lemma:phase2a}), where the positive neurons are on the verge of exploding in norm. Indeed, we can see on \cref{fig:alignment1} a small purple bump: the positive neurons just start to considerably grow in norm.

\cref{fig:estim2,fig:alignment2} show the learnt parameters during the norm growth of the positive neurons. They can be seen as the end of the Phase 2, where the learnt estimator is close enough to the optimal linear one. While the positive neurons considerably grew in norm, the negative neurons remained small in norm and are actually irrelevant to the learnt estimator. 

\cref{fig:estim3,fig:alignment3} finally illustrate the learnt parameters at convergence. The learnt estimator converged to the optimal linear one, as predicted by \cref{thm:noconvergence}. In \cref{fig:alignment3}, all the neurons (including the negative ones) are located in only two activation cones: either they are activated with all the data points, or they are activated with no data point. This confirms our theoretical analysis and in particular the description given in the proof of \cref{lemma:Rt} in \cref{app:alignment}.
Another interesting observation is that the final direction of the positive neurons changed from their direction at the end of the early alignment (see difference between \cref{fig:alignment1,fig:alignment3}). As explained in \cref{sec:phase2}, this is one of the challenge in analysing the complete dynamics of the parameters: simultaneously controlling the norm growth and change of direction of the neurons is technically intricate.

\section{Discussion}\label{sec:discussion}

\paragraph{Early alignment and implicit bias.}
\cref{thm:alignment} states that neurons satisfying \cref{cond:neurons2} end up aligned towards a few key directions, given by extremal vectors, at the end of the early alignment phase. We believe this quantization of represented directions to be at the origin of the implicit bias of (stochastic) gradient descent. Many recent works indeed suggest that gradient descent is implicitly biased towards small rank hidden weights matrix~$W$ or, equivalently, a sparse number of directions represented by the neurons \citep{shevchenko2022mean,safran2022effective}. Other works evidence that the implicit bias can often be characterised by minimal norm interpolator \citep{lyu2019gradient,chizat2020implicit}, which happen to be closely related to this notion of sparse represented directions \citep{parhi2021banach,stewart2022regression,boursier2023penalising}.
The early alignment phase seems to confirm such a behavior, since it enforces alignment of weights towards a small number of directions, even with an omnidirectional initialisation. 

\medskip

Another interesting point that can be made from both \cref{thm:noconvergence} and the experiments in \cref{sec:expe,sec:additionalexpe} is that even if the training procedure does not manage to perfectly fit all the data, there still seems to be an implicit regularisation of the estimator at convergence. 
Specifically, for the $3$ points example given in \cref{thm:noconvergence,sec:expe}, the network does not manage to fit the data. It actually ends up being equivalent to a single neuron network, while two neurons are necessary to fit the data. However, there still seems to be some implicit regularisation at play, as the learnt estimation is very simple. Even more striking is that the learnt estimator represents the best possible way to fit the data with a single neuron, which is here given by the optimal linear estimator.
A similar observation is made in \cref{sec:stewart}, where the network ultimately becomes equivalent to a $5$ neurons network, but $8$ neurons are required to fit the data. Moreover, it also seems that the obtained estimator is given by the best way to fit the data points with only $5$ neurons here. Although these claims are pure intuitions and empirical observations, we believe that the stationary points reached in our different examples are only bad from an optimisation point of view, but could still yield good generalisation behaviors. Such an intuition has been theoretically confirmed in a subsequent work for a linear data model \citep{boursier2024simplicity}.

\paragraph{Omnidirectionality of the weights.}
\looseness=-1
Omnidirectionality is the property that the weights $(w_j^t,\sign(a_j^t))$ cover all possible directions of $\R^d\!\times\!\{-1,1\}$, in the limit of infinite width $m\to\infty$. This property has been used to show that two layer infinite width networks converge to a global minimum of the loss. In particular with an omnidirectional initialisation, omnidirectionality is preserved along the whole training if either the activation is differentiable \citep{chizat2018global} or with ReLU and an infinite dataset \citep{wojtowytsch2020convergence}. This property allows to never get stuck in a bad local minimum, since there is always a direction along which increasing the neurons' norm decreases the training loss.

On the other hand, the spurious convergence result of \cref{sec:noconvergence} is based on the fact that the early alignment phase can cause loss of omnidirectionality of the weights.\footnote{We stress in the paragraph below that our choice of initialisation is compatible with \citet{chizat2018global,wojtowytsch2020convergence}.} 
At the end of the early alignment phase, the neurons with positive output weights only cover a very small cone around the sole extremal vector, as can be observed in \cref{fig:dynamicsmain}.
This result does not contradict \citet{chizat2018global,wojtowytsch2020convergence}, since the activation is non-differentiable and the data finite in our setting. It actually highlights how fragile and crucial is the omnidirectionality to guarantee convergence towards global minima in the mean field regime.

\medskip

Also, \cref{thm:alignment} does not necessarily imply a loss of omnidirectionality. 
The weights can still cover all possible directions after the early alignment phase. \cref{thm:alignment} instead only implies that a large fraction of them is concentrated around a few key directions. Even with omnidirectional weights, the key directions thus have a much larger number of neurons with respect to other directions. We believe this discrepancy between the directions to be sufficient for getting an implicit bias towards small rank weights matrix.

\paragraph{Scale of initialisation.}
\looseness=-1
A strong feature of \cref{thm:noconvergence} is that the scale initialisation $\tilde\lambda$ does not depend on the network width $m$. The only dependence on $m$ is given by the $\frac{1}{\sqrt{m}}$ scaling. The absence of this scaling corresponds to the lazy regime, where convergence towards global minimum happens, yet without feature learning, which can yield bad generalisation on unseen data \citep{chizat2019lazy}.

Obtaining \cref{thm:noconvergence} with a threshold $\tilde\lambda$ that depends on the width $m$ is technically easier. However, such a setting would not be compatible with the infinite width network regimes considered by \citet{chizat2018global,wojtowytsch2020convergence}. On the other hand when fixing $\lambda$ and taking $m\to\infty$, our result lies in the infinite width setting. Moreover, our initialisation regime can be made compatible with the result of \citet{wojtowytsch2020convergence} by choosing $|a_j^0|=\|w_j^0\|$ for all $j$, as explained at the bottom of their second page. 
As explained above, the only reason that our result does not contradict \citet{chizat2018global,wojtowytsch2020convergence} is that we considered non-differentiable activation with finite data.

\medskip

A conclusion of \cref{thm:alignment,thm:noconvergence} is that choosing a small initialisation scale should yield a better implicit bias towards low rank hidden weights matrix, but at the risk of converging towards spurious stationary points (still with low rank matrix).  On the other hand, choosing a large initialisation scale can affect generalisation on new, unseen data \citep{chizat2019lazy}. In practice, an intermediate scale might thus be the best trade-off to yield both convergence to global minima and small generalisation errors. Yet on the theoretical side, analysing such intermediate regimes is even harder than extremal ones.

\paragraph{Nature of the stationary point.} Until now, we only described the convergence point reached in the no convergence example of \cref{thm:noconvergence} as a spurious stationary point. Its exact nature, i.e., whether it is a saddle point or a spurious local minimum actually depends on the choice of initialisation. If the initialisation counts \textit{perfectly  balanced} neurons with $|a_j^0|=\|w_j^0\|$, the convergence point $\theta^{\infty}$ of the parameters, described by \cref{thm:noconvergencegeneral} in the Appendix, can include zero neurons ($a_j^{\infty}=0$, $w_j^{\infty}=\mathbf{0}$), which corresponds to saddle points of the loss. We can indeed slightly perturb such a neuron in a good direction, so that it decreases the loss.

If on the other hand we only consider \textit{dominated neurons} with $|a_j^0|>\|w_j^0\|$, the convergence point $\theta^{\infty}$ does not include zero neurons, as $a_j^{\infty}$ is non-zero for any $j$. From there,  computations similar to the proof of \cref{lemma:localPL}, along with the description of $\theta^{\infty}$ by \cref{thm:noconvergencegeneral}, yield that the convergence point $\theta^{\infty}$ is actually a local minimum of the loss.

\paragraph{Generality of \cref{thm:noconvergence}.}

\cref{thm:noconvergence} states that for some cases of data, convergence towards spurious stationary points happens. First note that \cref{ass:noconvergenceall} describes a non-zero measure set for any $\eta>0$, so that this data example is not fully degenerate. Moreover, \cref{thm:noconvergence} is proven for $\eta<\frac{1}{6}$, but similar observations are empirically made for datasets that do not satisfy \cref{ass:noconvergenceall} (see \cref{sec:expe}). 

More generally, we believe that this case of bad convergence happens on many different data examples. First, \cref{ass:noconvergenceall} is needed for three technical conditions, detailed in \cref{app:assumptions}.
One of them allows to show that the whole network behaves as a single neuron after the early alignment. Yet on more complex data examples, the network can still get stuck after having created several neurons, as can be observed in \cref{sec:additionalexpe}. Having a precise description of the dynamics then becomes more intricate, but the reason for the failure of training is the same: the loss of weights omnidirectionality hinders the growth of neurons in a good direction. 
We believe that such \textit{bad convergence} could happen in many low-dimensional settings (in the sense $n\gtrsim d$). Indeed in such cases, the early alignment should lead to a sparse number of represented directions at its end, given by the number of critical points of the alignment function $G$. We conjecture this number of directions to be independent of $n$ in the low-dimensional setting for typical structured data models \citep[such as the ones that can be found in][Table 1]{bach2017breaking}. Such an independence would then lead to an impossibility of fitting all data points, since the network would become equivalent to a smaller network with less than $n$ neurons. 

\medskip

Besides its scale, the initialisation has the property of balancedness given by \cref{eq:balancedinit} here. This property allows to fix the sign of the output weights $a_j^t$. 
Without it, some of the weights $a_j^t$ change their sign during the early alignment phase. Because of that, the omnidirectionality of the weights $(w_j,\sign(a_j))$ on $\R^d\times\{-1,1\}$ is still preserved when considering unbalanced initialisations in the example of \cref{thm:noconvergence}. The parameters $\theta^t$ thus converge towards a global minimum for unbalanced initialisation with the data of \cref{ass:noconvergenceall}.

We yet believe that this omnidirectionality is not generally preserved with unbalanced initialisation, as the differentiability of the loss $L(\theta)$ is the general argument for the conservation of the omnidirectionality \citep{chizat2018global,wojtowytsch2020convergence} and would still not be satisfied. This intuition is empirically confirmed in \cref{sec:stewart}, where omnidirectionality of the weights is not preserved and the network fails at finding a global minimum,  despite an unbalanced initialisation. 
Theoretically studying examples where unbalanced initialisations fail to find global minima however is more complex and left open for future work.

\medskip

\cref{thm:noconvergence} focuses on a regression task. 
\citet{stewart2022regression} suggest that finding global minima is easier for classification than regression tasks and could be a reason for the empirical success of binning (i.e., recasting a regression task as a classification one).
Whether a similar failure of training is possible with classification tasks is left open for future work. We believe it is indeed possible, since the early alignment phase still happens and can lead to the loss of omnidirectionality, which is also key in classification settings.

Additionally, \cref{thm:noconvergence} is specific to the ReLU activations. This particular choice of activation is crucial to its proof, as it implies that \textit{all} positive neurons either align in a short time towards the extremal vector (point 1 in \cref{lemma:phase1}), or are not relevant (point 2 in \cref{lemma:phase1}). 
Using leaky ReLU activations, neurons in $[m]\setminus(\cI\cup\cN)$ become relevant; and some of them  move at an arbitrarily slow rate during the first phase. As a consequence, omnidirectionality would be preserved during the early alignment phase and convergence towards an interpolator still happens.\footnote{\cref{thm:noconvergence} still holds for leaky ReLU activations if we allow $\tilde\lambda$ to depend on $m$.} We yet allege that failure of convergence can still happen with leaky ReLU activations, yet at the expense of a slightly more intricate data example where omnidirectionality can be lost.

\section{Conclusion}\label{sec:conclusion}

This work provides a first general and rigorous quantification of the early alignment phenomenon, that happens for small intialisations in one hidden layer (leaky) ReLU networks and was first introduced by \citet{maennel2018gradient}. We believe that such a result can be directly used in future works that study the training dynamics of gradient flow for specific data examples.
In particular, we apply this result to describe the training trajectory on a 3 points example. Despite the simplicity of this example, gradient flow fails at learning a global minimum and only converges towards a spurious stationary point (corresponding to the least squares regressor), even with an infinite number of neurons. 
This example thus illustrates the duality of the early alignment phenomenon: although it induces some sparsity in the learnt representation, which is closely related to the implicit bias of gradient flow, it also comes with a risk of failure in minimising the training loss. 



\acks{The authors thank Lena\"ic Chizat for insightful discussions. 
This work was partially funded by an unrestricted gift from Google and by the Swiss National Science Foundation (grant number 212111).}

%

\clearpage
\bibliography{main.bib}

\clearpage
\appendix

\addcontentsline{toc}{section}{Appendix} 
\part{Appendix} 
\parttoc 

\section{More General Assumptions for \cref{sec:noconvergence}}\label{app:assumptions}
In this section, we provide more general assumptions than \cref{ass:noconvergenceall} that still lead to the spurious convergence result of \cref{thm:noconvergence}.

\begin{assumption}\label{ass:noconvergence}
For all $k,k'$, $\langle x_k, x_{k'}\rangle > 0$, $y_k>0$. 
Also, $n\geq d$ and any family $(x_k)_{k}$ with at most $d$ vectors is linearly independent.
\end{assumption}

The goal of this assumption is to have a single extremal vector, so that the early alignment leads to a concentration of (almost) all neurons towards the same direction. We assume the matrix $X\in \R^{n\times d}$ to be full rank with $n\geq d$ for the sake of simplicity. Note that the neurons do not move in the orthogonal space of $\Span(x_1, \ldots, x_n)$ with gradient methods, so that our results can be extended to the case of rank deficient $X$ (or $n<d$).

\begin{assumption}\label{ass:cK}
For any $k\in\cK$, $y_k\|x_k\|^2>\sqrt{\sum_{k'=1}^n y_{k'}^2}\sqrt{\sum_{k'\neq k}^n \langle x_k', x_k \rangle^2}$, where 
\begin{equation*}
\cK \coloneqq \left\{ k\in[n] \mid \exists w\in \R^d\setminus\{\mathbf{0}\}, \langle w,x_k \rangle =0 \text{ and } \forall k'\in[n],\langle w,x_{k'} \rangle \geq 0 \right\}.
\end{equation*}
\end{assumption}

This assumption is needed to assume that in the dynamics of a single neuron network, this neuron stays activated with all the data points ($\langle w,x_k\rangle >0$ for any $k$) along its trajectory. It is needed so that our estimator behaves as a linear regression in the whole second phase of training.

Define in the following for any $u\in\{-1,0,1\}^n$ and for $\beta^*$ as in \cref{thm:noconvergence},
\begin{gather}
\tD_u^{\beta^*} \coloneqq \frac{1}{n} \sum_{k=1}^n \iind{u_k=1} (y_k-\langle \beta^*,x_k\rangle) x_k,\label{eq:Dtilde}\\
\D_u^{\beta^*}\coloneqq\Big\{ \frac{1}{n}\sum_{k=1}^n \eta_k (y_k-x_k^\top\beta^*)x_k \ \Big|\ \forall k\in[n], \eta_k \begin{cases} \in [0,1] \text{ if }\langle w_j^t,x_k\rangle =0 \\ =1 \text{  if }\langle w_j^t,x_k\rangle>0 \\
=0 \text{ otherwise}
\end{cases}\Big\}\notag.
\end{gather}
\begin{assumption}\label{ass:phase3}
For any $k\in\cK$,  $\langle \beta^*, x_k \rangle<y_k$.\\
Moreover, for any $u\in A(\R^d)$, such that $\exists k,k'\in[n], u_k=1$ and $u_{k'}=-1$, there exists $\delta_u>0$ such that both
\begin{enumerate}
\item $\inf_{D\in\D_u^{\beta^*}}u_k\langle D, x_k\rangle>0$ for any $k\in \cK$ such that both $u_k\neq 0$ and $\inf\limits_{w\in A^{-1}(u)} |\langle \frac{w}{\|w\|},\frac{x_k}{\|x_k\|}\rangle|\leq \delta_u$;
\item for any $\delta\in(0,\delta_u)$, $\inf\limits_{w\in A^{-1}_\delta(u)} \langle \tD_u^{\beta^*}, \frac{w}{\|w\|} \rangle >0$, where
\begin{equation*}
A^{-1}_\delta(u)\coloneqq \left\{w\in A^{-1}(u) \mid \exists k,k' \in \cK, \text{ s.t. } \langle \frac{w}{\|w\|},\frac{x_k}{\|x_k\|}\rangle \geq \delta \text{ and }\langle \frac{w}{\|w\|},\frac{x_{k'}}{\|x_{k'}\|}\rangle\leq-\delta \right\}.
\end{equation*}
\end{enumerate}
\end{assumption}

Although technical, this assumption is equivalent to assuming that the least squares regression estimator $\beta^*$ only gets worse, when  a ReLU neuron with a negative output weight is added to it. This ensures that after the second phase of training, where the estimator is arbitrarily close to the least squares, the neurons with negative output weights do not grow in norm. 

In our three points example, the two conditions stated in \cref{ass:phase3} are equivalent and mean that for the two extreme points, $\langle \beta^*, x_k \rangle <y_k$.

\begin{lem}\label{lem:assumptions}
If \cref{ass:noconvergenceall} holds with $\eta<\frac{1}{6}$, then
the dataset $(x_i,y_i)_{i=1,2,3}$ satisfies \cref{ass:noconvergence,ass:cK,ass:phase3}.
\end{lem}

\begin{proof}
1) First check that \cref{ass:noconvergence} holds. Simple computations indeed lead to $\langle x_k, x_{k'} \rangle>0$ if $\eta<1$. Moreover, we indeed have $y_k>0$. Also, $x_1,x_2,x_3$ are pairwise linearly independent for $\eta<1/4$.

\bigskip

2) Some cumbersome but direct computations (that we here skip for sake of readibility) show that if $\eta< \frac{1}{6}$, we can then write
\begin{equation}\label{eq:x2decomp}
\begin{gathered}
x_2 = \alpha_1 x_1 + \alpha_3 x_3\\
\text{with } \alpha_1>\eta \text{ and }\alpha_3>0.
\end{gathered}
\end{equation}
From there, we can choose $w_1\neq\mathbf{0}$ such that $\langle x_1, w_1 \rangle=0$ and  $\langle x_3, w_1 \rangle>0$. \cref{eq:x2decomp} then directly leads to $\langle x_2, w_1 \rangle>0$. By definition, this yields $1\in\cK$. Using similar arguments, $3\in\cK$.

Assume for some $w\neq \mathbf{0}$, $\langle w,x_2 \rangle=0$ with $\langle w,x_1 \rangle\geq 0$ and $\langle w,x_3 \rangle\geq 0$. \cref{eq:x2decomp} then implies that all the products are actually $0$, so that $w=\mathbf{0}$. Necessarily, we thus have $\cK=\{1,3\}$ in the considered example.

\medskip

Let us now check that the inequalities hold for any $k\in\cK$. By definition of the data, we have for $k\in\{1,3\}$:
\begin{align*}
y_k\|x_k\|^2\geq (1+(1-\eta)^2)^2.
\end{align*}
And also
\begin{align*}
\sqrt{\sum_{k'=1}^n y_k^2}\sqrt{\sum_{k'\neq k}\langle x_{k'},x_k\rangle} &\leq \sqrt{\eta^2 + (1+\eta)^2}\sqrt{(\eta(1-\eta)+(1+\eta)^2)^2+(-(1-\eta)^2+(1+\eta)^2)^2}\\
& = \sqrt{1+2\eta + 2\eta^2}\sqrt{(1+3\eta)^2+16\eta^2}.
\end{align*}
A simple comparison then allows to have when $\eta<\frac{1}{6}$:
\begin{align*}
y_k\|x_k\|^2\geq (1+(1-\eta)^2)^2 > \sqrt{\sum_{k'=1}^n y_k^2}\sqrt{\sum_{k'\neq k}\langle x_{k'},x_k\rangle}.
\end{align*}
So \cref{ass:cK} holds.

\bigskip

3) Let's now check \cref{ass:phase3}. Denote in the following for any $i\in[n]$, $r_i\coloneqq x_i^\top\beta^* - y_i$.
By definition of $\beta^*$,
\begin{equation*}
\sum_{i=1}^3 r_i x_i = \mathbf{0}.
\end{equation*}
In particular, since $x_1$ and $x_3$ are linearly independent, it comes:
\begin{equation}\label{eq:interpolation0}
r_1 + \alpha_1 r_2 = 0 \quad \text{ and } \quad r_3+\alpha_3 r_2=0.
\end{equation}
Also, by definition of $r_i$,
\begin{equation*}
r_2 = \alpha_1 r_1 + \alpha_3 r_3 + \Big( \alpha_1 y_1+\alpha_3 y_3 -y_2 \Big).
\end{equation*}
Since $\alpha_1 y_1>y_2$ and $\alpha_3>0$, the term in parenthesis is positive, so that the last equality becomes
\begin{align*}
r_2 > \alpha_1 r_1 + \alpha_3 r_3\\
= - (\alpha_1^2+\alpha_3^2)r_2.
\end{align*}
The last inequality comes from \cref{eq:interpolation0}. Necessarily, this yields $r_2>0$, but \cref{eq:interpolation0} then yields that $r_1<0$ and $r_3<0$, \ie for any $k\in\cK$, $y_k> \langle \beta^*,x_k \rangle$.

Now, consider $u\in A(\R^d)$ such that $\exists k,k'$ with $u_k=1$ and $u_{k'}=-1$. Since $2\not\in\cK$, we can actually choose $k$ and $k'$ both in $\cK$. Assume without loss of generality that $u_1=1$ and $u_3=-1$.
There are now three cases, given by $u_2\in \{-1,0,1\}$. If $u_2=0$, note that the considered cone is just given by a half line. So that for a small enough $\delta_u>0$, there is no $k\in \cK$ such that $\inf\limits_{w\in A^{-1}(u)} |\langle \frac{w}{\|w\|},\frac{x_k}{\|x_k\|}\rangle|\leq \delta_u$. The first point is then automatic in that case.

If instead $u_2=1$, then for a small enough $\delta_u$, $k=3$ is the only $k\in\cK$ satisfying $\inf\limits_{w\in A^{-1}(u)} |\langle \frac{w}{\|w\|},\frac{x_k}{\|x_k\|}\rangle|\leq \delta_u$. From there, note that $\D_u^{\beta^*} = \{\tD_u^{\beta^*}\}$ and
\begin{align}
\tD_u^{\beta^*} & =\frac{1}{n}\sum_{k=1}^2 (y_k-x_k^\top \beta^*)x_k \notag\\
= -\frac{1}{n}(y_3-x_3^\top \beta^*)x_3.
\label{eq:tDcase1}
\end{align}
Since $(y_3-x_3^\top \beta^*)>0$, this indeed yields that $-\langle \tD_u^{\beta^*} ,x_3\rangle >0$, which also yields the first point in that case.

The remaining case is $u_2=-1$. In that case, we similarly only need to check it for $k=1$ and 
\begin{align}
\tD_u^{\beta^*} 
= \frac{1}{n}(y_1-x_1^\top \beta^*)x_1.
\label{eq:tDcase2}
\end{align}
This then also yields the first point.

\medskip

It now remains to check the second point of \cref{ass:phase3} in all $3$ cases. The cases $u_2=0$ and $u_3=-1$ are actually dealt together, since they have the same $\tD_u^{\beta^*}$. In that case, \cref{eq:tDcase2} actually yields that
\begin{align*}
\inf_{w\in A_\delta^{-1}(u)} \langle \tD_u^{\beta^*}, \frac{w}{\|w\|}\rangle & = \frac{1}{n}(y_1-x_1^\top \beta^*) \langle x_1, \frac{w}{\|w\|}  \rangle\\
& \geq \frac{1}{n}(y_1-x_1^\top \beta^*)\delta\|x_1\|>0.
\end{align*}
A similar argument yields in the case $u_2=1$ to
\begin{align*}
\inf_{w\in A_\delta^{-1}(u)} \langle \tD_u^{\beta^*}, \frac{w}{\|w\|}\rangle  \geq \frac{1}{n}(y_3-x_3^\top \beta^*)\delta\|x_3\|>0.
\end{align*}
This concludes the proof.
\end{proof}

\section{Additional Experiments}\label{sec:additionalexpe}


\subsection{Experimental Details}\label{app:expedetails}

In \cref{sec:expe}, we considered the following univariate $3$ points dataset:
\begin{gather*}x_1=-0.75 \text{ and } y_1=1.1; \\
x_2=-0.5 \text{ and } y_2=0.1;\\
x_3=0.125\text{ and } y_3=0.8.
\end{gather*}
The neural network architecture counts bias terms in the hidden layer, which we recall is equivalent to the parameterisation given by \cref{eq:param}, when adding $1$ as the second coordinate to all the data points $x_i$. The activation function is ReLU, the initialisation follows \cref{eq:initscaling} with $\lambda=10^{-3}$ and
\begin{gather*}
\tilde{w}_j \sim \cN(0,I_2), \\
\tilde{a}_j = s_j \|\tilde{w}_j\| \quad \text{with } s_j \sim \cU(\{-1,1\}).
\end{gather*}
We consider a \textbf{perfectly balanced} initialisation for three reasons:
\begin{itemize}
\item it is compatible with the initialisation regime described by \citet{wojtowytsch2020convergence};
\item it yields more interesting observations, since the neurons with $|a_j|>\|w_j\|$ actually align faster with extremal vectors;
\item given the remark on the nature of the stationary point in \cref{sec:discussion}, convergence towards a bad estimator is less obvious in this perfectly balanced case and might be more easily challenged by the use of finite step sizes.
\end{itemize}
Lastly, the neural network is trained with $m=200\ 000$ neurons to approximate the infinite width regime. We ran gradient descent with learning rate $10^{-3}$ up to $2$ millions of iterations. The parameters clearly converged at this point and do not move anymore.

\subsection{Stewart et al. Example}\label{sec:stewart}

\cref{fig:dynamicsstewart} presents the training dynamics on another example of data. The considered example is here borrowed from \citet{stewart2022regression}. Precisely, we consider $40$ univariate data points $x_i$ sampled uniformly at random in $[-1, 1]$. The labels $y_i$ are given by a teacher network closely resembling the one\footnote{We just choose a simpler teacher network here by removing a small neuron in the representation.} of \citet{stewart2022regression}, so that $y_i=f^*(x_i)$ where $f^*$ is a $8$ neurons network.

The considered neural network follows the same parameterisation as given in \cref{app:expedetails}. We believe this is still a large enough width to approximate the infinite width regime here. The only other difference with \cref{app:expedetails} is that the initialisation is here unbalanced; precisely, the weights $\tilde{w}_j$ and $\tilde{a}_j$ are drawn i.i.d. as
\begin{gather*}
\tilde{w}_j \sim \cN(0,I_2), \\
\tilde{a}_j \sim \cN(0, 1).
\end{gather*}
This is an important remark as such an initialisation is closer to the initialisation chosen in practice and does not satisfy \cref{eq:balancedinit}, which is crucial to our analysis. 
Besides providing another, more complex example of convergence towards spurious stationary points, this experiment thus also illustrates that the balancedness condition given by \cref{eq:balancedinit} might not be always necessary to yield convergence towards spurious stationary points.

As opposed to \cref{fig:dynamicsmain}, the figures on the right column do not show the different activation cones and data points $x_i$ here. This is just for the sake of clarity, since there are $40$ points and $80$ activation cones here.

\medskip

Optimisation is hard on this example, as the teacher function counts two little bumps on the left and right slopes. These bumps are hard to learn while optimising, as they are restricted to small regions, i.e. activation cones. Since omnidirectionality is easily lost during training, having no remaining weights in this small region should be enough to get a ``bad'' training (i.e. convergence towards spurious stationary points).
 
\begin{figure}[b!] 
\centering
\begin{subfigure}{0.5\linewidth}
\includegraphics[width=\linewidth, trim=1cm 0.9cm 1.5cm 1cm, clip]{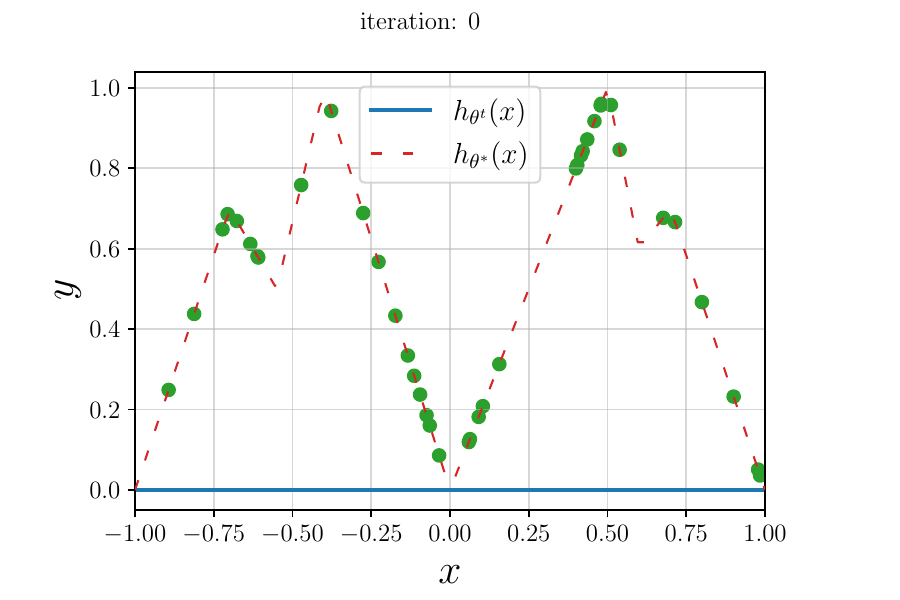}
\caption{\label{fig:stewestim0}Estimated function at initialisation (iteration $0$)}
\end{subfigure}
\hfill
\begin{subfigure}{0.4\linewidth}
      \includegraphics[width=\linewidth,trim=2.5cm 0.6cm 2.5cm 0.53cm, clip]{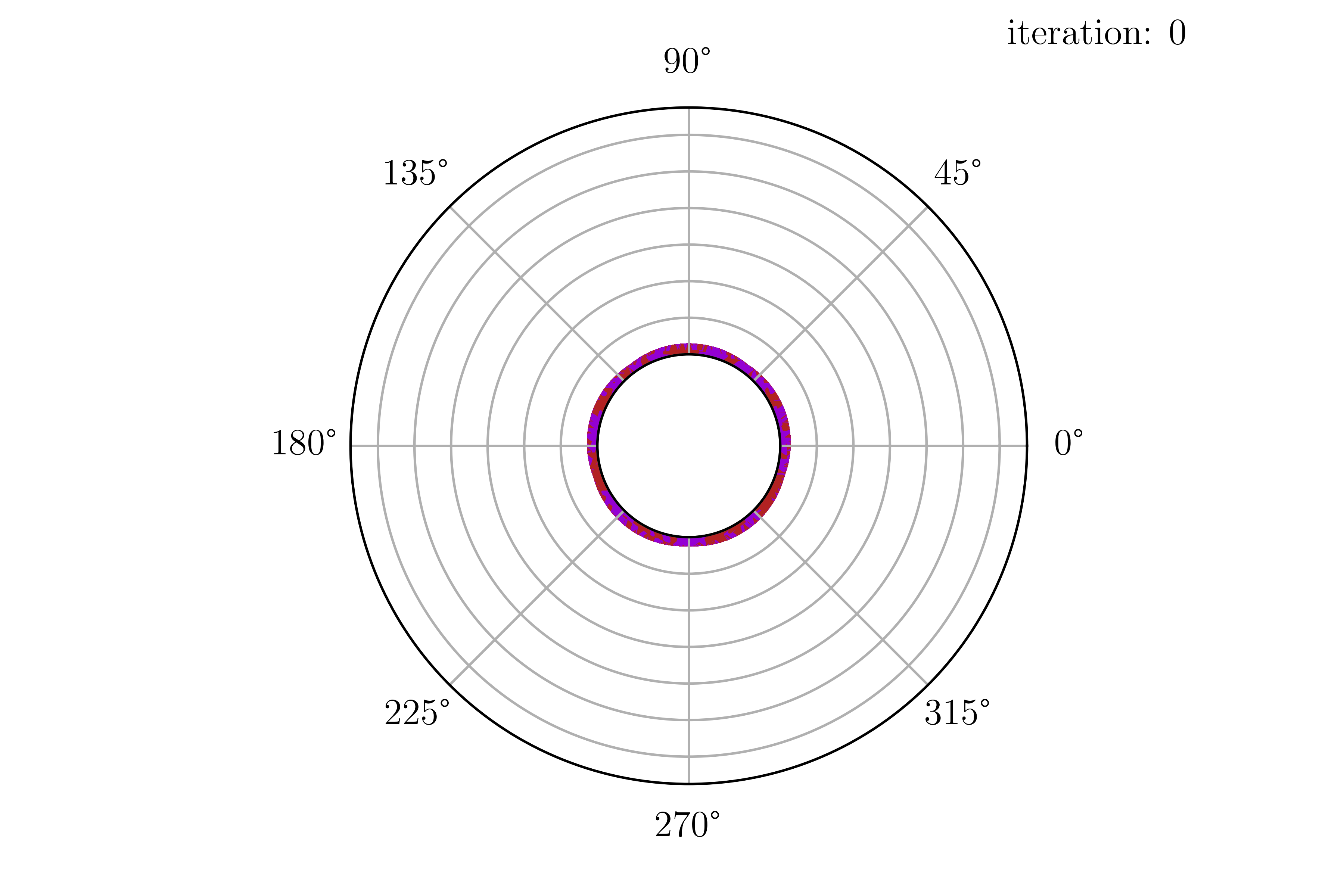}
\caption{\label{fig:stewalignment0}Weights' repartition at initialisation}
\end{subfigure}
\begin{subfigure}{0.5\linewidth}
\includegraphics[width=\linewidth, trim=1cm 0.9cm 1.5cm 1cm, clip]{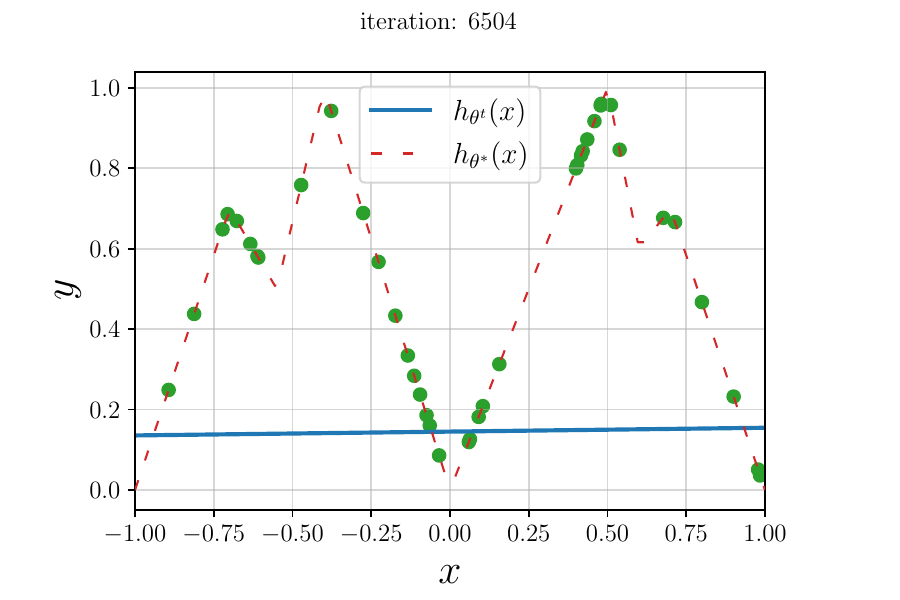}
\caption{\label{fig:stewestim1}Estimated function after early phase (iteration $6500$)}
\end{subfigure}
\hfill
\begin{subfigure}{0.4\linewidth}
      \includegraphics[width=\linewidth,trim=2.5cm 0.6cm 2.5cm 0.53cm, clip]{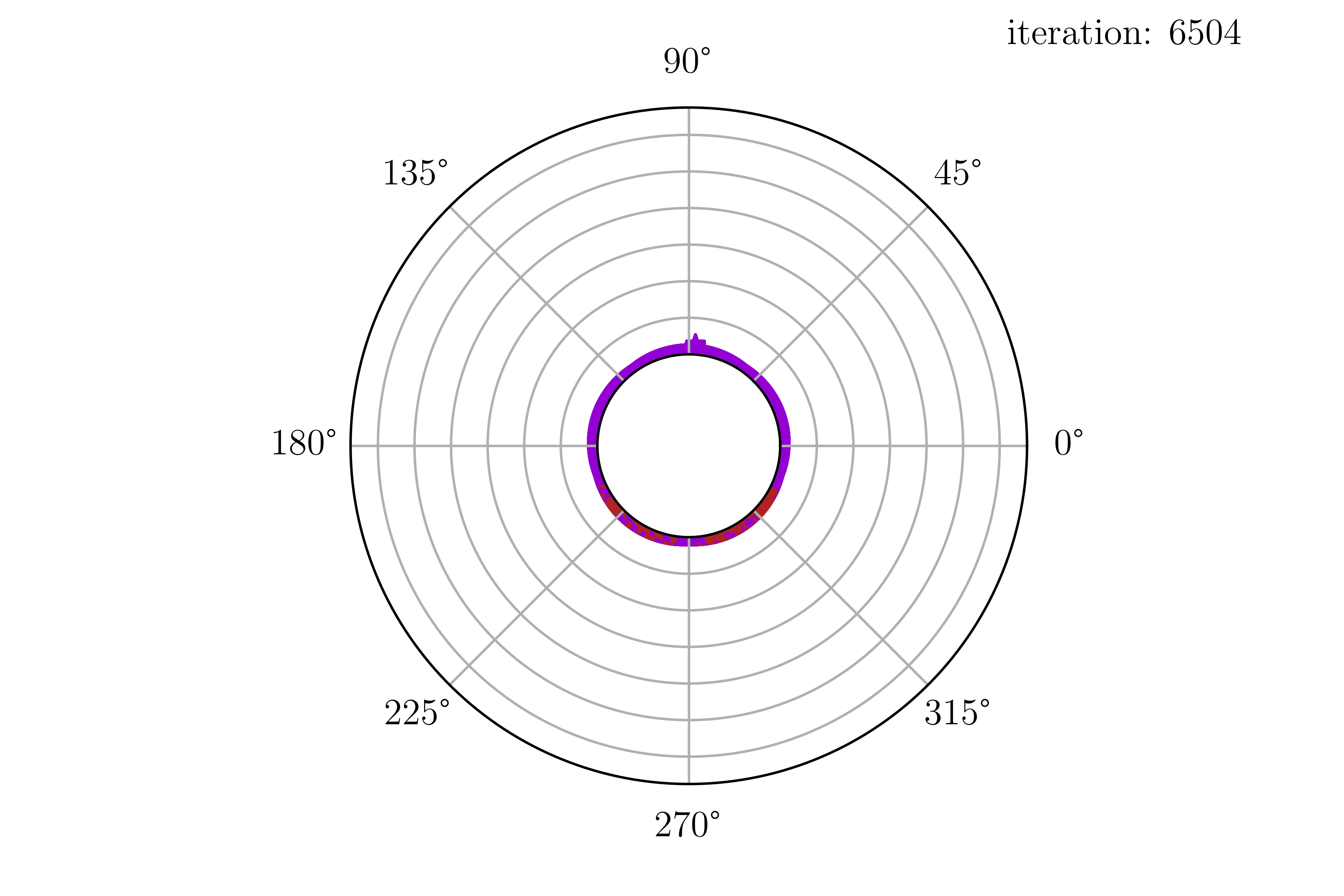}
\caption{\label{fig:stewalignment1}Weights' repartition after early phase}
\end{subfigure}
   %
   
  \caption{\label{fig:dynamicsstewart} Training dynamics on \citet{stewart2022regression} example (part 1/2).}
\end{figure} 

    \begin{figure}[t]
    \addtocounter{figure}{-1}
    \centering
   \begin{subfigure}{0.5\linewidth}
   \addtocounter{subfigure}{4}
\includegraphics[width=\linewidth, trim=1cm 0.9cm 1.5cm 1cm, clip]{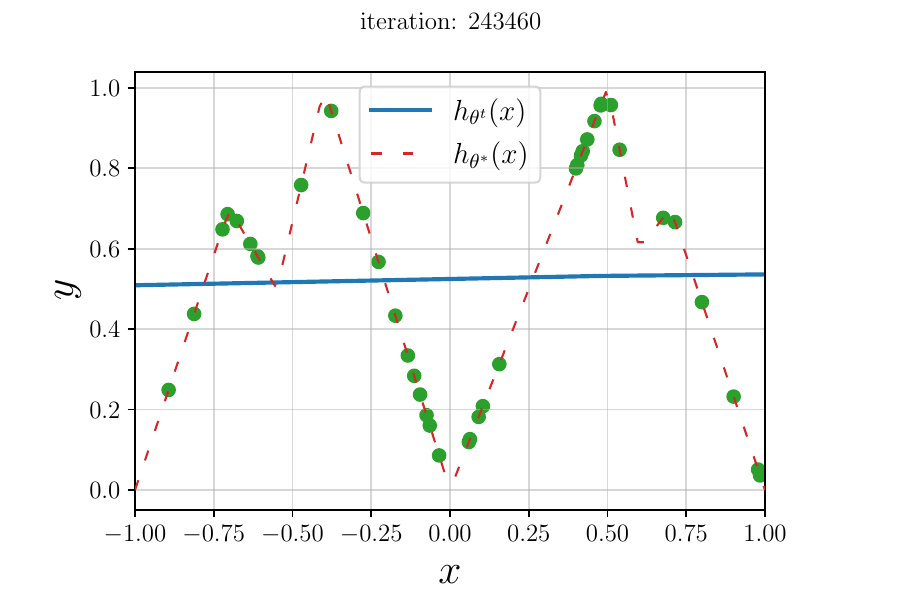}
\caption{\label{fig:stewestim2}Estimated function after first neuron growth (iteration $240\ 000$)}
\end{subfigure}
\hfill
\begin{subfigure}{0.4\linewidth}
      \includegraphics[width=\linewidth,trim=2.5cm 0.6cm 2.5cm 0.53cm, clip]{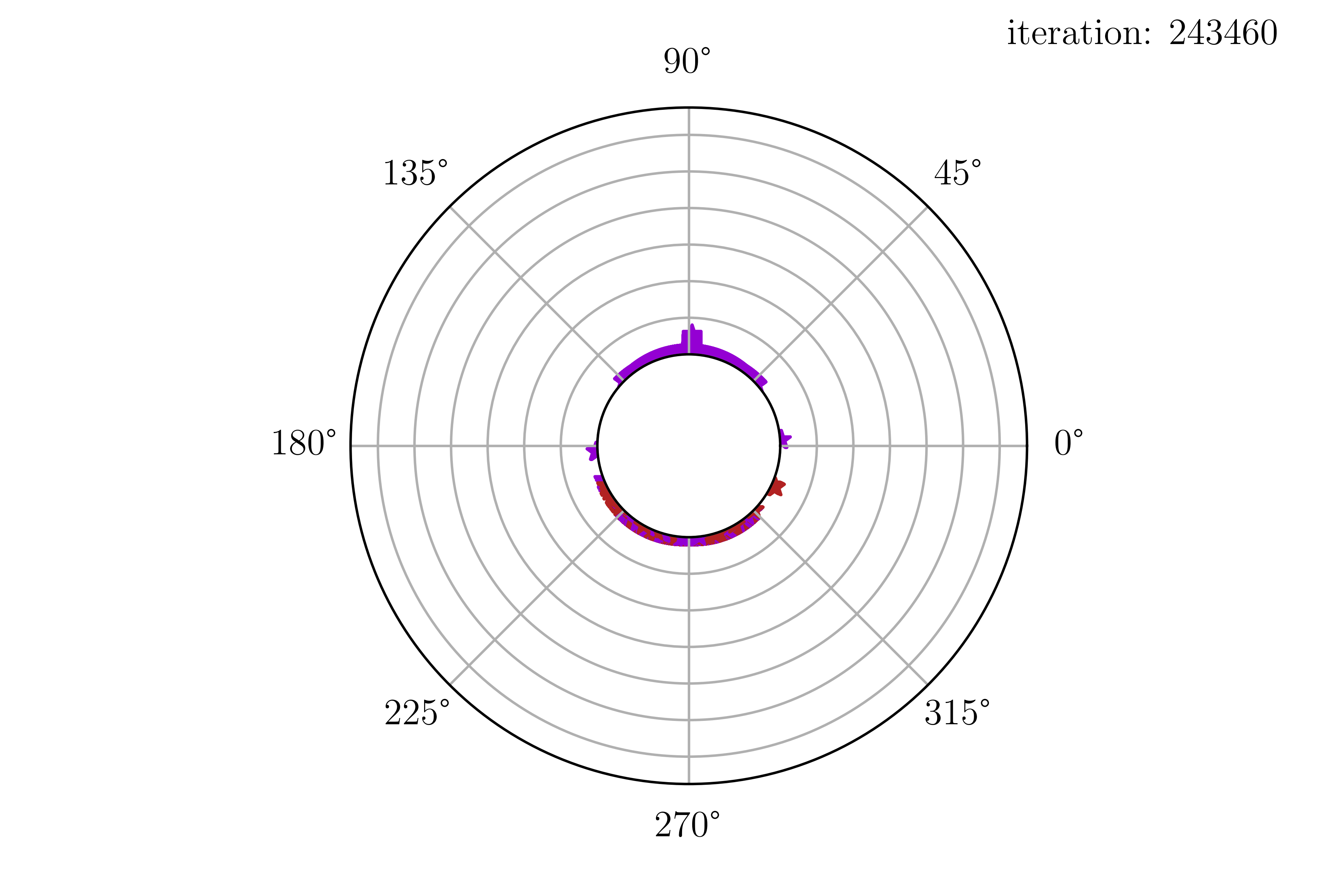}
\caption{\label{fig:stewalignment2}Weights' repartition after first neuron growth}
\end{subfigure}
       \begin{subfigure}{0.5\linewidth}
\includegraphics[width=\linewidth, trim=1cm 0.9cm 1.5cm 1cm, clip]{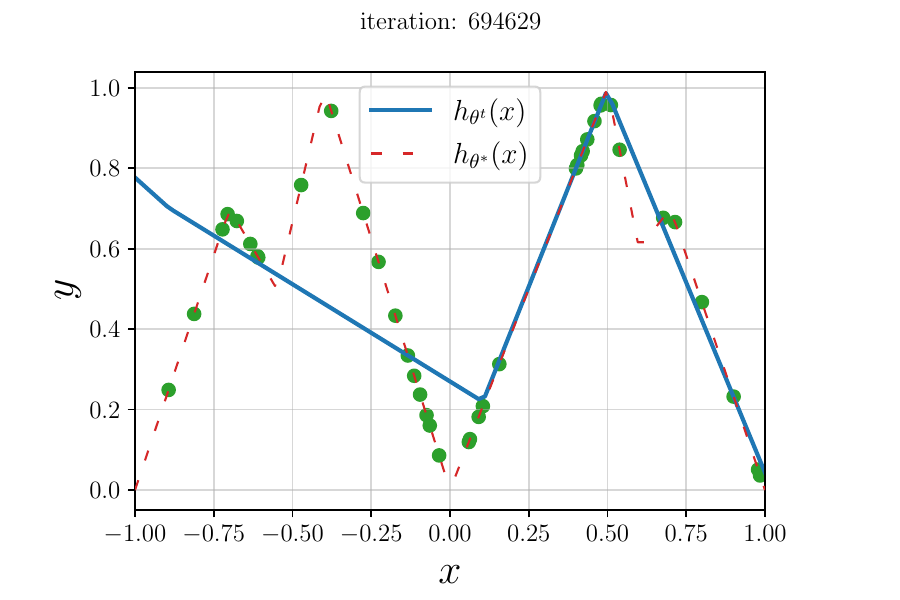}
\caption{\label{fig:stewestim3}Estimated function after second neuron growth (iteration $700\ 000$)}
\end{subfigure}
\hfill
\begin{subfigure}{0.4\linewidth}
      \includegraphics[width=\linewidth,trim=2.5cm 0.6cm 2.5cm 0.53cm, clip]{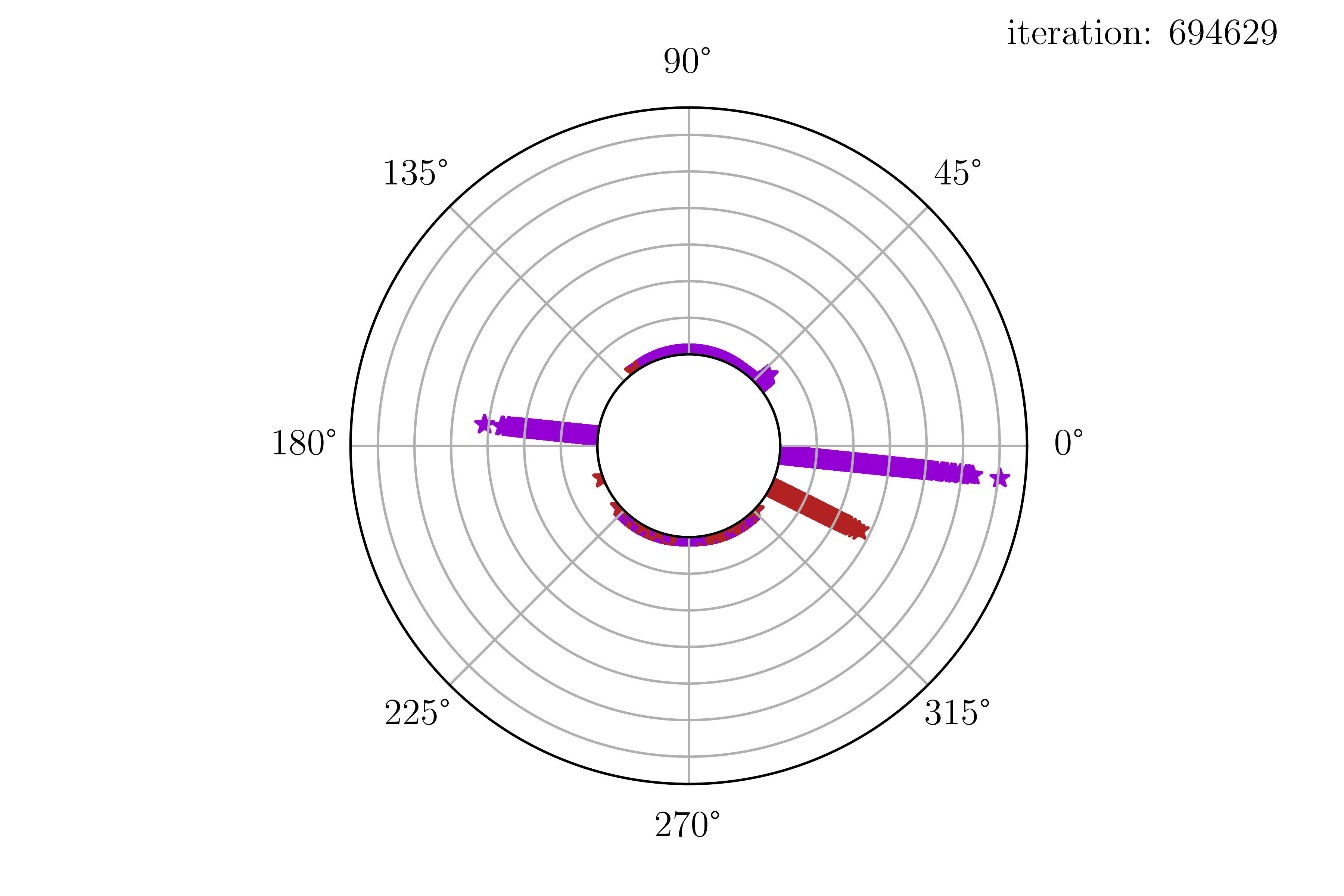}
\caption{\label{fig:stewalignment3}Weights' repartition after second neuron growth}
\end{subfigure}

       \begin{subfigure}{0.5\linewidth}
\includegraphics[width=\linewidth, trim=1cm 0.9cm 1.5cm 1cm, clip]{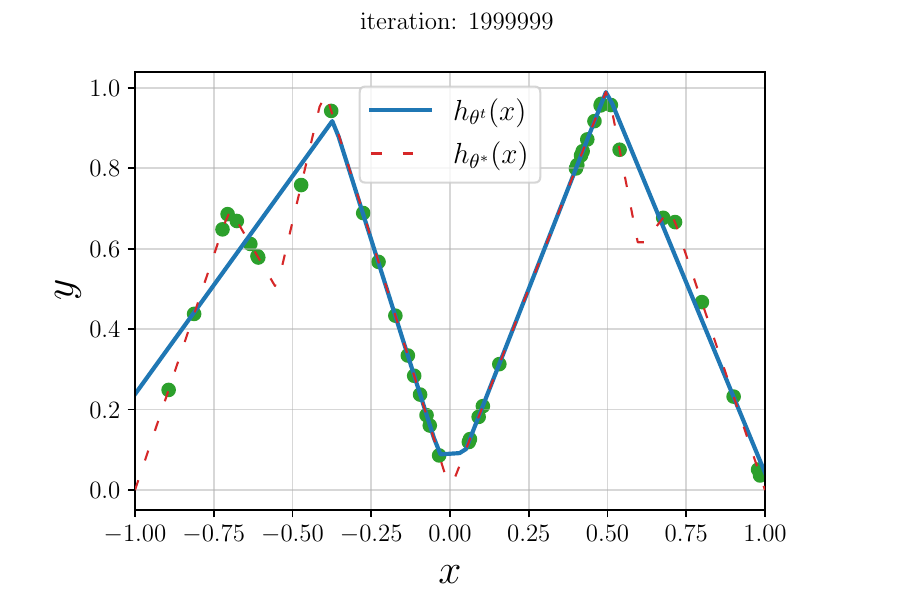}
\caption{\label{fig:stewestim4}Estimated function at convergence (iteration $2\times 10^6$)}
\end{subfigure}
\hfill
\begin{subfigure}{0.4\linewidth}
      \includegraphics[width=\linewidth,trim=2.5cm 0.6cm 2.5cm 0.53cm, clip]{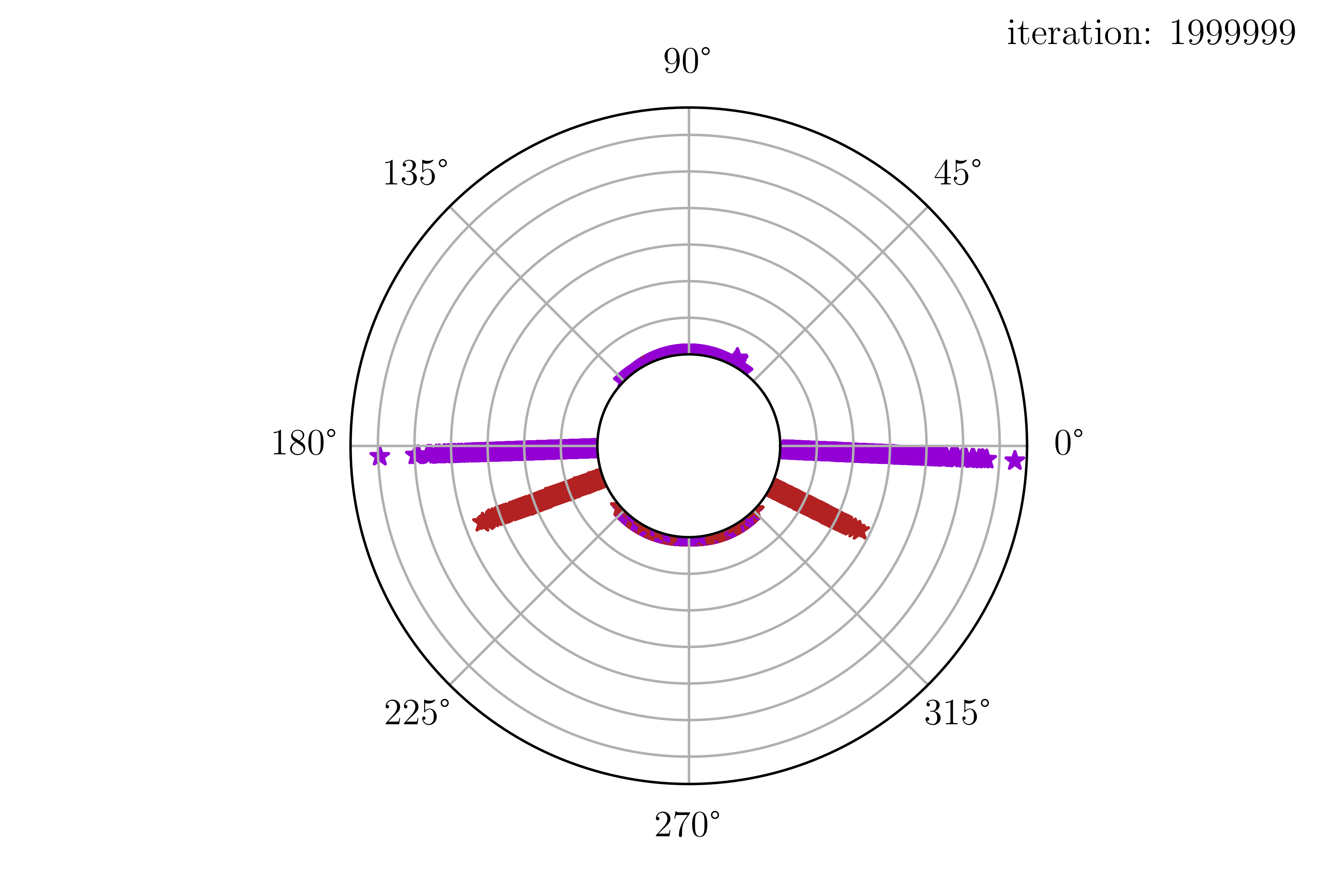}
\caption{\label{fig:stewalignment4}Weights' repartition at convergence}
\end{subfigure}
  \caption{\label{fig:dynamicsstewart2}Training dynamics on \citet{stewart2022regression} example (part 2/2).}
\end{figure}

\cref{fig:dynamicsstewart} illustrates this difficulty: these small cones of interest indeed do not contain any remaining neuron after some point in the training. From there, it becomes impossible to improve the current estimation, leading to convergence towards a spurious stationary point.

\cref{fig:stewestim0,fig:stewalignment0} show the network at initialisation. Due to small scale of initialisation, both neurons and estimated function are nearly $0$ here. 

\cref{fig:stewestim1,fig:stewalignment1} give the learnt parameters at the end of the early phase, just when some neurons start to considerably grow in norm and the estimated function is not zero anymore. A first interesting observation is that there still seems to be omnidirectionality of the weights here, as opposed to \cref{fig:alignment1} for the $3$ points example. This is because the early alignment result of \cref{thm:alignment} only applies to neurons with $|a_j|\geq \|w_j\|$. The alignment of neurons with $|a_j| < \|w_j\|$ can happen at a much slower rate. Such neurons are thus not yet aligned with some extremal vector in \cref{fig:stewalignment1}.

\cref{fig:stewestim2,fig:stewalignment2} show the state of the network after a first neurons growth. At this point, the estimated function is equivalent to a $1$ neuron, constant, network. Also interestingly, \cref{fig:stewalignment2} highlights that omnidirectionality of the weights is already lost at this point of training. The neurons indeed seem to be concentrated in some specific cones, while other cones happen to have no neuron at all. This absence of neurons in key cones is at the origin of failure of training at convergence.

\cref{fig:stewestim3,fig:stewalignment3} illustrates the learnt parameters after the second neurons growth. After this growth, the network is nearly equivalent to a $4$ neurons network (a fifth one is starting to grow). Again, some activation cones do not include any neuron now.

Finally, \cref{fig:stewestim4,fig:stewalignment4} show the state of the network at convergence. The estimated function is now equivalent to a $5$ neurons network and thus fails at fitting all the data that is given by an $8$ neurons network. The neurons that the network fails to learn correspond to the $2$ little bumps mentioned above. At this point the network is unable to learn them, since the cones corresponding to these bumps are empty (of neurons). 
The network thus does not perfectly fit the data in the end. Yet, it gives a good $5$ neurons approximation of the data; which confirms our conjecture made in \cref{sec:discussion} that despite imperfect data fitting, there still is some implicit bias incurred while training.


\section{Proof of Intermediate Lemmas}\label{app:aux}

\subsection{Proof of \cref{lemma:balanced}}

From \cref{eq:ODEs}, it comes that a.e.
\begin{align*}
\frac{\df (a_i^t)^2}{\df t} - \frac{\df \|w_i^t\|^2}{\df t} & \in 2a_i^t\langle w_i^t, \D_i^t\rangle -  2a_i^t\langle w_i^t, \D_i^t\rangle\\
& = 0.
\end{align*}
The last equality comes from the fact that the scalar product does not depend on the choice of the subgradient $D\in\D_i^t$. 
Finally, $(a_i^t)^2-\|w_i^t\|^2$ is constant, which yields \cref{lemma:balanced}.

\subsection{Proof of \cref{lemma:Dj}}

By contradiction, assume that for some $u$ such that $\mathbf{0}\not\in\D_u$, we have some $D\in\D_u \cap \partial\bar{A^{-1}(u)}$. Let $(\eta_k)_{k\in[n]}$ in $[\gamma,1]^n$ such that
\begin{gather*}
D = -\frac{1}{n}\sum_{k=1}^n \eta_k \partial_1\ell(0,y_k) x_k\\
\eta_k = 1 \text{ if }u_k=1, \text{ and }\eta_k=\gamma \text{ if }u_k=-1.
\end{gather*}
Let $K(u)=\{ k \in[n]\mid u_k=0\}$. Also write in the following
\begin{equation*}
\tilde{D}_u = -\frac{1}{n}\sum_{k=1}^n (\iind{u_k=1}+\gamma\iind{u_k=-1}) \partial_1\ell(0,y_k)x_k.
\end{equation*}
Since $D\in \partial\bar{A^{-1}(u)}$, $\langle D, x_k \rangle=0$ for any $k\in K(u)$.
As a consequence, $D = P_{S(u)}(\tilde{D}_u)$, where $P_{S(u)}$ is the orthogonal projection on  $\{x_k \mid k\in K(u)\}^\perp$. In particular, the previous argument allows to show the following
\begin{equation}\label{eq:extremalvector}
D''\in\D_u \text{ is extremal } \implies D''=\argmin_{D'\in\D_u} \|D'\|. 
\end{equation}

Also, $D\in \partial\bar{A^{-1}(u)}$ implies there is at least one $k$ such that $u_k\neq 0$ and $\langle D, x_k \rangle=0$, i.e,
\begin{equation}\label{eq:Djproof}
\partial_1\ell(0,y_k) (\iind{u_k=1}+\gamma\iind{u_k=-1})\|P_{S(u)}(x_{k})\|^2+ \sum_{k'\neq k} \partial_1\ell(0,y_{k'}) (\iind{u_{k'}=1}+\gamma\iind{u_{k'}=-1})\langle x_k,P_{S(u)}(x_{k'}) \rangle = 0.
\end{equation}
If at least one of the $\iind{u_{k'}=1}+\gamma\iind{u_{k'}=-1}$ is non-zero in the sum, then observe that conditionally on $x_k$ and $\{x_{k'} \mid k'\in K(u)\}$ (all non-zero), this sum follows a continuous distribution. Hence, the event given by \cref{eq:Djproof} has $0$ probability.

If instead, all the $\iind{u_{k'}=1}+\gamma\iind{u_{k'}=-1}$ are zero in the sum, then we necessarily have $\iind{u_{k}=1}+\gamma\iind{u_{k}=-1}$ non-zero, as we assumed that $\mathbf{0}\not\in\D_u$. In that case, \cref{eq:Djproof} becomes
\begin{equation*}
\partial_1\ell(0,y_k)\iind{u_{k}=1}\|P_{S(u)}(x_{k})\|^2 = 0.
\end{equation*}
Again, the left term follows a continuous distribution, which leads to the first part of \cref{lemma:Dj}.

\medskip

The last part is immediate given \cref{ass:Dj} and the fact that $\partial_1\ell(0,y_k)$ is non-zero with probability $1$, thanks to \cref{ass:loss}.

\subsection{Proof of \cref{lemma:extremal}}

Consider the maximisation program
\begin{equation*}
\max_{w\in \bS_d} G(w).
\end{equation*}
Since $G$ is continuous and $\bS_d$ is compact, the maximum is reached at some point $w^*$. The KKT conditions at $w^*$ write
\begin{gather*}
\mathbf{0} \in - \partial G(w^*) + \mu w^*,
\end{gather*}
for some $\mu\in\R$. Now note that $\partial G(w^*)=\D(w^*,\mathbf{0})$, so that the KKT conditions rewrite
\begin{equation}
\mu w^* \in \D(w^*,\mathbf{0}).\label{eq:KKT}
\end{equation}
First assume $\mu\neq 0$. For any $v\in\D(w^*,\mathbf{0})$, $\langle w^*, v \rangle$ has the same value. This means that any $v\in\D(w^*,\mathbf{0})$ writes $v=\mu w^* + v^{\perp}$ where $v^{\perp}$ is orthogonal to $w^*$. As a consequence, $D(w^*,\mathbf{0}) = \mu w^*$ by norm minimisation. In particular, $A(D(w^*,\mathbf{0})) \in \{A(w^*), -A(w^*)\}$, so that $D(w^*,\mathbf{0})$ is extremal.

If instead $\mu=0$, then $\mathbf{0}\in\D(w^*,\mathbf{0})$, so that $D(w^*,\mathbf{0})=0$ and $G(w^*)=0$. But then when considering the minimisation program 
\begin{equation*}
\min_{w\in \bS_d} G(w),
\end{equation*}
the same reasoning yields that either there exists an extremal vector, or $\min_{w\in \bS_d} G(w)=0$.
As a conclusion, if there is no extremal vector, $G$ is constant equal to $0$. But then, note that for any $k$
\begin{align*}
0=G(\frac{x_k}{\|x_k\|}) & = \frac{1}{n\|x_k\|}\sum_{k'=1}^n  \partial_1\ell(0,y_{k'}) \sigma(\langle x_k, x_{k'}\rangle) \\
& = \frac{1}{n\|x_k\|}\sum_{k'\neq k}^n \partial_1\ell(0,y_{k'}) \sigma(\langle x_k, x_{k'}\rangle)  + \frac{\partial_1\ell(0,y_{k})}{n}\|x_k\|.
\end{align*}
If \cref{ass:Dj} holds, this equality only holds with probability $0$, yielding \cref{lemma:extremal}.

\section{Proof of \cref{thm:alignment}}\label{app:alignment}

\subsection{Additional Quantities}

Define for any $u\in \{-1,0,1\}^n$,
\begin{equation}\label{eq:Zu}
\begin{gathered}
Z_u \coloneqq \frac{1}{n} (\partial_1\ell(0,y_i) x_i)_{i\in K(u)} \in \R^{d\times |K(u)|},\\
\text{where } K(u) \coloneqq \left\{ k \in [n] \mid u_k=0\right\}.
\end{gathered}
\end{equation}
In words, the columns of $Z_u$ are given by the vectors $\frac{\partial_1\ell(0,y_i) x_i}{n}$ for all $i$ such that $u_i= 0$. Let also in the following say that a set $\D_u$ is \textbf{extremal} if there is some extremal vector $D\in\D_u$, and define for any $u\in A(\R^d)$ and $\theta$,
\begin{gather*}
D_u^{\theta} = D(w,\theta) \quad \text{ for some }w\in A^{-1}(u),\\
\tD_u =  -\frac{1}{n}\sum_{k=1}^n (\iind{u_k=1}+\gamma\iind{u_k=-1}) \partial_1\ell(0,y_k)x_k.
\end{gather*}
We shorten $D_u=D_u^{\mathbf{0}}$. Note that the definition does not depend on the choice of $w\in A^{-1}(u)$.
We define
\begin{gather}\label{eq:delta0}
\delta_0 \coloneqq \min(\delta'_{0},\delta''_{0})\\
\text{where } \delta'_{0} \coloneqq  \min_{u, \D_u \text{ is extremal}} \min_{k, u_k= 0} \sigma_{\min}(Z_u) \min(\eta^u_k-\gamma, 1-\eta^u_k)\notag\\
\text{and }\delta''_{0} \coloneqq \min_{u, \D_u \text{ is extremal}}\min_{k, u_k\neq 0} \frac{|\langle D_u, x_k \rangle|}{\|x_k\|}\notag.
\end{gather}
By convention in the following, we note $\min \emptyset = +\infty$.
In \cref{eq:delta0}, $\sigma_{\min}(Z_u)$ is the smallest singular value of $Z_u$ and $\pmb{\eta}^u$ is the unique vector $\pmb{\eta}\in\R^n$ satisfying\footnote{This vector is indeed unique when $Z_u$ is injective, i.e., when $\sigma_{\min}(Z_u)>0$. When it is not definitely unique, the value is then $0$ as $\sigma_{\min}(Z_u)=0$.} $D_u = \tilde{D}_u + Z_u \pmb{\eta}$.

\medskip

The value of $\lambda_{\alpha_0}^*$ in \cref{thm:alignment} is given by
\begin{gather}
\lambda^*_{\alpha_0}=\Bigg(\min\left(\frac{n}{\sum_{k=1}^n\|x_k\|^2}\min(\frac{\alpha_{\min}^2}{8}D_{\min},\frac{\alpha_0^2}{4}D_{\min},\delta_0); \ \min_{k\in[n]}\frac{|\partial_1\ell(0,y_k)|}{\|x_k\|}\right)\Bigg)^{\frac{1}{2-4\varepsilon}}.\label{eq:lambdastaralpha}\\
\text{where} \quad D_{\min}\coloneqq\min_{u,\mathbf{0}\not\in\D_u}\min_{D\in\D_u}\|D\|;\notag\\
\alpha_{\min}\coloneqq\min(\alpha_{\min,+},\alpha_{\min,-});\label{eq:alphamin}\\
\alpha_{\min,+} \coloneqq \sqrt{1-\left(\max_{\substack{u, \mathbf{0}\not\in\D_u\\ \D_u \cap A^{-1}(u) =\emptyset}}\max_{w \in \bar{A^{-1}(u)}\setminus\{\mathbf{0}\}}   \frac{\langle w, D_u\rangle}{\|w\|\|D_u\|}\right)^2}\\
\alpha_{\min,-} \coloneqq \sqrt{1-\left(\min_{\substack{u, \mathbf{0}\not\in\D_u\\ \D_u \cap -A^{-1}(u) =\emptyset}}\min_{w \in \bar{A^{-1}(u)}\setminus\{\mathbf{0}\}}   \frac{\langle w, D_u\rangle}{\|w\|\|D_u\|}\right)^2}\\
\text{and }\delta_0 \text{ is defined in \cref{eq:delta0}}.\notag
\end{gather}

\cref{lemma:alphamin,lemma:delta0} below imply that $\lambda_{\alpha_0}^*$ for any $\alpha_0>0$.

\subsection{Additional Lemmas}

\begin{lem}\label{lemma:alphamin}
Under \cref{ass:Dj}, the quantities $\alpha_{\min}$ and $D_{\min}$ defined in \cref{eq:alphamin} almost surely satisfy $\alpha_{\min} >0$ and $D_{\min} >0$.
\end{lem}
\begin{proof}
First show that $\alpha_{\min,+}>0$. The defined $\max$ is reached as the supremum of a continuous function on a compact set (the constraint set can indeed be made compact by restricting the problem to the sphere).
By contradiction, if $\alpha_{\min,+}=0$, since the $\max$ is reached, there is $u$ such that $D_u\not\in A^{-1}(u)$ and $\lambda D_u\in \bar{A^{-1}(u)}$ for some $\lambda>0$. This implies that $D_u \in \partial\bar{A^{-1}(u)}$, which contradicts \cref{lemma:Dj}.

The same argument holds to show that $\alpha_{\min,-}>0$, and thus  $\alpha_{\min}>0$

\medskip

For $D_{\min}$, the infimum is also reached as each $\D_u$ is a compact set. By definition, it is necessarily non-zero.
\end{proof}

\begin{lem}\label{lemma:delta0}
Under \cref{ass:Dj}, the quantity $\delta_{0}$ defined in \cref{eq:delta0} almost surely satisfies $\delta_{0} >0$.
\end{lem}
\begin{proof}
There is a finite number of $u$ and $k$ to consider in the $\min$ defined in $\delta'_0$ and $\delta''_0$, so that we only need to prove that the quantity is positive for each $u$ and $k$ in the constraint set. We thus consider $u$ and $k$ in the constraint set in the following.  Thanks to \cref{eq:extremalvector}, we necessarily have $\mathbf{0}\not\in\D_u$.

First, it is easy to show that $\delta''_0>0$. Thanks to \cref{eq:extremalvector} again, $D_u$ is extremal. In particular, for $k$ such that $u_k\neq 0$, $\langle D_u, x_k \rangle \neq 0$, which implies that  $\delta''_0>0$.

It is more technical for $\delta'_0$. Since $u\neq \mathbf{0}$, $\Span\left((x_k)_{k\in K(u)}\right) \neq \R^d$, where $K(u)\coloneqq\{k\in[n]\mid u_k=0\}$.  \cref{lemma:Dj} then implies that $|K(u)|< d$ and the vectors  $(\partial_1\ell(0,y_i) x_k)_{k\in K(u)}$ are linearly independent. By definition of $Z_u$, this directly yields that $Z_u$ is injective, i.e., $\sigma_{\min}(Z_u)>0$ and $\eta^u$ is indeed uniquely defined, and by definition of $D_u$, $\eta_k^u\in[\gamma,1]$.

By contradiction, assume that $\eta_k^u = 1$ for some $k\in K(u)$. Then, we can define $\tilde{u} \in\{-1,0,1\}^n$
\begin{equation*}
\tilde{u}_i = \begin{cases} 1 \text{ if } i=k\\ u_i \text{ if }i\neq k \end{cases}.
\end{equation*}
It then comes that $D_{\tilde{u}}=D_u$. In particular, $D_{\tilde{u}}\in -A^{-1}(u)\cup A^{-1}(u) \subset -\partial\bar{A^{-1}(\tilde{u})}\cup \partial\bar{A^{-1}(\tilde{u})}$. This directly contradicts \cref{lemma:Dj}, so that $\eta_k^u < 1$. A symmetric argument holds to prove that $\eta_k^u > \gamma$. In the end, we proved $\delta'_0>0$, which allows to conclude.
\end{proof}

\begin{lem}\label{lemma:interior}
If \cref{ass:Dj} holds, then with probability $1$, for any $\theta$ such that $\frac{1}{n}\sum_{k=1}^n\|h_{\theta}(x_k)x_k\|_2< \delta_0$ and for any $u\in A(\R^d)$,
\begin{enumerate}
\item $D_u=\mathbf{0}\implies D_u^{\theta} =\mathbf{0}$;
\item $D_u \in A^{-1}(u) \implies D_u^{\theta} \in A^{-1}(u)$;
\item $D_u \in -A^{-1}(u) \implies D_u^{\theta} \in -A^{-1}(u)$.
\end{enumerate}
\end{lem}
\begin{proof}
Let us prove the first two points of \cref{lemma:interior}. The last  part is proven similarly to the second point.

1) Consider some $u\in A(\R^d)$ such that $D_u = \mathbf{0}$ and define $K(u)\coloneqq \{k\in[n] \mid u_k=0 \}$. First recall that for some $\eta_k\in[\gamma,1]$
\begin{equation*}
D_u = -\sum_{k\not\in K(u)} \partial_1\ell(0,y_k)x_k(\iind{u_k=1}+\gamma\iind{u_k=-1}) - \sum_{k\in K(u)} \eta_k\partial_1\ell(0,y_k)x_k
\end{equation*} 

If $(\iind{u_k=1}+\gamma\iind{u_k=-1})\neq 0$ for some $k\not\in K(u)$, \cref{lemma:Dj} implies with probability $1$ that $|K(u)|\geq d$. This then implies $|K(u)|=n$, so that $(\iind{u_k=1}+\gamma\iind{u_k=-1})= 0$ for any $k$. As a consequence, $D_u^{\theta} =\mathbf{0}$.

\medskip

2) Assume that $D_u\in A^{-1}(u)\setminus\{\mathbf{0}\}$.  By definition of $\delta''_0$, for any $k\not\in K(u)$, $|\langle D_u,x_k\rangle|\geq \delta_0 \|x_k\|$. Moreover, note that by $1$-Lipschitz property of $\partial_1\ell$, $$\|D_u^\theta-D_u\|\leq \frac{1}{n}\sum_{k=1}^n \|h_{\theta}(x_k)x_k\|.$$ 
So that when $\frac{1}{n}\sum_{k=1}^n\|h_{\theta}(x_k)x_k\|< \delta_0$, $\langle D_u^\theta, x_k \rangle$ has the same sign than $\langle D_u, x_k \rangle$ for any $k\not\in K(u)$.

\medskip

It now remains to show for any $k\in K(u)$ that $\langle D_u^{\theta}, x_k \rangle=0$.
Let us write $D_u$ as
\begin{gather*}
D_u = \tilde{D}_u - \frac{1}{n}\sum_{k\in K(u)} \eta_k \partial_1\ell(0,y_k) x_k,\\
\text{where }\tilde{D}_u \coloneqq -\frac{1}{n}\sum_{k=1}^n \partial_1\ell(0,y_k)x_k(\iind{u_k=1}+\gamma\iind{u_k=-1})
\text{ and } \eta_k \in [\gamma,1] \text{ for any }k \in K(u).
\end{gather*}
$D_u\in A^{-1}(u)$ implies by definition that $D_u\in \left((x_k)_{k\in K(u)}\right)^{\perp}$, so that $D_u = P_{S(u)}(\tilde{D}_u)$, where $P_{S(u)}$ is the orthogonal projection on the subspace $\left((x_k)_{k\in K(u)}\right)^{\perp}$.

Similarly, define $\tD_u^{\theta}\coloneqq -\frac{1}{n}\sum_{k=1}^n \partial_1\ell(\theta,y_k)x_k(\iind{u_k=1}+\gamma\iind{u_k=-1})$. We can also choose $(\eta'_{k})_{k\in K(u)} \R^{K(u)}$ such that 
\begin{equation*}
P_{S(u)}(D(w,\theta)) - \tilde{D}(w,\theta) = -\frac{1}{n}\sum_{k\in K(u)}\eta'_{k} \partial_1\ell(\theta,y_k) x_k.
\end{equation*}
Note that the $\eta'_k$ are not necessarily in $[\gamma,1]$. Our goal is now to show they are actually in $[\gamma,1]$. Denote for simplicity:
\begin{gather*}
v = P_{S(u)}(\tD_u) - \tD_u \\
v^{\theta} = P_{S(u)}(\tD_u^{\theta}) - \tD_u^{\theta}.
\end{gather*}
Again, by Lipschitz property of the considered functions, it comes
\begin{align*}
\|v - v^{\theta}\|  & \leq \frac{1}{n}\sum_{k=1}^n \|h_{\theta}(x_k)x_k\|<\delta_0.
\end{align*}
Moreover, with the definition of $Z_u$ given by \cref{eq:Zu}, note that
\begin{align*}
\|v - v^\theta\| & = \| Z_u (\eta-\eta')\|\\
&\geq \sigma_{\min}(Z_u) \|\eta-\eta'\|_2\\
&\geq \sigma_{\min}(Z_u) \|\eta-\eta'\|_\infty.
\end{align*}
So overall, it comes
\begin{align*}
\|\eta-\eta'\|_\infty & < \frac{1}{\sigma_{\min}(Z_u)}\delta_0.
\end{align*}
Moreover, by definition of $\delta_0$ (which is positive thanks to \cref{lemma:delta0}), 
\begin{equation*}
\min(\eta_k-\gamma, 1-\eta_k) \geq \frac{\delta_0}{\sigma_{\min}(Z_u)} \quad \text{for any } k\in K(u).
\end{equation*}
This yields that $\eta'_k \in (\gamma,1)$ for any $k\in K(u)$, i.e.
\begin{equation*}
P_{S(u)}(\tD_u^{\theta}) \in \tD_u^{\theta} - \left\{ \frac{1}{n}\sum_{k\in S(u)}\zeta_{k} y_k x_k \mid \zeta_{k} \in (\gamma,1) \text{ for any } k\in K(u) \right\}.
\end{equation*}
By minimization of the norm, this necessarily implies $D_u^{\theta}=P_{S(u)}(\tD_u^{\theta})$. In particular, $\langle D_u^{\theta}, x_k \rangle=0$ for any $k\in K(u)$. This finally yields that $A(D_u^{\theta})=u$ and concludes the proof.

\medskip

3) The last case is proved similarly to the second one.
\end{proof}

\subsection{Proof of \cref{thm:alignment} (i)}\label{app:proofalignmenti}

We prove in this subsection the point (i) of \cref{thm:alignment}. \\
Define the stopping time $t_1 = \min \left\{ t\geq 0 \mid \sum_{j=1}^m (a_j^t)^2 \geq \lambda^{2-4\varepsilon} \right\}$. Thanks to \cref{eq:initialisation}, $t_1>0$.
Moreover \cref{eq:ODEs} yields for any $t\leq t_1$,
\begin{align*}
\left|\frac{\df a_j^t}{\df t}\right| & \leq \|w_j^t\|\  \|D(w_j^t, \theta^t) \|\\
& \leq  \|w_j^t\| \left(D_{\max} + \frac{1}{n}\sum_{k=1}^n \|h_{\theta^t}(x_k)x_k\| \right)\\
& \leq  |a_j^t| \left(D_{\max} + \frac{\lambda^{2-4\varepsilon}}{n}\sum_{k=1}^n \|x_k\|^2 \right).
\end{align*}
The first inequality comes from the fact that $\langle w_j^t, D\rangle$ is independent from the choice of $D\in\D(w_j^t, \theta^t)$. The second one comes from observing that $\|D(w_j^t,\theta^t)- D(w_j^t, \mathbf{0})\|\leq \frac{1}{n}\sum_{k=1}^n \|h_{\theta^t}(x_k)x_k\|$, thanks to the Lipschitz property of \cref{ass:loss}.
Gr\"onwall's inequality then implies for any $t\leq \min(t_1,\tau)$:
\begin{align*}
|a_j^t| \leq |a_j^0| \exp\left(t \left(D_{\max} + \frac{\lambda^{2-4\varepsilon}}{n}\sum_{k=1}^n \|x_k\|^2 \right) \right).
\end{align*}
Plugging the value of $\tau$ yields for $\lambda\leq\lambda^*$,
\begin{align*}
|a_j^t| < |a_j^0| \lambda^{-\varepsilon\left(1+\frac{\lambda^{2-4\varepsilon}}{nD_{\max}}\sum_{k=1}^n \|x_k\|^2\right)}\leq |a_j^0| \lambda^{-2\varepsilon} \quad \text{for any } t<\min(t_1,\tau). 
\end{align*} 
This implies $t_1\geq\tau$. Similarly, we have $|a_j^t| > |a_j^0|\lambda^{2\varepsilon}$ for any $t<\tau$, which yields the first point of \cref{thm:alignment}.

\medskip

The following sections aim at proving point (ii) of \cref{thm:alignment}.  In that objective, we first need to prove several technical intermediate lemmas that are essential to control the neurons trajectory.

\subsection{Local Stability of Critical Manifolds} \label{app:localstability}

For any $u\in A(\R^d)$, we denote $\cM_u \coloneqq A^{-1}(u)$ the activation manifold given by $u$. 
We also say a manifold $\cM_u$ is \textbf{critical} if $D_u\in -\cM_u\cup\cM_u\cup\{\mathbf{0}\}$ and $\cM_u\neq\{\mathbf{0}\}$, i.e. if it is associated to a critical direction of $G$. Similarly to the proof of \cref{lemma:extremal}, a study of the KKT points of $G$ shows that some vector $v\in\cM_u$ is a critical point\footnote{$v$ is of the form $\pm \frac{D_u}{\|D_u\|}$ if $D_u\neq\mathbf{0}$.} of $G$. Due to the absence of saddle points of $G$, such a point is actually a local extremum of $G$. This implies an even stronger stability property of the manifold $\cM_u$.
This observation is described precisely by \cref{lemma:saddles1} below.

Before that, define for any $v\in\R^d$,
\begin{equation*}
\bar{D}_u(v) = \lim_{t\to 0^+} D(w+tv,\mathbf{0}) \quad \text{ for any } w\in A^{-1}(u).
\end{equation*}
The definition of $\bar{D}_u(v)$ is valid as it does not depend on $w\in A^{-1}(u)$. More precisely, we have $\bar{D}_u(v)=D_{u(v)}$, where
\begin{equation*}
u(v)_k = \begin{cases} u_k \text{ if } u_k\neq 0\\
\sign(\langle v,x_k \rangle) \text{ otherwise}. 
\end{cases}
\end{equation*}

\begin{lem}\label{lemma:saddles1}
Assume $G$ does not have any saddle point, \cref{ass:Dj} holds, $\theta$ satisfies the condition of \cref{lemma:interior} and for any $k\in[n]$, $|h_{\theta}(x_k)|<|\partial_1\ell(0,y_k)|$.
Then with probability $1$, for any critical manifold $\cM$, $\exists\varepsilonM\in\{-1,1\}$, such that  $\forall v\in \bar{\cM}^{\perp}$,
\begin{equation*}
\varepsilonM\langle v,D_{u(v)}^\theta \rangle \leq 0.
\end{equation*}
\end{lem}
\cref{lemma:saddles1} states that for any critical manifold $\cM$, the gradients of $G$ around $\bar{\cM}$ either all points toward $\bar{\cM}$, or point outside $\bar{\cM}$. The case $\varepsilonM=1$ corresponds to the former case, in which case a local maximum of $G$ lies in $\cM$. In that case, \cref{lemma:saddles1} implies that the manifold $\bar{\cM}$ is locally stable when running gradient ascent over the function $G$, even with a small perturbation in the function $G$ due to the parameters $\theta$.
Conversely, $\varepsilonM=-1$ corresponds to a local minimum, that is stable when running gradient descent over $G$.

\bigskip

\begin{proof}[Proof of \cref{lemma:saddles1}]
Let $\cM=A^{-1}(u)$ be a critical manifold. Necessarily, there is $\bar{w}\in\bS_d\cap\cM$ a critical point of $G$. By assumption it is either a local minimum or maximum of $G$. Without loss of generality, we assume in the following it is a local maximum. The case of local minimum is dealt with similarly, after a change of sign for $\varepsilonM$.

\medskip

Let us first show that for any $v\in\bar{w}^{\perp}$, $\langle v,D_{u(v)} \rangle\leq0$. Let $v\in\bS_d\cap\bar{w}^{\perp}$ and define for any $\delta>0$,
\begin{equation*}
w(\delta) = \sqrt{1-\delta^2}\bar{w} + \delta v \in \bS_d.
\end{equation*}
Since $\bar{w}$ is a local maximum of $G$, it comes for small enough $\delta$ that $G(w)\geq G(w(\delta))$. Moreover, as $D(w(\delta),\mathbf{0})$ is piecewise constant, it comes for small enough $\delta$ by definition of $\bar{D}_u(v)$ that
\begin{equation*}
G(w(\delta))=\langle w(\delta), \bar{D}_u(v) \rangle.
\end{equation*}
Moreover, by continuity of $G$ it comes $G(\bar{w})=\langle \bar{w}, \bar{D}_u(v) \rangle$, so that
\begin{align*}
0&\geq \langle w(\delta)-\bar{w} , D_{u(v)}\rangle.
\end{align*}
Noting that $w(\delta)-\bar{w} = \delta v +\bigO{\delta^2}$, this necessarily implies that $\langle v , D_{u(v)}\rangle\leq 0$. By a rescaling argument, we thus have for any $v\in\bar{w}^{\perp}$, 
\begin{equation}\label{eq:stability1}
\langle v,D_{u(v)} \rangle\leq0.
\end{equation}
Now assume for $v\in\bar{w}^{\perp}$ that $\langle v,D_{u(v)} \rangle=0$. 
Let $S(u(v))=\Span\{x_k \mid u(v)_k=0\}$. Note that for any $\Delta\in S(u(v))^\perp$, $D_{u(v)}=D_{u(v+t\Delta)}$ for a small enough $t>0$. By \cref{eq:stability1}, we then have for any $\Delta\in \left(S(u(v))\cup\{\bar{w}\}\right)^\perp$,
\begin{align*}
\langle \Delta,D_{u(v)} \rangle\leq0,
\end{align*}
which actually implies $\langle \Delta,D_{u(v)} \rangle=0$ by symmetry, i.e.
\begin{equation*}
D_{u(v)} \in S(u(v))+\R\bar{w}.
\end{equation*}
Necessarily, $D_{u(v)} \in S(u(v))+G(\bar{w})\bar{w}$ since $\langle D_{u(v)},\bar{w} \rangle=G(\bar{w})$.  The KKT conditions at $\bar{w}$ imply that $D_u\in\R\bar{w}$ and again, this implies that $D_u=G(w)\bar{w}$. From there, we then have $\alpha_k$ such that
\begin{equation*}
D_{u(v)}= D_u + \sum_{k, u(v)_k=0}\alpha_k x_k.
\end{equation*}
But by definition of $D_u$ and $D_{u(v)}$, we also have $\eta_k$ and $\eta_k'$ such that
\begin{equation}\label{eq:Duv1}
D_{u(v)} = D_u -\frac{1}{n}\sum_{k, u_k=0} (\eta'_k-\eta_k)\partial_1\ell(0,y_k)x_k.
\end{equation}
In particular, both equations imply
\begin{equation*}
\sum_{k,u_k=0}\left(\alpha_k\iind{\langle v,x_k\rangle=0}+\frac{\eta'_k-\eta_k}{n}\partial\ell_1(0,y_k) \right)x_k = 0.
\end{equation*}
\cref{lemma:Dj} then implies a.s. that either $u_k=0$ for at least $d$ values of $k$; or $\alpha_k\iind{\langle v,x_k\rangle=0}+\frac{\eta'_k-\eta_k}{n}\partial\ell_1(0,y_k)=0$ for any $k$ such that $u_k=0$.

The former yields that $S(u)=\R^d$. Since $\bar{w}\in S(u)^\perp$, this yields $\bar{w}=0$, which contradicts $\bar{w}\in \bS_d$. 
Necessarily, $\alpha_k\iind{\langle v,x_k\rangle=0}+\frac{\eta'_k-\eta_k}{n}\partial\ell_1(0,y_k)=0$ for any $k$ such that $u_k=0$. In particular, for any $k$ such that $u(v)_k\neq 0$, $\eta_k=\eta_k'$. Then note that \cref{eq:Duv1} becomes
\begin{equation*}
D_{u(v)} = D_u -\frac{1}{n}\sum_{k, u(v)_k=0} (\eta'_k-\eta_k)\partial_1\ell(0,y_k)x_k.
\end{equation*}
For any $k$ such that $u(v)_k=0$, $\eta'_k$ is chosen in $[\gamma,1]$ to minimise the norm of $D_{u(v)}$. The choice $\eta'_k=\eta_k$ then minimises the norm, so that $D_{u(v)} = D_u$. In that case \cref{lemma:Dj} necessarily implies that $D_{u(v)}=\mathbf{0}$ (otherwise, we would have $D_{u(v)}=D_u\in-\cM\cup\cM \subset -\partial\bar{A^{-1}(u(v))}\cup\partial\bar{A^{-1}(u(v))}$).

\medskip

We just showed that for any $v\in\bar{w}^\perp$
\begin{equation}
\langle v, D_{u(v)} \rangle <0\quad \text{ or }\quad D_{u(v)}=\mathbf{0}.
\end{equation}
Now consider any $v\in\mathcal{\bar{w}}^{\perp}$. First, if $D_{u(v)}=\mathbf{0}$, then \cref{lemma:interior} directly implies that $D_{u(v)}^{\theta}=\mathbf{0}$. Now assume that $D_{u(v)}\neq\mathbf{0}$, so that $\langle D_{u(v)}, v' \rangle < 0$ for any $v'\in \cS_v$ where
\begin{equation*}
\cS_v \coloneqq \{ v' \in \bar{w}^\perp \mid u(v')=u(v)\}.
\end{equation*}
Let $i\in K(u)$ in the following.
First note that the set $\Span(\{x_k \mid k\in K(u), k\neq i\})$ is of dimension at most $d-2$ with probability $1$, since $\dim\left(\Span(\{x_k \mid k\in K(u)\})\right)\leq d-1$. Indeed, if this last set was $\R^d$, we would have $D_u=\mathbf{0}$. 
In particular, $\left(\{x_k \mid k\in K(u), k\neq i\}\cup\{\bar{w}\}\right)^{\perp}$ is of dimension at least $1$. Moreover, \cref{ass:Dj} along with $\bar{w}\perp K(u)$ implies that a.s.
\begin{equation*}
x_i \not\in \Span(\{x_k \mid k\in K(u), k\neq i\}) + \R \bar{w}.
\end{equation*}
This directly implies that
\begin{equation*}
\left(\{x_k \mid k\in K(u), k\neq i\}\cup\{\bar{w}\}\right)^{\perp} \not\subset \{x_i\}^{\perp}.
\end{equation*}
Therefore, there is some $v_i \in \left(\{x_k \mid k\in K(u), k\neq i\}\cup\{\bar{w}\}\right)^{\perp}$ such that $\langle v_i, x_i \rangle =1$.
Now note that
\begin{align}
0 &\geq \langle v_i , D_{u(v_i)}\rangle\notag\\
& = -\frac{\eta'_i-\eta_i}{n}\partial\ell_1(0,y_i),\label{eq:componentwise}
\end{align}
where
\begin{gather*}
D_u = -\frac{1}{n}\sum_{k=1}^n \eta_k \partial\ell_1(0,y_k)x_k\\
D_{u(v_i)} = -\frac{1}{n}\sum_{k=1}^n \eta'_k \partial\ell_1(0,y_k)x_k.\\
\end{gather*}
\cref{eq:componentwise} comes from the fact that $\eta_k$ and $\eta'_k$ coincide for any $k\not\in K(u)$; for $k\in K(u)$, we then used the fact that $\langle v_i, x_k \rangle = \iind{k=i}$.

In particular, \cref{lemma:delta0} implies that $\eta_i \in (\gamma, 1)$, so that
$\eta'_i-\eta_i=1-\eta_i>0$.  \cref{eq:componentwise} then implies that $\partial\ell_1(0,y_i)\geq 0$. This then holds for any $i\in K(u)$.
By \cref{ass:loss,ass:Dj} the inequality is actually strict. In particular, for any $\theta$ satisfying the inequality in \cref{lemma:saddles1}, it also comes $\partial\ell_1(h_\theta(x_i),y_i)> 0$. Moreover, \cref{lemma:interior} also yields that $D_u^\theta\in-\cM\cup\cM$. This then yields for any $v\in\bar{\cM}^{\perp}$:
\begin{align*}
\langle v, D_{u(v)}^{\theta} \rangle & = -\frac{1}{n} \sum_{\substack{k\in K(u)\\u(v)_k\neq 0}} (\eta'_k(\theta)-\eta_k(\theta))\partial_1\ell(h_\theta(x_k),y_k)\langle v,x_k\rangle.
\end{align*}
Note that $\eta'_k(\theta)-\eta_k(\theta)$ is of the same sign than $\langle v,x_k\rangle$ (or zero). So that finally, every summand is non-negative.

Finally, we just showed that for any $v\in\bar\cM^\perp$ and $\theta$ with $h_\theta$ small enough,
\begin{equation*}
\langle v, D_{u(v)}^{\theta} \rangle \leq 0.
\end{equation*}
\end{proof}

\subsection{Global Stability}\label{app:globalstability}
\begin{lem}\label{lemma:saddles}
If \cref{ass:Dj} holds, the function $G$ does not admit any saddle point and $\lambda<\lambda_{1}^*$, then for any $j\in[m]$ and $t\in[0,\tau]$, 
\begin{equation*}
a_j^{t}D(w_j^t,\theta^t)\in A^{-1}(A(w_j^t))\cup\{\mathbf{0}\} \text{ and } w_j^t\neq \mathbf{0}\implies w_j^{t'}\in A^{-1}(A(w_j^t)) \text{ for any }t'\in[t,\tau].
\end{equation*}
\end{lem}

\begin{proof}
Recall that $\lambda_1^*$ is given by \cref{eq:lambdastaralpha} for $\alpha_0=1$. 
The first point of \cref{thm:alignment} holds thanks to \cref{app:proofalignmenti}. Thus, for any $t\leq\tau$, $\theta^t$ satisfies the conditions of \cref{lemma:saddles1}, thanks to our choice of $\lambda$.
Let $j\in[m]$ and $t_0\in[0,\tau]$ such that $a_j^{t_0}D(w_j^{t_0},\theta^{t_0})\in A^{-1}(A(w_j^{t_0}))\cup\{\mathbf{0}\}$ and $w_j^t\neq \mathbf{0}$. For $u\coloneqq A(w_j^{t_0})$, note that $\cM=\cM_u$ is a critical manifold by definition. Assume in the following that $a_j^{t_0}>0$. The negative case is dealt with similarly. 

The definition of $\alpha_{\min}$ implies if $D_u$ is not extremal that
\begin{equation*}
\min_{w\in\bar{A^{-1}(u)}}1-\frac{\langle w, D_u \rangle^2}{\|w\|^2\|D_u\|^2} \geq \alpha_{\min}^2.
\end{equation*}

First, note that when $\frac{1}{n}\sum_{k=1}^n\|h_\theta(x_k)x_k\|\leq \delta$, 
\begin{equation*}
\left\| \frac{D_u^{\theta}}{\|D_u^{\theta}\|}-\frac{D_u}{\|D_u\|} \right\| \leq \frac{2\delta}{\|D_u\|-\delta}
\end{equation*}
As a consequence, our choice of $\lambda$ would imply if\footnote{We here use $\frac{1}{n}\sum_{k=1}^n\|x_k\|^2 \lambda^{2-4\varepsilon}\leq \frac{1}{8}D_{\min}\alpha_{\min}^2$.} $D_u$ was not extremal that
\begin{equation*}
\min_{w\in\bar{A^{-1}(u)}}1-\frac{\langle w, D_u^{\theta} \rangle^2}{\|w\|^2\|D_u^{\theta}\|^2} >0,
\end{equation*}
which contradicts that $D(w_j^{t_0},\theta^{t_0})\in A^{-1}(A(w_j^{t_0}))\cup\{\mathbf{0}\}$. Necessarily, for our choice of $\lambda$ and thanks to \cref{lemma:interior},
\begin{align*}
D_u^{\theta} \in A^{-1}(u) & \implies D_u \in A^{-1}(u) \\
& \implies D_u^{\theta'} \in A^{-1}(u).
\end{align*}
and similarly with $D_u^{\theta} \in -A^{-1}(u)$ and $D_u^{\theta}=\mathbf{0}$.
As a consequence, for any $t\in[t_0,\tau]$, it comes that $a_j^{t}D(w_j^{t_0},\theta^t)\in A^{-1}(A(w_j^{t_0}))\cup\{\mathbf{0}\}$.

\medskip

If $t_0=0$, then with probability $1$, $\cM$ is an open manifold of dimension $d$, i.e. $u_k\neq0$ for all $k$. The inclusion $a_j^{t}D(w_j^{t_0},\theta^{t})\in A^{-1}(A(w_j^{t_0}))\cup\{\mathbf{0}\}$ then directly implies that the value of $|\langle x_k, w_j^t\rangle|$ is increasing over time on $[t_0, \tau]$ as
\begin{align*}
\frac{\df \langle x_k, w_j^t\rangle}{\df t} = a_j^t \langle x_k, D(w_j^{t_0},\theta^{t})\rangle.
\end{align*}
Thus for any $t\in[t_0, \tau]$, $w_j^t\in A^{-1}(A(w_j^{t_0})\cup\{\mathbf{0}\}$.

Now assume that we cannot choose $t_0=0$, i.e. $a_j^{0}D(w_j^{0},\theta^0)\not\in A^{-1}(A(w_j^0))\cup\{\mathbf{0}\}$. The distance of $w_j^t$ to $\bar{\cM}$ evolves as follows a.e., with $\PM$ the orthogonal projection on $\Span(\bar{\cM})$,

\begin{align*}
\frac{\df d(w_j^t,\Span(\bar{\cM}))}{\df t} & = 2\langle \frac{\df (w_j^t - \PM(w_j^t))}{\df t},w_j^t - \PM(w_j^t)\rangle\\
& = 2 \langle \frac{\df w_j^t}{\df t},w_j^t - \PM(w_j^t)\rangle 
\end{align*}
The second equality comes from the fact that $\frac{\df  \PM(w_j^t)}{\df t}\in\Span(\bar{\cM})$, while $ w_j^t - \PM(w_j^t)\in \bar{\cM}^{\perp}$.

Since $w_j^{t_0}\in\cM$ and $w_j^t$ is continuous in time, there exists $\eta>0$ such that for any $t\in[t_0-\eta, t_0+\eta]$, $\{k\mid \langle x_k , w_j^t \rangle = 0\}\subset \{k \mid u_k=0\}$, where we recall $\cM=\cM_u$.
As a consequence, $\frac{\df w_j^t}{\df t}\in D(w_j^t,\theta^t) + \Span(\{x_k\mid \langle x_k, w_j^t\rangle=0\})\subset D(w_j^t,\theta^t) + \bar{\cM}^{\perp}\cap \{w_j^t\}^{\perp}$. This then implies
\begin{align*}
\frac{\df d(w_j^t,\Span(\bar{\cM}))}{\df t} & = 2 \langle \frac{\df w_j^t}{\df t},w_j^t - \PM(w_j^t)\rangle \\
&= 2 \langle D(w_j^t,\theta^t),w_j^t - \PM(w_j^t)\rangle .
\end{align*}
Moreover by continuity, we can also choose $\eta>0$ such that for any $t\in[t_0-\eta, t_0+\eta]$, $D(w_j^t,\theta^t)=D_{u(v^t)}^{\theta^t}$ where $v^t=w_j^t - \PM(w_j^t)\in \bar{\cM}^{\perp}$. 
\cref{lemma:saddles1} then implies a.s. for any $t\in[t_0-\eta, t_0+\eta]$:
\begin{equation}\label{eq:manifolddist}
\varepsilonM \frac{\df d(w_j^t,\Span(\bar{\cM}))}{\df t}\leq 0.
\end{equation}
We can choose $\eta>0$ so that $w_j^{t_0-\eta}\not\in\Span(\bar{\cM})$. In that case, \cref{eq:manifolddist} implies that\footnote{The symmetric case $a_j^t<0$ would here result in $\varepsilonM=-1$.} $\varepsilon_{\cM}=1$. The $\varepsilonM=-1$ case would indeed have resulted in a growth of $d(w_j^t,\Span(\bar{\cM}))$ on $[t_0-\eta,t_0]$, contradicting the fact that it is non zero at $t_0-\eta$ and zero at $t_0$.

As a consequence, for any $t_1\leq \tau$ such that $w_j^{t_1}$, we have for similar reasons that $d(w_j^t,\Span(\bar{\cM}))$ is non-increasing on $[t_1,t_1+\eta]$, so that $w_j^t \in \Span(\bar{\cM})$ for any $t\in[t_1,t_1+\eta]$ for a small enough $\eta$. Moreover, note that $a_j^{t}D(w_j^{t_0},\theta^{t})\in A^{-1}(A(w_j^{t_0}))\cup\{\mathbf{0}\}$ implies here again that the values of $|\langle x_k, w_j^t\rangle|$ are non-decreasing for any $k$ as long as $w_j^t\in\cM$ and $t\leq \tau$.
The last two points thus imply that the conditions $\langle w_j^t, x_k \rangle=0$ are stable and the conditions $u_k \langle w_j^t, x_k \rangle>0$ are enforced as long as $w_j^t\in\cM$ and $t\leq \tau$. This concludes the proof of \cref{lemma:saddles}.
\end{proof}

\subsection{Quantization of Misalignment}

\begin{lem}\label{lemma:quantization}
For  any critical manifold $\cM_u$ such that $D_u\neq\mathbf{0}$, if $w_u\in\cM_u\cap \bS_d$ is a local maximum of $G$, then either $D_u\in\cM_u$ or $\cM_u\cap\bS_d = \{-\frac{D_u}{\|D_u\|}\}$.

Similarly for  any critical manifold $\cM_u$ such that $D_u\neq\mathbf{0}$, if $w_u\in\cM_u\cap \bS_d$ is a local minimum of $G$, then either $D_u\in-\cM_u$ or $\cM_u\cap\bS_d = \{\frac{D_u}{\|D_u\|}\}$.
\end{lem}

\begin{proof}
Consider the first case of \cref{lemma:quantization}.
The KKT condition of \cref{eq:KKT} at $w_u$ then yields  that $\mu w_u=D_u$. Since $D_u\neq\mathbf{0}$, this necessarily yields that $w_u=\pm\frac{D_u}{\|D_u\|}$. If $w_u=\frac{D_u}{\|D_u\|}$, we then simply have $D_u\in\cM_u$.
Now assume that $w_u=-\frac{D_u}{\|D_u\|}$. Note that $w_u$ is a local maximum of $G$ and $G(w_u)=-\|D_u\|$ in that case. 
Now observe that if $\cM_u\cap\bS_d \neq \{-\frac{D_u}{\|D_u\|}\}$, we can choose $w\in\cM_u\cap\bS_d$ that is not perfectly aligned with $-D_u$. In particular 
\begin{equation*}
G(w) = \langle w,D_u \rangle > -\|D_u\|.
\end{equation*}
Necessarily, this yields that $\cM_u\cap\bS_d = \{-\frac{D_u}{\|D_u\|}\}$.

\medskip

The second point of \cref{lemma:quantization} is proved with symmetric arguments.
\end{proof}

\subsection{Proof of \cref{thm:alignment} (ii)} \label{app:proofalignmentii}

Define for the remaining of the proof $\w_j^t = \frac{w_j^t}{a_j^t}$ and assume both $a_j^0>0$ and $j$ satisfies \cref{cond:neurons}. We then have $a_j^t>0$ for any $t$ thanks to \cref{lemma:balanced} (and a simple Gr\"onwall argument if $\|w_j^0\|=|a_j^0|$). The symmetric case ($a_j^0<0$) is dealt with similarly. Note that $\w_j^t\in B(0,1)$. Moreover, we have for any $j\in[m]$ the differential inclusion
\begin{equation}\label{eq:ODEdirection}
\frac{\df \w_j^t}{\df t} \in \D_j^{t} - \langle \w_j^t, \D_j^{t} \rangle \w_j^t. 
\end{equation}
Here again, the set $\langle \w_j^t, \D(w_j^t,\theta) \rangle$ is a singleton for any $\theta$, i.e. the scalar product does not depend on the choice of the vector $D\in\D(w_j^t,\theta)$, so that \cref{eq:ODEdirection} rewrites for any $D,D'\in\D(w_j^t, \mathbf{0})$:
\begin{equation*}
\frac{\df \langle\w_j^t,D\rangle}{\df t} = \frac{\df \langle\w_j^t,D'\rangle}{\df t} \in \langle\D_j^{t},D'\rangle - \langle \w_j^t, D(w_j^t,\theta^t) \rangle \langle \w_j^t,D\rangle. 
\end{equation*}
$D(w_j^t,\mathbf{0})$ is piecewise constant, so that $\frac{\df D(w_j^t,\mathbf{0})}{\df t}=0$ almost everywhere. Also, the quantity $\langle \w_j^t, D(w_j^t,\mathbf{0}) \rangle$ is absolutely continuous. In particular, for almost any $t\leq \tau$ and $D'\in\D(w_j^t, \mathbf{0})$, 
\begin{align*}
\frac{\df \langle \w_j^t, D(w_j^t,\mathbf{0}) \rangle}{\df t} & \in  \langle D',\D_j^t \rangle - \langle \w_j^t, D(w_j^t,\mathbf{0})\rangle\langle \w_j^t, D(w_j^t,\theta^t)\rangle.
\end{align*}
Note that for any $D\in\D_j^t$, there exists $D'\in\D(w_j^t,\mathbf{0})$ such that $\|D-D'\|\leq \frac{1}{n}\sum_{k=1}^n\|h_{\theta^t}(x_k)x_k\|$, which yields
\begin{align*}
\frac{\df \langle \w_j^t, D(w_j^t,\mathbf{0}) \rangle}{\df t} & \geq \max_{D'\in\D(w_j^t,\mathbf{0})}\min_{D\in\D_j^t}\langle D,D'\rangle - \langle \w_j^t, D(w_j^t,\mathbf{0})\rangle\langle \w_j^t, D(w_j^t,\theta^t)\rangle\\
& \geq \min_{D'\in\D(w_j^t,\mathbf{0})}\left(\|D'\|^2-\frac{\|D'\|}{n}\sum_{k=1}^n\|h_{\theta^t}(x_k)x_k\|\right)- \langle \w_j^t, D(w_j^t,\mathbf{0})\rangle\langle \w_j^t, D(w_j^t,\theta^t)\rangle.
\end{align*}
Moreover, thanks to the first part of \cref{thm:alignment}
\begin{align}
\frac{\df \langle \w_j^t, D(w_j^t,\mathbf{0}) \rangle}{\df t}&\geq \min_{D'\in\D(w_j^t,\mathbf{0})}\left(\|D'\|^2-\frac{\|D'\|}{n}\lambda^{2-4\varepsilon}\sum_{k=1}^n\|x_k\|^2\right)\notag\\
&\phantom{\geq}- \langle \w_j^t, D(w_j^t,\mathbf{0})\rangle^2 - \frac{\|D(w_j^t,\mathbf{0})\|}{n}\lambda^{2-4\varepsilon}\sum_{k=1}^n\|x_k\|^2\notag\\ 
& \geq  \|D(w_j^t,\mathbf{0})\|^2 - \langle \w_j^t, D(w_j^t,\mathbf{0})\rangle^2 - \frac{2}{n}\|D(w_j^t,\mathbf{0})\|\lambda^{2-4\varepsilon}\sum_{k=1}^n \|x_k\|^2.\label{eq:ODEalign1}
\end{align}
The last inequality comes from $\lambda<\lambda_{\alpha_0}^*$, which yields that the minimum is reached for the minimal value of $\|D'\|$ in the considered set.

Let $u_j^t = A(w_j^t)$ and $D_{u_j^t}\coloneqq D(w_j^t,\mathbf{0})\in\D_{u_j^t}$ in the following. 
If $D_{u_j^t}\not\in A^{-1}(u_j^t)\cup \{\mathbf{0}\}$, the definition of $\alpha_{\min,+}$ (which is positive thanks to \cref{lemma:alphamin}), given by \cref{eq:alphamin}, implies that $$\langle \w_j^t,D(w_j^t,\mathbf{0})\rangle \leq \|D(w_j^t,\mathbf{0})\| \sqrt{1-\alpha_{\min}^2}.$$ 
\cref{eq:ODEalign1} implies that as long as $\|D(w_j^t,\mathbf{0})\|^2 - \langle \w_j^t, D(w_j^t,\mathbf{0})\rangle^2 \geq \frac{2}{n}\|D(w_j^t,\mathbf{0})\|\lambda^{2-4\varepsilon}\sum_{k=1}^n \|x_k\|^2$, the quantity $\langle \w_j^t, D(w_j^t,\mathbf{0}) \rangle$ is increasing. In particular, the choice of $\lambda_{\alpha_0}^*$ makes it increasing as long as $D_{u_j^t}\not\in A^{-1}(u_j^t)\cup\{\mathbf{0}\}$ and $\langle \w_j^t, D_{u_j^t} \rangle\geq -\sqrt{1-\min(\alpha_{\min},\alpha_0)^2} \|D(w_j^t,\mathbf{0}\|$. Note that this inequality holds at initialisation. Moreover when entering a new activation cone, either $\langle \w_j^t, D_{u_j^t} \rangle\geq -\sqrt{1-\min(\alpha_{\min,-},\alpha_0)^2} \|D_{u_j^t}\|$ or $D_{u_j^t}\in-A^{-1}(u_j^t)\cup\{\mathbf{0}\}$. 
We thus have by continuous induction that as long as $D_{u_j^t}\not\in A^{-1}(u_j^t)\cup\{\mathbf{0}\}$ and $w_j^t$ does not \textbf{enter} an activation cone such that $D_{u_j^t}\in-A^{-1}(u_j^t)$:
\begin{equation*}
-\|D(w_j^t,\mathbf{0})\|\sqrt{1-\min(\alpha_{\min},\alpha_0)^2}\leq \langle \w_j^t, D_{u_j^t} \rangle\leq \|D(w_j^t,\mathbf{0})\| \sqrt{1-\alpha_{\min}^2}.
\end{equation*}
\cref{eq:ODEalign1} and our choice of $\lambda$ then imply that as long as $D_{u_j^t}\not\in A^{-1}(u_j^t)\cup\{\mathbf{0}\}$ and $w_j^t$ does not \textbf{enter} an activation cone such that $D_{u_j^t}\in-A^{-1}(u_j^t)$:
\begin{equation*}
\frac{\df \langle \w_j^t,D(w_j^t,\mathbf{0}) \rangle}{\df t} \geq \frac{D_{\min}^2}{2}\min(\alpha_{\min},\alpha_0)^2.
\end{equation*}
Since $\langle \w_j^t,D(w_j^t,\mathbf{0}) \rangle$ is bounded in absolute value by $D_{\max}$, there is a time $t_2$ bounded as $t_2\leq\frac{4D_{\max}}{D_{\min}^2\min(\alpha_{\min},\alpha_0)^2}=\Theta_{\alpha_0}(1)$ such that
\begin{gather}
D_{u_j^{t_2}}\in A^{-1}(u_j^{t_2})\cup\{\mathbf{0}\}\label{eq:criticalmanifold1}\\
\text{or } D_{u_j^{t_2}}\in -A^{-1}(u_j^{t_2}) \text{ and } u_j^{t_2}\neq u_j^{0}.\label{eq:criticalmanifold2}
\end{gather}
Note that either $\lambda^{\varepsilon}=\Omega_{\alpha_0}(1)$ or $t_2\leq\tau$. The first case trivially yields the second point of \cref{thm:alignment}, so that we assume from now that $t_2\leq\tau$.

In the case of \cref{eq:criticalmanifold2}, an argument similar to the proof of \cref{lemma:saddles} yields for $\cM=A^{-1}(u_j^{t_2})$ that $\varepsilonM=1$, where $\varepsilonM$ is defined by \cref{lemma:saddles1}. This corresponds to the case where a local maximum of $G$ lies in $\cM$. But since $D_{u_j^{t_2}}\in -A^{-1}(u_j^{t_2})$, \cref{lemma:quantization} implies that $\cM=\R_-^* D_{u_j^{t_2}}$. 

Similarly to the proof of \cref{lemma:saddles}, the conditions $\langle w_j^t,x_k \rangle=0$ are then stable as long as $w_j^t\in A^{-1}(u_j^{t_2})$. But since $w_j^t$ is proportional to $D_{u_j^{t}}$ in this space and $D_{u_j^{t_2}}\in -A^{-1}(u_j^{t_2})$, the values of $u_j^t$ can only change all at once when $w_j^t =\mathbf{0}$. We thus have from here, thanks to the second point of \cref{cond:neurons2}, that either $w_j^{t} \in \R_-^* D_{u_j^{t_2}}$ for any $t\in[t_2,\tau]$, or $w_j^{\tau}=\mathbf{0}$. The latter actually implies $a_j^\tau\neq 0$ and
\begin{equation*}
\langle D(w_j^{\tau},\mathbf{0}),\frac{w_j^{\tau}}{a_j^\tau}\rangle \geq \|D(w_j^{\tau},\mathbf{0})\|.
\end{equation*}

\medskip

With $\cM=A^{-1}(u_j^{t_2})$ again, the remaining cases of \cref{eq:criticalmanifold1} correspond to
\begin{itemize}
\item $w_j^{t_2}=\mathbf{0}$;
\item or $D_{u_j^{t_2}}\in\cM\cup\{\mathbf{0}\}$ and $w_j^{t_2}\neq\mathbf{0}$.
\end{itemize}
The first case directly yields $w_j^{\tau}$ thanks to the second point of \cref{cond:neurons2}.
The second case yields, thanks to \cref{lemma:saddles} that
\begin{equation*}
w_j^{t}\in \cM \text{ for any }t\in[t_2,\tau].
\end{equation*}
\cref{eq:ODEalign1} then rewrites for any $t\in[t_2,\tau]$
\begin{equation}
\frac{\df \langle \w_j^t, D(w_j^t,\mathbf{0}) \rangle}{\df t}  \geq  \|D(w_j^t,\mathbf{0})\|^2 - \langle \w_j^t, D(w_j^t,\mathbf{0})\rangle^2 - \frac{2}{n}\|D(w_j^t,\mathbf{0})\|\lambda^{2-4\varepsilon}\sum_{k=1}^n \|x_k\|^2,\label{eq:ODEalign2}
\end{equation}
where $D(w_j^t,\mathbf{0})$ is constant on $[t_2,\tau]$.
Solutions of the ODE $f'(t)= c^2 - f^2(t)$ with value in $(-c, c)$ are of the form $f(t)= c \tanh(c(t -
t_0))$ for $t_0\in\R$. A Gr\"onwall type comparison leads to
\begin{equation}\label{eq:tanh2}
\begin{gathered}
\forall t \in [t_2, \tau], \langle \w_j^t, D(w_j^t,0) \rangle \geq c_j \tanh(c_j(t-t_0))\\
\text{where } \quad \langle \w_j^{t_2}, D(w_j^{t_2},0) \rangle = c_j \tanh(c_j(t_2-t_0)) \\\text{ and }\quad c_j = \sqrt{\|D(w_j^{t_2},0)\|^2 -\frac{2}{n}\|D(w_j^{t_2},0)\|\lambda^{2-4\varepsilon}\sum_{k=1}^n \|x_k\|^2}.
\end{gathered}\end{equation}
Since $\langle \w_j^{t_2}, D(w_j^{t_2},0) \rangle\geq -\|D(w_j^{t_2},0)\|\sqrt{1-\min(\alpha_{\min}^2,\alpha_0^2)}$, a simple computation using the inequality $\tanh(x)\leq -1 + e^x$ yields
\begin{align*}
e^{c_j(t_2-t_0)}&\geq 1-\frac{\sqrt{1-\min(\alpha_{\min}^2,\alpha_0^2)}}{\sqrt{1-\frac{2}{n\|D(w_j^{t_2},0)\|}\lambda^{2-4\varepsilon}\sum_{k=1}^n \|x_k\|^2}}\\
& \geq  1-\frac{\sqrt{1-\min(\alpha_{\min}^2,\alpha_0^2)}}{\sqrt{1-\frac{\min(\alpha_{\min}^2,\alpha_0^2)}{2}}}\\
&\geq \frac{\min(\alpha_{\min}^2,\alpha_0^2)}{4}.
\end{align*}
The second inequality comes from the choice of $\lambda_{\alpha_0}^*$, while the third one comes from the convex inequality $1-\frac{\sqrt{1-x}}{\sqrt{1-\frac{x}{2}}}\geq \frac{x}{4}$.

\medskip

Using $\tanh(x) \geq 1-e^{-x}$, we now get for any $t \in [t_2, \tau]$
\begin{align*}
\langle \w_j^t, D(w_j^t,0) \rangle & \geq c_j(1-e^{-c_j(t_2-t_0)}e^{c_j(t-t_2)}) \\
 &\geq c_j \left(1-\frac{4}{\min(\alpha_{\min}^2,\alpha_0^2)}e^{-c_j(t-t_2)}\right).
\end{align*}
Plugging the values of $c_j,t_2$ and using the choice of $\lambda_{\alpha_0}^*$, this finally leads to
\begin{align*}
\langle \w_j^{\tau},D(w_j^\tau,\mathbf{0}) \rangle & \geq \|D(w_j^\tau,\mathbf{0})\| - \frac{4 D_{\max}}{\min(\alpha_{\min}^2,\alpha_0^2)}e^{\frac{4D_{\max}^2}{D_{\min}^2\min(\alpha_{\min},\alpha_0)^2}}\lambda^{\frac{\|D(w_j^{\tau},0)\|}{D_{\max}}\varepsilon}\lambda^{-\sqrt{\frac{2\sum_{k=1}^n \|x_k\|^2}{n D_{\max}}}\lambda^{1-2\varepsilon}\varepsilon}\\&\phantom{\geq}-\sqrt{\frac{2}{n}\|D(w_j^{\tau},0)\|\sum_{k=1}^n \|x_k\|^2}\lambda^{1-2\varepsilon}.
\end{align*}
This yields \cref{thm:alignment} where the hidden constant\footnote{We here use the fact that $-x\ln(x)\leq\frac{1}{e}$ and $\varepsilon<\frac{1}{3}$ to deal with the $\lambda^{-\lambda}$ term.} $c_{\alpha_0}$ in the $\mathcal{O}_{\alpha_0}$ is
\begin{equation}\label{eq:calpha0}
c_{\alpha_0} = \frac{4 D_{\max}}{\min(\alpha_{\min}^2,\alpha_0^2)}e^{\frac{4D_{\max}^2}{D_{\min}^2\min(\alpha_{\min},\alpha_0)^2}}e^{\sqrt{\frac{2 \sum_{k=1}^n \|x_k\|^2}{n D_{\max}e^2}}}+\sqrt{\frac{2}{n}D_{\max}\sum_{k=1}^n \|x_k\|^2},
\end{equation}
given that $t_2\leq\tau$ (otherwise, we can trivially take $c_{\alpha_0}=2D_{\max}$).

\section{General Alignment Theorem}\label{app:alignmentgeneral}

This section provides an alternative version of \cref{thm:alignment}, that holds even if the function $G$ has saddle points. In return, it requires a stronger condition on the neurons, that we explain further below.

\begin{condition}\label{cond:neurons}
The neuron $j\in[m]$ satisfies:
\begin{enumerate}
\item $\langle D(w_j^0,\mathbf{0}),\frac{w_j^0}{a_j^0}\rangle>0$;
\item for any $t\in \R_+$ and $\varepsilon>0$:
\begin{equation*}
a_j^{t+\delta} D(w_j^t,\theta^{t+\delta})\in  A^{-1}(A(w_j^t)) \text{ for any } \delta \in [0,\varepsilon]\implies w_j^{t+\delta}\in  A^{-1}(A(w_j^{t})) \text{ for any } \delta\in[0,\varepsilon].
\end{equation*}
\end{enumerate}
\end{condition}

\begin{thm}\label{thm:alignmentgeneral}
If \cref{ass:loss,ass:Dj} hold, then the following holds almost surely for any constant $\varepsilon\in(0,\frac{1}{3})$ and initialisation scale $\lambda < \lambda_1^*$, where $\lambda_1^*>0$ only depends\footnote{The exact value of $\lambda_1^*$ is given by \cref{eq:lambdastaralpha} in \cref{app:alignment}.} on the data $(x_k,y_k)_k$; with $D_{\max} \coloneqq \max\limits_{w\in\R^d}\|D(w,\mathbf{0})\|$ and $\tau\coloneqq -\frac{\varepsilon\ln(\lambda)}{D_{\max}}$,
\begin{enumerate}[itemsep=-0.5em, topsep=-1em, label=(\roman*),leftmargin=1cm]
\item neurons' norms do not change until $\tau$:
\begin{equation*}
\forall t \leq \tau,\forall j\in[m], |a_j^0|\lambda^{2\varepsilon} \leq |a_j^t| \leq |a_j^0|\lambda^{-2\varepsilon}.
\end{equation*}

\item Moreover,  for any neuron $j$ satisfying \cref{cond:neurons}, $D(w_j^\tau,\mathbf{0})$ is an extremal vector, along which $w_j^\tau$ is aligned:
\begin{equation*}
\langle D(w_j^\tau,\mathbf{0}),\frac{w_j^\tau}{a_j^\tau}\rangle\geq \|D(w_j^\tau,\mathbf{0})\| -\bigO{\lambda^{\frac{\|D(w_j^\tau,\mathbf{0})\|}{D_{\max}}\varepsilon}}.
\end{equation*}
\end{enumerate} 
\end{thm}
\begin{rem}
The first point of \cref{cond:neurons} lower bounds the quantity $\langle D(w_j^t,\mathbf{0}),\frac{w_j^t}{a_j^t}\rangle$ by $0$ at initialisation and is needed for the second point of \cref{thm:alignmentgeneral}. Even though $\langle D(w_j^t,\mathbf{0}),\frac{w_j^t}{a_j^t}\rangle$ is increasing over time, the quantity $\|D(w_j^t,\mathbf{0})\|+\langle D(w_j^t,\mathbf{0}),\frac{w_j^t}{a_j^t}\rangle$ is not necessarily monotone, since $\|D(w_j^t,\mathbf{0})\|$ can drastically change from an activation cone to another. As a consequence, it becomes challenging to bound $\langle D(w_j^t,\mathbf{0}),\frac{w_j^t}{a_j^t}\rangle$ away from $-\|D(w_j^t,\mathbf{0})\|$ during the whole alignment phase; which would allow to lower bound the alignment rate. The $\langle D(w_j^0,\mathbf{0}),\frac{w_j^0}{a_j^0}\rangle > 0$ condition allows to do so, since it holds during the whole alignment phase by monotonicity.

In the absence of saddle points of $G$, a more general condition $\langle D(w_j^0,\mathbf{0}),\frac{w_j^0}{a_j^0}\rangle > -\sqrt{1-\alpha_0^2}\|D(w_j^t,\mathbf{0})\|$ is used. In that case, we can indeed show that $\langle D(w_j^t,\mathbf{0}),\frac{w_j^t}{a_j^t}\rangle$ is either bounded away from $-\|D(w_j^t,\mathbf{0})\|$ when entering a new activation cone, or is exactly $-\|D(w_j^t,\mathbf{0})\|$, in which case it stays aligned with $-D(w_j^t,\mathbf{0})$ for the whole phase. 
Such a \textit{quantisation} of the possible values of $\langle D(w_j^t,\mathbf{0}),\frac{w_j^t}{a_j^t}\rangle$ when entering a new activation cone does not hold with the presence of saddle points.
\end{rem}

\begin{rem}
The second point in \cref{cond:neurons} roughly states that the neuron $j$ does not spontaneously leave an activation manifold when it can remain inside this manifold (see \cref{ex:condition} in \cref{sec:excondition} for further insights on the condition). This condition is needed to avoid degenerate situations that could happen near saddle points of $G$ (or their corresponding activation manifolds).
\cref{lemma:saddles} indeed states that this point is guaranteed (at least up to a time $\tau$, which is sufficient to prove \cref{thm:alignment}) with the absence of saddle points.
\end{rem}

\subsection{Understanding \cref{cond:neurons}}\label{sec:excondition}

\begin{ex}\label{ex:condition}
Consider a simplified case with a single data point $(x_k,y_k)=(1,1)$. In the early phase, \cref{eq:ODE1} can then be rewritten with some positive function $g_j(t)$ as\footnote{To be rigorous $g_j(t)$ should also depend on $(w_j^s)_{s<t}$, but this does not change the point.}
\begin{equation}\label{eq:excondition}
\frac{\df w_j^t}{\df t} \in g_j(t) \partial \sigma(w_j^t).
\end{equation}
In the case of ReLU activation ($\gamma = 0$) and $w_j^0=0$, the set of solutions of \cref{eq:excondition} are described by functions of the following form, for all $t_0 \in \R_+\cup\{\infty\}$,
\begin{equation*}
w_j^t = \begin{cases}
0 \text{ for any } t\in[0,t_0]\\
\int_{t_0}^t g_j(s)\df s \text{ for any } t>t_0
\end{cases}.
\end{equation*}
In particular, only the constant solution $w_j^t \equiv 0$ satisfies the second point of \cref{cond:neurons}. \cref{cond:neurons} indeed prevents the neuron to spontaneously leave at a finite time (given by $t_0$) its activation manifold (here given by $w=0$) when it can stay within. 
\end{ex}

Although solutions satisfying the second point of \cref{cond:neurons} might seem natural from a gradient flow point of view, we believe that the limits of gradient descent dynamics with arbitrarily small step size (i.e. Euler solutions) do not necessarily correspond to this kind of solutions. Non-zero step size solutions can indeed bounce from a side to another side of a stable manifold (without never being exactly located on this manifold). They could then escape this manifold if it later becomes unstable, due to the changes in the estimated function $h_\theta$. Characterising the exact limit of gradient descent with non-differentiable activations is out of our scope and remains open for future work.

\subsection{Proof of \cref{thm:alignmentgeneral}}

(i) The first point follows the exact same lines as \cref{app:proofalignmenti}.

\bigskip

\noindent (ii) Consider a neuron $j$ that satisfies \cref{cond:neurons} and $a_j^0>0$. The symmetric case ($a_j^0$) is dealt with similarly. The exact same arguments as in the proof of  \cref{thm:alignment} in \cref{app:proofalignmentii} lead again to \cref{eq:ODEalign1} that we recall here.
\begin{equation*}
\frac{\df \langle \w_j^t, D(w_j^t,\mathbf{0}) \rangle}{\df t} \geq  \|D(w_j^t,\mathbf{0})\|^2 - \langle \w_j^t, D(w_j^t,\mathbf{0})\rangle^2 - \frac{2}{n}\|D(w_j^t,\mathbf{0})\|\lambda^{2-4\varepsilon}\sum_{k=1}^n \|x_k\|^2.
\end{equation*}

\medskip

We define again $u_j^t = A(w_j^t)$ and $D_{u_j^t}\coloneqq D(w_j^t,\mathbf{0})\in\D_{u_j^t}$ in the following. 
If $D_{u_j^t}\not\in A^{-1}(u_j^t)\cup \{\mathbf{0}\}$, the definition of $\alpha_{\min,+}$ (which is positive thanks to \cref{lemma:alphamin}), given by \cref{eq:alphamin}, implies that $$\langle \w_j^t,D(w_j^t,\mathbf{0})\rangle \leq \|D(w_j^t,\mathbf{0})\| \sqrt{1-\alpha_{\min,+}^2}.$$ 
\cref{eq:ODEalign1} implies that as long as $\|D(w_j^t,\mathbf{0})\|^2 - \langle \w_j^t, D(w_j^t,\mathbf{0})\rangle^2 \geq \frac{2}{n}\|D(w_j^t,\mathbf{0})\|\lambda^{2-4\varepsilon}\sum_{k=1}^n \|x_k\|^2$, the quantity $\langle \w_j^t, D(w_j^t,\mathbf{0}) \rangle$ is increasing. In particular, the choice of $\lambda$ makes it increasing as long as $D_{u_j^t}\not\in A^{-1}(u_j^t)$ and $\langle \w_j^t,D_{u_j^t}\rangle> 0$. 
Since  $\langle \w_j^0,D_{u_j^0}\rangle> 0$ by \cref{cond:neurons}, we have again by continuous induction that as long as $D_{u_j^t}\not\in A^{-1}(u_j^t)$:
\begin{equation}
\frac{\df \langle \w_j^t, D(w_j^t,\mathbf{0}) \rangle}{\df t} \geq \frac{D_{\min}^2}{2} \alpha_{\min}^2.
\end{equation}
This then implies that after a time $t_2$, $D_{u_j^{t_2}}\in A^{-1}(u_j^{t_2})$ with
\begin{equation*}
t_2 \leq \frac{2D_{\max}}{D_{\min}^2\alpha_{\min}^2}.
\end{equation*}
Again, we can observe that either $\lambda^{\varepsilon}=\Omega(1)$ or $t_2\leq \tau$ and we focus only on the latter case. 
%
From there, \cref{lemma:interior}, proven in \cref{app:alignment}, implies that for any $t\in[t_2,\tau]$, $D(w_j^{t_2},\theta^t) \in A^{-1}(u_j^{t_2})$. The second point of \cref{cond:neurons} then implies that $w_j^t \in A^{-1}(u_j^{t_2})$ for any $t\in[t_2,\tau]$. \cref{eq:ODEalign1} now rewrites
\begin{equation*}
\frac{\df \langle \w_j^t, D(w_j^t,\mathbf{0}) \rangle}{\df t}  \geq  \|D(w_j^t,\mathbf{0})\|^2 - \langle \w_j^t, D(w_j^t,\mathbf{0})\rangle^2 - \frac{2}{n}\|D(w_j^t,\mathbf{0})\|\lambda^{2-4\varepsilon}\sum_{k=1}^n \|x_k\|^2,
\end{equation*}
where $D(w_j^t,\mathbf{0})$ is constant on $[t_2,\tau]$.
Again, comparing with solutions of the ODE $f'(t)= c^2 - f^2(t)$, we get
\begin{equation*}
\begin{gathered}
\forall t \in [t_2, \tau], \langle \w_j^t, D(w_j^t,\mathbf{0}) \rangle \geq c_j \tanh(c_j(t-t_2))\\
\text{where } \quad c_j = \sqrt{\|D(w_j^t,\mathbf{0})\|^2 -\frac{2}{n}\|D(w_j^t,\mathbf{0})\|\lambda^{2-4\varepsilon}\sum_{k=1}^n \|x_k\|^2}.
\end{gathered}\end{equation*}
Since $\tanh(x) \geq 1-e^{-x}$, the previous equation rewrites for any $t \in [t_2, \tau]$
\begin{equation*}
\langle \w_j^t, D(w_j^t,\mathbf{0}) \rangle \geq c_j \left(1-e^{-c_j(t-t_2)}\right).
\end{equation*}
Using the value of $c_j$ and bound on $t_2$,
\begin{align*}
\langle \w_j^{\tau}, D(w_j^{\tau},\mathbf{0}) \rangle & \geq c_j\left(1-e^{c_j t_2}\lambda^{\frac{c_j}{D_{\max}}\varepsilon}\right)\\
&\geq \|D(w_j^{\tau},\mathbf{0})\| - D_{\max}e^{\frac{2D_{\max}^2}{D_{\min}^2\alpha_{\min}^2}}\lambda^{\frac{\|D(w_j^{\tau},0)\|}{D_{\max}}\varepsilon}\lambda^{-\sqrt{\frac{2\sum_{k=1}^n \|x_k\|^2}{n D_{\max}}}\lambda^{1-2\varepsilon}\varepsilon}\\&\phantom{\geq}-\sqrt{\frac{2}{n}\|D(w_j^{\tau},0)\|\sum_{k=1}^n \|x_k\|^2}\lambda^{1-2\varepsilon}.
\end{align*}
This leads to the second part of \cref{thm:alignmentgeneral} where the hidden constant $c$ in the $\mathcal{O}$ is
\begin{equation}\label{eq:c}
c= D_{\max}e^{\frac{2D_{\max}^2}{D_{\min}^2\alpha_{\min}^2}}e^{\sqrt{\frac{2 \sum_{k=1}^n \|x_k\|^2}{n D_{\max}e^2}}} + \sqrt{\frac{2}{n}D_{\max}\sum_{k=1}^n \|x_k\|^2},
\end{equation}
given that $t_2\leq\tau$ (otherwise, we take $c=2D_{\max}$).

\subsection[Absence of Saddle Points for Dimension 2.]{Absence of Saddle Points for $d\leq 2$.}\label{app:saddlesdim}

\cref{foot:saddlesdim} states that saddles points of $G$ only exist on sectors of dimension at most $d-2$. A rigorous statement of this affirmation is given by \cref{lemma:saddlesdim} below. For that we need an assumption on the data generation, that is slightly weaker than \cref{ass:Dj}.

\begin{assumption}\label{ass:saddlesdim}
The data $(x_k,y_k)$ are drawn independently at random, following a distribution such that for any $k,k'\in[n]$, the random variable $y_k \langle \frac{x_{k'}}{\|x_{k'}\|}, x_k \rangle$ admits a density in $\R$.
\end{assumption}

\begin{lem}\label{lemma:saddlesdim}
If $d\leq 2$ and \cref{ass:saddlesdim} holds, then for any stationary point $w^*\in\bS_d$ of the function $G$ defined in \cref{eq:G}, $w^*$ is a local minimum of $G$ on $\bS_d$.
\end{lem}
\begin{proof}
If $d=1$, then $\bS_d$ is discrete and so any point is a local minimum.

Consider now $d=2$ and a stationary point $w^*$ of $G$. First assume that $\langle w^*, x_k \rangle \neq 0$ for any $k\in[n]$ such that $y_k x_k\neq 0$. Then $G$ is by definition linear in a neighbourhood of $w^*$, so that $w^*$ is not a saddle point of $G$.

Now assume instead that $\Span(\{x_k \mid \langle w^*,x_k\rangle=0\})$ has dimension $1$---it cannot be of dimension $2$ since $w^*\neq \mathbf{0}$. Even if it means reordering the data, we can assume that $x_1\neq 0$, $\langle x_1, w^*\rangle = 0$. Moreover, \cref{ass:saddlesdim} implies that almost surely, for any $k\geq 2$, $(x_1, x_k)$ are unaligned and so $\langle w^*, x_k\rangle \neq 0$. 
Then note that $\R^d = \Span(x_1) \oplus^\perp \Span(w^*)$. In consequence, we can consider the following neighbourhood of $w^*$ in $\bS_d$ for any $\delta$:
\begin{equation}
N(\delta) = \left\{\cos(\theta) w^* + \frac{\sin(\theta)}{\|x_1\|} x_1 \mid |\theta|<\delta\right\}.
\end{equation}
At proximity of $w^*$ (i.e., for small enough $\theta$),
\begin{align}
G\left(\cos(\theta) w^* + \frac{\sin(\theta)}{\|x_1\|} x_1\right) & = G(w^*)+
\sin(\theta)_+\frac{y_1\|x_1\|}{n} + \frac{1}{n}\sum_{k\geq 2} y_k\sin(\theta)\langle \frac{x_1}{\|x_1\|}, x_k \rangle \\& \phantom{=}+ \frac{1-\cos(\theta)}{n}\sum_{k \geq 2}y_k \langle w^* , x_k \rangle\notag\\
& = G(w^*) +g_1(\theta) + \bigO{\theta^2},\label{eq:DLsaddlesdim}
\end{align}
where $g_1(t) = \frac{y_1\|x_1\|}{n} (t)_+ +\frac{t}{n}\sum_{k\geq 2} y_k \langle \frac{x_1}{\|x_1\|}, x_k \rangle$. Note that $g_1$ is a continuous, piecewise affine function on $\R$. Moreover, both its left and right slopes at $0$ are almost surely non-zero, thanks to \cref{ass:saddlesdim}. As a consequence, the piecewise affine terms will dominate in \cref{eq:DLsaddlesdim}, and $w^*$ is necessarily a local minimum---this comes from the fact that a piecewise affine function in $\R$ does not admit any saddle point.
\end{proof}

\section{Proof of \cref{thm:noconvergence}}\label{app:proofnoconvergence}

We prove in this section a more general version of \cref{thm:noconvergence}, given by \cref{thm:noconvergencegeneral} below.
\begin{thm}\label{thm:noconvergencegeneral}
If \cref{ass:noconvergence,ass:cK,ass:phase3} hold, then there exists some constant $\tilde\lambda=\Theta(1)$ such that for any $\lambda < \tilde\lambda$ and $m \in \N$, the parameters $\theta^t$ converge to some $\theta^{\infty}_{\lambda}$ such that
\begin{gather*}
h_{\theta^{\infty}_\lambda}(x)= x^{\top}\beta^* \text{ for any }x\in\conex,
\end{gather*}
where $\conex=\left\{tz \mid t\in\R_+, z\in\convex(\{x_k\mid k\in[n]\})\right\}$ is the cone generated by the convex hull of the data points. 
Moreover
\begin{gather*}
\forall x\in\R^2, \lim_{\lambda\to 0^+}h_{\theta^{\infty}_\lambda}(x) = (x^{\top}\beta^*)_+.
\end{gather*}
\end{thm}
In the following, we prove \cref{lemma:phase1,lemma:phase2a,lemma:phase2b,lemma:phase3}, but in more general versions, as follows:
\begin{itemize}
\item \cref{lemma:phase1} only require \cref{ass:noconvergence} (instead of \cref{ass:noconvergenceall}).
\item \cref{lemma:phase2a,lemma:phase2b} only require \cref{ass:noconvergence,ass:cK}.
\item \cref{lemma:phase3} only require \cref{ass:noconvergence,ass:cK,ass:phase3}.
\end{itemize}

\subsection{Phase 1}

\begin{proof}[Proof of \cref{lemma:phase1}.]
Note in this proof $\w_i^t = \frac{w_i^t}{a_i^t}$.
A crucial observation is that the only non-zero extremal vector is $D^*$, which corresponds to a local maximum of $G$. \cref{ass:noconvergence} actually implies \cref{lemma:Dj,lemma:alphamin,lemma:delta0,lemma:interior} (and the absence of saddle points), that are sufficient (instead of \cref{ass:Dj}), to apply\footnote{Otherwise, we could just state \cref{thm:noconvergence} for almost all the datasets in the considered open set.} \cref{thm:alignment}. Showing these implications is direct once we observe $D^*$ is the only extremal vector and is omitted for the sake of conciseness.

\medskip

1) The only non-zero extremal vector is $D^*$, which corresponds to a local maximum of $G$. Also notice that for every $i \in \cI$, \cref{ass:noconvergence} implies that $\langle D(w_i^0,0),\w_i^0\rangle>0$. Moreover, we can show that $\langle D(w,\theta^t),w\rangle\geq 0$ for any $t\leq \tau$ and $w\in\R^d$, so that $a_i^t$ is non-decreasing during the early phase for $i\in\cI$. This yields $w_i^t\neq\mathbf{0}$ for $i\in \cI$ and $t\leq \tau$. 
As a consequence, any neuron $i\in \cI$ satisfies \cref{cond:neurons2}.
\cref{thm:alignment} can then be applied and implies the neuron $i$ is aligned with $D^*$ at time $\tau$ i.e.,
\begin{equation*}
\langle D^*, \w_i^{\tau} \rangle \geq \|D^*\| - c\lambda^{\varepsilon}.
\end{equation*}
Moreover, $a_i^0 \leq a_i^{\tau} \leq a_i^{0}\lambda^{-2\varepsilon}$ follows by \cref{thm:alignment} and noticing that $a_i^t$ is non-decreasing on $[0,\tau]$ for any $i\in \cI$.

\medskip

2) For the same reason as in 1), $a_i^t$ is non-decreasing, and non-positive thanks to \cref{lemma:balanced}.

\medskip

3) By definition of $\cI$ and $\cN$, note that a.s. for $i\not\in \cI\cup \cN$, $\langle w_i^0, x_k\rangle < 0$ for any $k\in[n]$. As a consequence, any neuron $i \not\in \cI\cup \cN$ can be ignored, since its gradient is null during the whole training: $\D(w_i^t, \theta^t)=\{\mathbf{0}\}$ for any $t\geq 0$. 
\end{proof}

\subsection{Phase 2a}
\begin{proof}[Proof of \cref{lemma:phase2a}.]
Consider some $\varepsilon_2>0$ and define 
\begin{equation*}
\tau_2 \coloneqq \inf\left\lbrace t\geq \tau \mid \sum_{i\in \cI} (a_i^t)^2\geq \varepsilon_2 \text{ or } \sum_{i\in \cN} (a_i^t)^2 \geq 2\lambda^2 \right\rbrace.
\end{equation*}
First, for any $t\in[\tau,\tau_2]$, $ -\bigO{\lambda^2 y_k}\leq h_{\theta^t}(x_k)\leq \bigO{\varepsilon_2 y_k}$ for any $k$. 
Thanks to \cref{ass:noconvergence} and our choice of $\lambda$, this implies the following equality for any $t\in[\tau,\tau_2]$
\begin{gather*}
D^t = \|D^*\| + \bigO{\varepsilon_2}
\end{gather*}
As a consequence, for a small enough $\varepsilon^*_2 =\Theta(1)$ such that $\varepsilon_2\leq \varepsilon_2^*$ and positive constants $\delta_k =\langle \frac{D^*}{\|D^*\|} ,x_k\rangle - \Theta(\varepsilon_2)$, this yields
\begin{equation*}
D^t \in \left\{ u \in \R^d \mid \forall k\in[n],\langle \frac{u}{\|u\|} , x_k \rangle \geq \delta_k \right\}.
\end{equation*}
From there, \cref{eq:ODEdirection} implies that for any $t\in[\tau,\tau_2]$ and as long as $\langle \w_i^{t}, x_k\rangle > 0$ for any $k$:
\begin{equation*}
\frac{\df \langle \w_i^t, x_k\rangle}{\df t} \geq \|D^{t}\|\left(\delta_k -\langle \w_i^t,x_k\rangle\right). 
\end{equation*}
Moreover, $\langle \w_i^{\tau}, x_k\rangle \geq \delta_k$ thanks to \cref{lemma:phase1,lemma:wixktau} for a small enough $\frac{\lambda^{\frac{\varepsilon}{2}}}{\varepsilon_2}$. This last ODE then yields that $\langle \w_j^t, x_k\rangle\geq \delta_k$ for any $k\in[n]$ and $t\in[\tau,\tau_2]$, which is the second point of \cref{lemma:phase2a}. 

\medskip

Also we can choose $\varepsilon_2^*=\Theta(1)$ and $\tilde\lambda=\Theta(1)$ small enough so $|h_{\theta^t}(x_k)|< y_k$ for any $t\in[\tau,\tau_2]$. This then implies that $a_i^t$ is non-decreasing for any $i\in[m]$ on $[\tau,\tau_2]$. In particular, this is the case for $i\in \cN$, which implies the fourth point of \cref{lemma:phase2a}. 

\medskip

Moreover, we can bound the growth of $a_i^t$ for $i\in \cI$. In particular, we have for $i\in \cI$ and $t\in[\tau,\tau_2]$,
\begin{align}\label{eq:growthbound}
\frac{\df a_i^t}{\df t} & = a_i^t \langle \w_i^t,D^t\rangle \geq \frac{a_i^t}{n}\sum_{k=1}^n \delta_k (1-\bigO{\varepsilon_2})y_k.
\end{align}
This implies an exponential lower bound on the growth of $a_i^t$, and so we have $\tau_2 < +\infty$. Since the second condition in the definition of $\tau_2$ does not break at $\tau_2$ (thanks to the fourth point already shown above), the first condition necessarily breaks at $\tau_2$. By continuity, this implies the first point of \cref{lemma:phase2a}: $\sum_{i\in \cI}(a_i^{\tau_2})^2=\varepsilon_2$.

\medskip

It now remains to prove the third point of \cref{lemma:phase2a}. We can also upper bound the growth of $a_i^t$ on $[\tau, \tau_2]$:
\begin{align*}
\frac{\df a_i^t}{\df t} & \leq \frac{a_i^t}{n}\sum_{k=1}^n (1+\bigO{\lambda^2})y_k\langle \w_i^t,x_k\rangle\\
& = a_i^t (1+\bigO{\lambda^2})\langle \w_i^t, D^*\rangle \\
& \leq a_i^t(1+\bigO{\lambda^2}) \| D^*\|.
\end{align*}
A Gr\"onwall comparison argument then allows to write, with $\varepsilon$ given by \cref{lemma:phase1}:
\begin{align*}
\sum_{i\in \cI} (a_i^{\tau_2})^2 & \leq  e^{2(1+\bigO{\lambda^2})\| D^*\|(\tau_2-\tau)}\sum_{i\in \cI} (a_i^{\tau})^2 \\
& \leq  e^{2(1+\bigO{\lambda^2})\| D^*\|(\tau_2-\tau)} \lambda^{2-4\varepsilon}.
\end{align*}
Since the left term is exactly $\varepsilon_2$, this gives the following bound on $\tau_2$:
\begin{align}
\tau_2 - \tau &\geq -\frac{1-2\varepsilon}{\| D^*\|}\ln(\lambda) -\bigO{\ln(\frac{1}{\varepsilon_2})+\lambda^2}\notag\\
&\geq -\frac{1-2\varepsilon}{\| D^*\|}\ln(\lambda) -\bigO{\ln(\frac{1}{\varepsilon_2})}\label{eq:tau2bound}.
\end{align}
Moreover, a similar bound to \cref{eq:growthbound} yields for any $i\in \cI$ and $t\in[\tau,\tau_2]$:
\begin{align*}
\langle \w_i^t, D^t\rangle & \geq \frac{1}{n}\sum_{k=1}^n (1-\bigO{\varepsilon_2})\delta_k y_k\\
& = (1-\bigO{\varepsilon_2})\langle \frac{ D^*}{\| D^*\|}, D^*\rangle \\
& =  (1-\bigO{\varepsilon_2})\| D^*\|.
\end{align*}
Let $i,j\in \cI$, we can now study the evolution of $\langle \w_i^t, \w_j^t \rangle$ for $t\in[\tau, \tau_2]$:
\begin{align*}
\frac{\df \langle \w_i^t, \w_j^t\rangle}{\df t} & = \langle D^t, \w_i^t + \w_j^t\rangle \left(1-\langle \w_i^t,\w_j^t \rangle \right)\\
& \geq 2(1-\bigO{\varepsilon_2})\| D^*\|\left(1-\langle \w_i^t,\w_j^t \rangle \right).
\end{align*}
Again, a Gr\"onwall comparison argument implies, using the bound on $\tau_2-\tau$ in \cref{eq:tau2bound},
\begin{align}
1-\langle \w_i^{\tau_2}, \w_j^{\tau_2}\rangle & \leq e^{-2(1-\bigO{\varepsilon_2}) \| D^*\|(\tau_2-\tau)}\left( 1-\langle \w_i^{\tau}, \w_j^{\tau}\rangle\right)\notag\\
&\leq e^{-(1-\bigO{\varepsilon_2})\left(-(2-4\varepsilon)\ln(\lambda)-\bigO{\ln(\frac{1}{\varepsilon_2})}\right)}\left( 1-\langle \w_i^{\tau}, \w_j^{\tau}\rangle\right)\notag\\
&\leq \bigO{\varepsilon_2^{-1+\bigO{\varepsilon_2}}\lambda^{(1-\bigO{\varepsilon_2})(2-4\varepsilon)}}\left( 1-\langle \w_i^{\tau}, \w_j^{\tau}\rangle\right)\label{eq:wiwj}.
\end{align}
Moreover, we can bound $\left( 1-\langle \w_i^{\tau}, \w_j^{\tau}\rangle\right)$, thanks to \cref{lemma:phase1}. We can indeed write
\begin{gather*}
\w_i^{\tau} = \langle \w_i^{\tau},\frac{ D^*}{\| D^*\|} \rangle \frac{ D^*}{\| D^*\|} + v_i^{\tau}\\
\text{where } \langle v_i^{\tau},  D^*\rangle = 0 \text{ and } \|v_i^{\tau}\|^2=\bigO{\lambda^{\varepsilon}}.
\end{gather*}
This then implies
\begin{align*}
\langle \w_i^{\tau}, \w_j^{\tau}\rangle & = \langle \w_i^{\tau},\frac{ D^*}{\| D^*\|} \rangle\langle \w_j^{\tau},\frac{ D^*}{\| D^*\|} \rangle + \langle  v_i^{\tau}, v_{j}^{\tau}\rangle \\
&  \geq 1 - \bigO{\lambda^{\varepsilon}}.
\end{align*}
\cref{eq:wiwj} then rewrites
\begin{align*}
1-\langle \w_i^{\tau_2}, \w_j^{\tau_2}\rangle & \leq \bigO{\varepsilon_2^{-1+\bigO{\varepsilon_2}}\lambda^{(1-\bigO{\varepsilon_2})(2-4\varepsilon)+\varepsilon}}\\
& \leq \bigO{\varepsilon_2^{-1}\lambda^{(1-\bigO{\varepsilon_2})(2-4\varepsilon)+\varepsilon}}.
\end{align*}
The third point of \cref{lemma:phase2a} then follows for a small enough choice of $\varepsilon_2^*=\Theta(1)$ (that can depend on $\varepsilon$) and $\varepsilon_2=\Omega(\lambda^{\varepsilon})$.
\end{proof}
\begin{lem}\label{lemma:wixktau}
Consider $\varepsilon, \tau$ defined in \cref{lemma:phase1}.
For any $\delta_k =\langle \frac{D^*}{\|D^*\|} ,x_k\rangle - \Theta(\varepsilon_2)$, there exists a large enough constant $c'$ such that if \cref{ass:noconvergence} holds and $\varepsilon_2\geq c'\lambda^{\frac{\varepsilon}{2}}$, then for all $i\in \cI$:
\begin{equation*}
\forall k\in [n], \langle \w_i^{\tau},x_k\rangle \geq \delta_k,
\end{equation*}
where $\delta_k$ is defined in the proof of \cref{lemma:phase2a}.
\end{lem}
\begin{proof}
For any $i\in \cI$, we use the same decomposition as in the proof of \cref{lemma:phase2a}, 
\begin{gather*}
\w_i^{\tau} = \langle \w_i^{\tau},\frac{ D^*}{\| D^*\|} \rangle \frac{ D^*}{\| D^*\|} + v_i^{\tau}\\
\text{where } \langle v_i^{\tau},  D^*\rangle = 0 \text{ and } \|v_i^{\tau}\|^2=\bigO{\lambda^{\varepsilon}}.
\end{gather*}
This decomposition yields thanks to \cref{lemma:phase1}, for some data dependent constant $c>0$
\begin{align*}
\langle \w_i^{\tau}, x_k\rangle & \geq \left( 1-\bigO{\lambda^{\varepsilon}} \right) \langle \frac{ D^*}{\| D^*\|}, x_k \rangle - \|v_i^{\tau}\| \|x_k\|\\
& \geq \left( 1-\frac{c}{\| D^*\|}\lambda^{\varepsilon} \right) \langle \frac{ D^*}{\| D^*\|}, x_k \rangle - \bigO{\lambda^{\frac{\varepsilon}{2}}}\\
&\geq \langle  \frac{D^*}{\|D^*\|}, x_k \rangle - \bigO{\lambda^{\frac{\varepsilon}{2}}}.
\end{align*}
We thus have for a large enough $c'$ with $\varepsilon_2\geq c' \lambda^{\frac{\varepsilon}{2}}$,
\begin{equation*}
\langle \w_i^{\tau}, x_k\rangle\geq  \delta_k.
\end{equation*}
\end{proof}

\subsection{Phase 2b}
Define in the following
\begin{gather}
\beta \coloneqq  \min_{k\in\cK}\frac{y_k\|x_k^2\|}{n}-\frac{1}{n}\sqrt{\sum_{k'=1}^n y_{k'}^2}\sqrt{\sum_{k'\neq k}\langle x_k', x_k\rangle^2}\label{eq:beta},\\
\end{gather}
\cref{ass:cK} implies that $\beta$ is positive.

\begin{proof}[Proof of \cref{lemma:phase2b}.]
For this phase, define for some $\delta'=\Theta(1)$
\begin{equation*}
\begin{aligned}
\tau_3 \coloneqq \inf\big\{ t\geq \tau_2 \mid & \left\|\beta^* - \sum_{i\in I\cup N}a_i^{t} w_i^{t} \right\| \leq \varepsilon_3 \text{ or } \exists i,j\in \cI, \langle \w_i^t, \w_j^t \rangle \leq 1-\lambda^{\varepsilon}  \\
& \text{or } \exists i\in \cI, \exists k\in [n],  \langle \w_i^t, x_k \rangle \leq \delta'\|x_k\| \text{ or } \sum_{i\in \cN}(a_i^t)^2 \geq  \lambda^{2(1-\varepsilon)} \big\}.
\end{aligned}
\end{equation*}
Thanks to \cref{lemma:phase2bangle,lemma:phase2bnorm} proven below, we have for any $t\in[\tau_2,\tau_3]$:
\begin{gather*}
\forall  i\in \cI, \forall k\in [n], \langle \w_i^t, x_k\rangle > \delta'\|x_k\|\\
\text{and }\sum_{i\in \cI} (a_i^t)^2 \geq \varepsilon_2.
\end{gather*}
We note in the following $\beta_{\cI}^t = \sum_{i\in \cI} a_i^t w_i^t$, $H = \frac{1}{n}X^{\top}X$ and for any vector $w\in\R^d$, $\|w\|_H^2=w^{\top}Hw$. The distance between $\beta^*$ and $\beta_{\cI}^t$ then decreases as
\begin{align}
\frac{\df \|\beta^*-\beta^t_{\cI}\|^2_H}{\df t} & = 2(\beta_{\cI}^t-\beta^*)^{\top}H\frac{\df \beta_{\cI}^t}{\df t}  \\
& = 2(\beta_{\cI}^t-\beta^*)^{\top}H \left(\sum_{i\in \cI}(a_i^t)^2 I_d + w_i^t w_i^{t \ \top}\right) D^t.\label{eq:linearrateODE}
\end{align}
\cref{eq:linearrateODE} comes from the fact that neurons in $\cI$ are activated along all data points. 
Also, note that for any $t\in[\tau_2, \tau_3]$
\begin{align*}
D^t & = -\frac{1}{n}\sum_{k=1}^n (h^{\theta^t}(x_k)-y_k)x_k \\
& = -\frac{1}{n}\sum_{k=1}^n \Big( (\langle \beta_{\cI}^t, x_k\rangle-y_k)x_k + \sum_{i\in \cN} a_i^t \langle w_i^t, x_k\rangle_+ x_k\Big)\\
& = -\frac{1}{n} X^{\top}(X \beta_{\cI}^t - \mathbf{y}) + \bigO{\lambda^{2-2\varepsilon}} \\
& = -H(\beta_{\cI}^t-\beta^*) + \bigO{\lambda^{2-2\varepsilon}} ,
\end{align*}
For the sake of clarity, we denote in the following $\sigma_{\min}$ and $\sigma_{\max}$ respectively for the smallest and largest eigenvalue of $H$.
\cref{eq:linearrateODE} now rewrites with \cref{lemma:boundA}
\begin{align*}
\frac{\df \|\beta^*-\beta_{\cI}^t\|^2_H}{\df t} & = - 2(\beta_{\cI}^t-\beta^*)^{\top}H \left(\sum_{i\in \cI}(a_i^t)^2 I_d + w_i^t w_i^{t \ \top}\right)\left(H(\beta_{\cI}^t-\beta^*) - \bigO{\lambda^{2-2\varepsilon}} \right)\\
& \leq -2\varepsilon_2 (\beta_{\cI}^t-\beta^*)^{\top}H^2(\beta_{\cI}^t-\beta^*) + \bigO{\|\beta^*-\beta_{\cI}^t\|_H\lambda^{2-2\varepsilon}}\\
&\leq  -2\varepsilon_2 \sigma_{\min}\|\beta_{\cI}^t-\beta^*\|_H^2 + \bigO{\|\beta^*-\beta_{\cI}^t\|_H\lambda^{2-2\varepsilon}}.
\end{align*}
This finally yields the comparison for any $t\in[\tau_2,\tau_3]$
\begin{equation*}
\frac{\df \|\beta^*-\beta_{\cI}^t\|_H}{\df t} \leq -\varepsilon_2\sigma_{\min}\|\beta_{\cI}^t-\beta^*\|_H +\bigO{\lambda^{2-2\varepsilon}} .
\end{equation*}
A Gr\"onwall comparison argument gives for $t\in[\tau_2,\tau_3]$
\begin{align*}
\|\beta^*-\beta_{\cI}^t\|_H \leq \|\beta^{\tau_2}_{\cI}-\beta^*\|_H e^{-\varepsilon_2 \sigma_{\min}(t-\tau_2)} + \bigO{\frac{\lambda^{2-2\varepsilon}}{\varepsilon_2}}.
\end{align*}
As $\varepsilon_2=\Omega( \lambda^{\varepsilon/2})$, using the comparison between $\|\cdot \|_2$ and $\|\cdot\|_H$ norms,
\begin{align*}
\|\beta^*-\beta_{\cI}^t\|_2 & \leq \frac{1}{\sqrt{\sigma_{\min}}}\|\beta^*-\beta_{\cI}^t\|_{H}\\
&\leq \frac{1}{\sqrt{\sigma_{\min}}} \|\beta^*-\beta_{\cI}^{\tau_2}\|_{H} e^{-\varepsilon_2 \sigma_{\min}(t-\tau_2)} + \bigO{\lambda^{2-\frac{5}{2}\varepsilon}}\\
&\leq \left(\sqrt{\frac{\sigma_{\max}}{\sigma_{\min}}} \|\beta^*\|+\bigO{\varepsilon_2}\right) e^{-\varepsilon_2 \sigma_{\min}(t-\tau_2)}+ \bigO{\lambda^{2-\frac{5}{2}\varepsilon}}.
\end{align*}
After some time, the first condition in the definition necessarily breaks. This allows to bound $\tau_3$, thanks to our choice of $\lambda$ and $\varepsilon_3$, which finally yields the bound
\begin{equation}\label{eq:tau3bound}
\tau_3 - \tau_2 \leq \frac{1}{\sigma_{\min} \varepsilon_2} \ln\left(\frac{\kappa\|\beta^*\|+\bigO{\varepsilon_2}}{\varepsilon_3-\bigO{\lambda^{2-\frac{5}{2}\varepsilon}}}\right)_+,
\end{equation}
where $\kappa= \sqrt{\frac{\sigma_{\max}}{\sigma_{\min}}}$ is the condition number of the matrix $X$.

We already know that the third condition in the definition of $\tau_3$ does not break at $\tau_3$, thanks to \cref{lemma:phase2bangle}. The remaining of the proof shows that neither the second nor fourth condition can break at $\tau_3$, given the \textit{small} amount of time given between $\tau_2$ and $\tau_3$.

\medskip

For the second condition, by following the ODE computed in the proof of \cref{lemma:phase2a} for any~$i,j\in I$,
\begin{align*}
\frac{\df \langle \w_i^t, \w_j^t \rangle}{\df t} & \geq -2 \|D^t\| (1-\langle \w_i^t, \w_j^t\rangle)\\
&\geq -2\bar{D}(1-\langle \w_i^t, \w_j^t\rangle),
\end{align*}
where $\bar{D}=\sqrt{\frac{1}{n}\sum_{k=1}^ny_k^2}\sqrt{\frac{1}{n}\sum_{k=1}^n x_k^2}+\bigO{\lambda}$, thanks to the bound on $\|D^t\|$ given in the proof of \cref{lemma:phase2bangle}. From there, a Gr\"onwall comparison directly yields for any $t\in[\tau_2,\tau_3]$
\begin{align*}
\langle \w_i^t, \w_j^t \rangle & \geq 1 - \left( 1-\langle \w_i^{\tau_2}, \w_j^{\tau_2} \rangle\right)e^{2\bar{D} (t-\tau_2)}\\
&\geq  1 - \bigO{\lambda^{1-\varepsilon}}\max\left(1,\left(\frac{\kappa\beta^*+\bigO{\varepsilon_2}}{\varepsilon_3-\bigO{\lambda^{2-\frac{5}{2}\varepsilon}}}\right)^{\frac{2\bar{D}}{\sigma_{\min}\varepsilon_2}}\right),
\end{align*}
where we used in the second inequality the bound on $\tau_3-\tau_2$ given by \cref{eq:tau3bound} and the bound on $\langle \w_i^{\tau_2}, \w_j^{\tau_2} \rangle$ given by \cref{lemma:phase2a}. Now, thanks to the choice of $\lambda$, $\varepsilon_2$ and $\varepsilon_3$, for a small enough $c_3>0$ depending only on the data
\begin{equation*}
\max\left(1,\left(\frac{\kappa\beta^*+\bigO{\varepsilon_2}}{\varepsilon_3-\bigO{\lambda^{2-\frac{5}{2}\varepsilon}}}\right)^{\frac{2\bar{D}}{\sigma_{\min}\varepsilon_2}}\right) =\bigO{\lambda^{-\varepsilon}}
\end{equation*}
and so
\begin{equation*}
\langle \w_i^t, \w_j^t \rangle \geq 1 - \bigO{\lambda^{1-2\varepsilon}}.
\end{equation*}
Thus, the second condition in the definition of $\tau_3$ does not break for a small enough $\tilde\lambda$ since $\varepsilon<\frac{1}{3}$.

\medskip

We now show that the fourth condition does not break either. For any $i\in \cN$, $a_i^t \leq 0$ and we can bound the (absolute) growth of $a_i^t$ as
\begin{align*}
-\frac{\df a_i^t}{\df t} & \leq  -a_i^t \|D_i^t\| \\
&\leq -\bar{D}a_i^t,
\end{align*}
as the bound on $\|D^t\|$ actually also holds for any $D_i^t$. From there, using \cref{lemma:phase2a}, for any $t\in[\tau_2,\tau_3]$:
\begin{align*}
-a_i^t & \leq -a_i^{\tau_2}e^{\bar{D}(t-\tau_2)}.
\end{align*}
Here again, this yields\footnote{We indeed just showed in the previous computations that $e^{2\bar{D} (t-\tau_2)}=\bigO{\lambda^{-\varepsilon}}$ for $t\in[\tau_2,\tau_3]$.}
\begin{equation*}
-a_i^t = -a_i^0\bigO{\lambda^{-\frac{\varepsilon}{2}}}
\end{equation*}
This gives the third point of \cref{lemma:phase2b} and also shows that the fourth condition in the definition of $\tau_3$ does not break for a small enough $\tilde\lambda=\Theta(1)$. Necessarily, the first condition in the definition in $\tau_3$ breaks at $\tau_3$, which ends the proof of \cref{lemma:phase2b}.
\end{proof}

\begin{lem}\label{lemma:boundA}
If $\lambda<\tilde\lambda$ for a small enough $\tilde\lambda=\Theta(1)$, for any $t\in[\tau_2,\tau_3]$
\begin{equation*}
\sum_{i\in \cI}(a_i^t)^2 =\bigO{1}.
\end{equation*}
\end{lem}
\begin{proof}
Define
\begin{equation}
\alpha\coloneqq\inf\limits_{\substack{w\in \bS_{d}\\\forall k \in[n], \langle w, x_k\rangle\geq 0}}\frac{1}{n}\sum_{k=1}^n \langle w,x_k\rangle\label{eq:alpha}.
\end{equation}
First note that $\alpha$ is positive. It is defined as the infimum of a continuous function on a non-empty (thanks to \cref{ass:noconvergence}) compact set, thus its minimum is reached for some $w^* \in \bS_{d}$ such that $\langle w^*, x_k\rangle\geq 0$ for any $k$. Since the vectors $(x_k)_k$ span the whole space $\R^d$, the scalar product $\langle w^*, x_k\rangle$ is zero for all $k$ if and only if $w^*=0$. This necessarily implies that $\alpha > 0$.

\medskip

Now, the decreasing of the loss over time yields for any $t\geq 0$:
\begin{align*}
\sum_{k=1}^n (h_{\theta^t}(x_k)-y_k)^2 &\leq \sum_{k=1}^n (h_{\theta^0}(x_k)-y_k)^2 \\
&= \sum_{k=1}^n y_k^2 + \bigO{\lambda^2}.
\end{align*}
The triangle inequality then implies
\begin{align*}
\left(\sqrt{\sum_{k=1}^n h_{\theta^t}(x_k)^2}-\sqrt{\sum_{k=1}^n y_k^2}\right)^2&=\bigO{1}.
\end{align*}
Moreover, note that by comparison between $1$ and $2$ norms:
\begin{align*}
\sqrt{n \sum_{k=1}^n h_{\theta^t}(x_k)^2} & \geq \sum_{k=1}^n h_{\theta^t}(x_k) \\
& = \sum_{k=1}^n \sum_{i\in \cI} a_i^t \langle w_i^t, x_k \rangle + \sum_{k=1}^n \sum_{i\in N} a_i^t \langle w_i^t, x_k \rangle_+ \\
&\geq n\alpha\sum_{i\in \cI} a_i^t \|w_i^t\|  - \bigO{\lambda^{2(1-\varepsilon)}}.
\end{align*}
The last inequality comes from the fact that for $t\in[\tau_2,\tau_3]$, for any $i\in \cI$ and $k\in [n]$, $\langle w_i^t, x_k \rangle \geq 0$; which allows to bound using $\alpha$ defined above.

Moreover, using the balancedness property (\cref{lemma:balanced}), $\|w_i^t\| \geq a_i^t - a_i^0$, which gives
\begin{align*}
\sum_{i\in \cI} a_i^t \|w_i^t\|  & \geq \sum_{i\in \cI} (a_i^t)^2 - \sum_{i\in \cI} a_i^t a_i^0 \\
&\geq \sum_{i\in \cI} (a_i^t)^2 - \sqrt{\sum_{i\in \cI} (a_i^0)^2}\sqrt{\sum_{i\in \cI} (a_i^t)^2} \\
&\geq \sum_{i\in \cI} (a_i^t)^2 - \lambda\sqrt{\sum_{i\in \cI} (a_i^t)^2}.
\end{align*}
Wrapping up the different inequalities, we finally have for  $t\in[\tau_2,\tau_3]$:
\begin{equation*}
\sum_{i\in \cI} (a_i^t)^2 - \lambda\sqrt{\sum_{i\in \cI} (a_i^t)^2} =\bigO{1}.
\end{equation*}
This then implies that $\sum_{i\in \cI} (a_i^t)^2 = \bigO{1}$.
\end{proof}

\begin{lem}\label{lemma:phase2bangle}
For small enough $\tilde{\lambda}=\Theta(1)$ and $\delta'=\Theta(1)$, if $\lambda<\tilde{\lambda}$, then for any $t\in [\tau_2,\tau_3]$,
\begin{gather*}
\forall i\in \cI, \forall k\in[n], \langle \w_i^t, x_k \rangle >\delta'\|x_k\|.
\end{gather*}\end{lem}
\begin{proof}
Let 
\begin{equation*}
\cK_{\delta'} \coloneqq \left\{ k\in[n] \mid \exists v\in \R^d, \langle v,\frac{x_k}{\|x_k\|} \rangle =\delta' \text{ and } \forall k'\in[n],\langle v,\frac{x_{k'}}{\|x_{k'}\|} \rangle \geq \delta' \right\}.
\end{equation*}

By continuity, the condition $\langle \w_i^t, \frac{x_k}{\|x_k\|} \rangle>\delta'$ would necessarily break for some $k\in \cK_{\delta'}$ first. It is thus sufficient to show \cref{lemma:phase2bangle} for any $k\in \cK_{\delta'}$. 
Since for all $i,j\in \cI$ and $t\in[\tau_2,\tau_3]$, $\langle \w_i^t, \w_j^t \rangle \geq 1-\lambda^{\varepsilon}$ by definition of $\tau_3$, we can use a decomposition similar to \cref{lemma:wixktau}:
\begin{gather*}
\w_j^t = \alpha_{ij}^t \w_i^t + v_{ij}^t \quad \text{where } \alpha_{ij}^t = 1+\bigO{\lambda^{\varepsilon}},\\
v_{ij}^t\perp \w_i^t \quad \text{ and } \quad \|v_{ij}^t\|=\bigO{\lambda^{\varepsilon}},
\end{gather*}
where we used the fact that $\|\w_i^t\|\geq 1-\lambda^{\varepsilon}$.
Using this decomposition, we can write
\begin{align*}
h_{\theta^t}(x_k)&=\sum_{j\in \cI}\alpha_{ij}^t (a_j^t)^2 \langle \w_i^t, x_k\rangle + \sum_{j\in \cI} (a_j^t)^2 \langle v_{ij}^t , x_k \rangle +\sum_{i\in \cN} a_i^t \langle w_i^t, x_k \rangle_+ \\
&\leq A\langle \w_i^t, x_k\rangle + \bigO{\lambda^{\varepsilon}},
\end{align*}
where $A=\Theta(1)$ is a constant bounding $\sum_{j\in \cI} (a_j^t)^2 $, thanks to \cref{lemma:boundA}. From there,
\begin{align*}
\langle D^t, x_k \rangle &=\frac{1}{n}\sum_{k'\neq k} \left( y_{k'} -h_{\theta^t}(x_{k'}) \right) \langle x_{k'}, x_{k}\rangle + \frac{1}{n} (y_k -h_{\theta^t}(x_k))\|x_k\|^2 \\
&\geq -\sqrt{\frac{1}{n}\sum_{k'=1}^n \left( y_{k'} -h_{\theta^t}(x_{k'})\right)^2}\sqrt{\frac{1}{n}\sum_{k'\neq k}  \langle x_{k'}, x_{k}\rangle^2} +\frac{y_k}{n}\|x_k\|^2 - \frac{A}{n}\langle \w_i^t,x \rangle - \bigO{\lambda^{\varepsilon}}.
\end{align*}
Using the non-increasing property of gradient flow, 
\begin{align*}
\frac{1}{n}\sum_{k'=1}^n \left( y_{k'} -h_{\theta^t}(x_{k'})\right)^2&\leq \frac{1}{n}\sum_{k'=1}^n \left( y_{k'} -h_{\theta^0}(x_{k'})\right)^2 \\
&=\frac{1}{n}\sum_{k'=1}^n y_{k}^2 + \bigO{\lambda^2}.
\end{align*}
Note that $ \cK_{\delta'} \subset \cK$, using the correspondence $w=v - \delta'\frac{x_k}{\|x_k\|}$ in the definition of both sets.\footnote{The case $w=\mathbf{0}$ is a particular case, which implies $v= \delta'\frac{x_k}{\|x_k\|}$ and so all the $x_k$ would be aligned, which contradicts the fact that $X$ is full rank.} Thanks to \cref{ass:cK}, the constant $\beta$ defined in \cref{eq:beta} is positive. 
This then implies for any $k\in \cK_{\delta'}$ and $t\in[\tau_2, \tau_3]$:
\begin{equation*}
\langle D^t, x_k \rangle \geq \beta - \frac{A}{n}\langle \w_i^t,x \rangle - \bigO{\lambda^{\varepsilon}}.
\end{equation*}
Again by our choice of $\lambda$, it also comes by monotonicity of the loss
\begin{align*}
\|D^t\| & \leq \sqrt{\frac{1}{n}\sum_{k'=1}^n y_{k'}^2}\sqrt{\frac{1}{n}\sum_{k'=1}^n \|x_{k'}\|^2}+\bigO{\lambda}=\bar{D} = \Theta(1).
\end{align*}
From these last two inequalities, for any $k\in \cK_{\delta'}$ and $t\in[\tau_2,\tau_3]$, as long as $\langle \w_i^t, x_{k'}\rangle \geq \delta' \|x_{k'}\|$ for any $k'$
\begin{align*}
\frac{\df \langle \w_i^t, x_k\rangle}{\df t} & =  \langle D^t, x_{k}\rangle - \langle \w_i^t, D^t\rangle \langle \w_i^t, x_k\rangle\\
& \geq \beta -  (\frac{A}{n}+\bar{D})\langle \w_i^t, x_k\rangle- \bigO{\lambda^{\varepsilon}}\\
& \geq  \beta - \bigO{\langle \w_i^t, x_k\rangle}- \bigO{\lambda^{\varepsilon}}.
\end{align*}
Now note that for small enough $\tilde\lambda=\Theta(1)$, we can choose $\delta'=\Omega(1)$ small enough, so that $\beta-(\frac{A}{n}+\bar{D})\delta'\|x_k\| -\bigO{\lambda^{\varepsilon}} >0$ and that $\langle \w_i^{\tau_2}, x_k\rangle>\delta'\|x_k\|$, thanks to \cref{lemma:phase2a}. This necessarily implies that $\langle \w_i^{t}, x_k\rangle> \delta'\|x_k\|$ for any $t\in[\tau_2, \tau_3]$, as $\langle \w_i^t,x_k \rangle$ would be increasing before crossing the $\langle \w_i^{t}, x_k\rangle = \delta'\|x_k\|$ threshold.
\end{proof}

\begin{lem}\label{lemma:phase2bnorm}
There exist $\varepsilon^*_2=\Theta(1)$ and $\tilde\varepsilon^*_2=\Theta(1)$ small enough such that for $\varepsilon_2$ satisfying the conditions of \cref{lemma:phase2a,lemma:phase2b} and any $t\in [\tau_2,\tau_3]$,
\begin{gather*}
\sum_{i\in \cI} (a_i^t)^2 \geq \varepsilon_2.
\end{gather*}
\end{lem}
\begin{proof}
Note that we can chose $\varepsilon^*_2$ small enough, so that
\begin{equation}\label{eq:varepsilon2}
\sum_{i\in \cI} (a_i^t)^2 \leq 2 \varepsilon_2 \implies \forall k\in [n], y_k > h_{\theta^t}(x_k).
\end{equation}
Moreover we have for any $t\in[\tau_2,\tau_3]$ the ODE
\begin{align*}
\frac{\df \sum_{i\in \cI} (a_i^t)^2}{\df t}  &= 2\sum_{i\in \cI}(a_i^t)^2 \langle \w_i^t, D^t\rangle \\
 &= \frac{2}{n}\sum_{i\in \cI}(a_i^t)^2 \left(\sum_{k=1}^n \langle \w_i^t, x_k\rangle \left(y_k-h_{\theta^t}(x_k)\right) \right).
\end{align*}
Thanks to \cref{lemma:phase2bangle}, all the terms $\langle \w_i^t, x_k\rangle$ in the sum are positive for $t\in[\tau_2,\tau_3]$. Thanks to \cref{eq:varepsilon2}, this sum is thus positive as soon as $\sum_{i\in I}(a_i^t)^2 \leq 2\varepsilon_2$. Thus during this phase, $\sum_{i\in \cI}(a_i^t)^2$ is increasing as soon as it is smaller than $2\varepsilon_2$. By continuity and since it is exactly $\varepsilon_2$ at the beginning of the phase, it is always at least $\varepsilon_2$ during this phase.
\end{proof}

\subsection{Phase 3}\label{app:phase3}
Define for this proof
\begin{gather*}
\beta^t \coloneqq \sum_{i\in \cI_t} a_i^t w_i^t
\text{ and } R_t \coloneqq \frac{1}{2} \sum_{i \in \cN_t} \|w_i^t\|^2 \\
\text{where } \cI_t \coloneqq  \left\{ i\in[m] \Big| \exists k \in [n], \langle w_i^t, x_k \rangle > 0 \right\}\\
\text{and } \cN_t \coloneqq \left\{ i\in[m] \Big| \exists k,k' \in [n] \text{ s.t. } \langle w_i^t, x_k \rangle > 0 \text{ and } \langle w_i^t, x_{k'} \rangle < 0 \right\}.
\end{gather*}

Note the $\beta^t$ differs from $\beta^t_{\cI}$ in the proof of \cref{lemma:phase2b}, as it does not only count the neurons in $\cI$, but the neurons in $\cI_t\supset I$.

\begin{lem}\label{lemma:hphase3}
For any $t\in[\tau_3, \tau_4)$ and $x\in\R^2$, 
\begin{equation*}
|h_{\theta^t}(x)-\langle\beta^*,x\rangle_+| = \bigO{(\varepsilon_3+\varepsilon_4)\|x\|}.
\end{equation*}
\end{lem}
\begin{proof}
Thanks to \cref{lemma:phase2b},
\begin{equation*}
|h_{\theta^{\tau_3}}(x)-\langle\beta^*,x\rangle_+| =\bigO{(\varepsilon_3  + \lambda^{\varepsilon})\|x\|}.
\end{equation*}
Moreover for $x\neq\mathbf{0}$,
\begin{align*}
\frac{1}{\|x\|}|h_{\theta^{t}}(x)-h_{\theta^{\tau_3}}(x)| & \leq \sum_{i=1}^{m} \|a_i^t w_i^t - a_i^{\tau_3}w_i^{\tau_3}\|_2 \\
& \leq \sum_{i=1}^{m} |a_i^{\tau_3}| \|w_i^t - w_i^{\tau_3}\|_2 + |a_i^t- a_i^{\tau_3}| \|w_i^{t}\|_2 \\
&\leq \sum_{i=1}^{m} |a_i^{\tau_3}| \|w_i^t - w_i^{\tau_3}\|_2 + |a_i^t- a_i^{\tau_3}| |a_i^t| \\
&\leq 2 \max\left(\sqrt{\sum_{i=1}^{m} (a_i^{\tau_3})^2},\sqrt{\sum_{i=1}^{m} (a_i^{t})^2} \right)\|\theta^t-\theta^{\tau_3}\|_2 \\
&\leq 2\left(\sqrt{\sum_{i=1}^{m} (a_i^{\tau_3})^2}+\varepsilon_4\right)\varepsilon_4.
\end{align*}
\cref{lemma:boundA} yields
\begin{equation*}
\sum_{i=1}^{m} (a_i^{\tau_3})^2 = \bigO{1},
\end{equation*}
so that 
\begin{align*}
|h_{\theta^{t}}(x_k)-h_{\theta^{\tau_3}}(x_k)| = \bigO{\varepsilon_3+\varepsilon_4+\lambda^{\varepsilon}}.
\end{align*}
This finally  yields \cref{lemma:hphase3} as $\lambda^{\varepsilon}=\bigO{\varepsilon_4}$.
\end{proof}

Note that for any $k\in[n]$, $\langle\beta^*,x_k\rangle_+=\langle\beta^*,x_K\rangle$ in our example.

Consider in the following some $\delta'_4>0$ such that \cref{ass:phase3} holds with $\delta_u\geq \delta'_4$ for any considered $u$. We then define 
\begin{gather}\label{eq:alpha4}\alpha_4\coloneqq \inf_{\substack{u\in A(\R^d)\\\exists k, u_k=1\\
\exists k', u_{k'}=1}}\inf_{\substack{k\in\cK, u_k\neq 0\\ \exists w\in A^{-1}(u)\cap \bS_d, |\langle w,\frac{x_k}{\|x_k\|} \rangle|\leq \delta'_4}}\langle \tD_{u}^{\beta^*}, x_k\rangle.
\end{gather}
\cref{ass:phase3} directly yields $\alpha_4>0$ since the infimum is over a finite set. We then consider a small enough $\tdelta_4 = \Theta(\min(\delta'_4,\alpha_4))$ and define
\begin{equation}\label{eq:c4}
c_4 \coloneqq \inf_{\substack{u\in A(\R^d)\\\exists k, u_k=1\\
\exists k', u_{k'}=1}}\inf_{w\in A^{-1}_{\tdelta_4}(u)}\langle \tD_{u}^{\beta^*},\frac{w}{\|w\|}\rangle.
\end{equation}
Here again, \cref{ass:phase3} yields $c_4>0$. 
\begin{lem}\label{lemma:Rt}
Under \cref{ass:noconvergence,ass:cK,ass:phase3}, if $\lambda,\varepsilon_3,\varepsilon_4$ and $\delta_4$ satisfy the conditions of \cref{lemma:phase3} for small enough constants $\tilde\lambda,\varepsilon_3^*,\varepsilon_4^*,\delta^*_4$, then  
\begin{equation*}
R_t \leq  \bigO{\lambda^{2(1-\varepsilon)} e^{-c_4(t-\tau_3)}} \text{ for any }t \in[\tau_3, \tau_4).
\end{equation*}
\end{lem}
Importantly in \cref{lemma:Rt}, the $\mathcal{O}$ hides a constant that only depends on the data (and not on $t$).
\begin{proof}
First not that for any $i\in \cN_t$, $a_i^t< 0$. Moreover, for $u_i^t = A(w_i^t)$ and $\tilde{D}_u^{\beta^*}$ defined by in \cref{eq:Dtilde}.
\begin{align*}
\frac{\df \|w_i^t\|^2}{\df t} & = 2a_i^t \langle D(w_i^t,\theta^t),w_i^t \rangle\\
&\leq 2a_i^t \langle \tD_{u_i^t}^{\beta^*},w_i^t \rangle - 2a_i^t \|w_i^t\| \|\tD_{u_i^t}^{\beta^*} - \tD_{u_i^t}(\theta^t)\|,
\end{align*}
where $\tD_{u_i^t}(\theta^t)=\frac{1}{n}\sum_{k=1}^n \iind{u_k=1}(y_k-h_{\theta^t}(x_k))x_k$.
Thanks to \cref{lemma:hphase3} for our choice $\lambda$,
\begin{align*}
\|\tD_{u_i^t}^{\beta^*} - \tD_{u_i^t}(\theta^t)\| =\bigO{\varepsilon_3+\varepsilon_4}
\end{align*}
The previous inequality becomes
\begin{align}\label{eq:decreaserate}
\frac{\df \|w_i^t\|^2}{\df t} & \leq 2a_i^t \langle \tD_{u_i^t}^{\beta^*},w_i^t \rangle- 2 a_i^t \|w_i^t\| \bigO{\varepsilon_3+\varepsilon_4}.
\end{align}
First, $i\in \cN_t$ implies that there are $k,k'$ such that $u_{i\ k}^t=-1$ and $u_{i\ k'}^t=1$. Moreover at time $t$, we have two possibilities:
\begin{itemize}
\item either $w_i^t\in A^{-1}_{\tdelta_4}(u_i^t)$;
\item or $w_i^t\not\in A^{-1}_{\tdelta_4}(u_i^t)$.
\end{itemize}
In the former, \cref{eq:decreaserate} yields for small enough $\varepsilon_3^*,\varepsilon_4^*=\Theta(1)$
\begin{align}
\frac{\df \|w_i^t\|^2}{\df t}&\leq 2c_4 a_i^t \|w_i^t\| - 2a_i^t\|w_i^t\| \bigO{\varepsilon_3+\varepsilon_4}\notag\\
&\leq -(2c_4-\bigO{\varepsilon_3+\varepsilon_4})\|w_i^t\|^2.\label{eq:decreasecase1}
\end{align}
In the latter case, by definition of $A^{-1}_{\tdelta_4}$, we either have $\langle \frac{w_i^t}{\|w_i^t\|}, \frac{x_k}{\|x_k\|} \rangle \leq \tdelta_4$ for all $k\in\cK$ or $-\langle \frac{w_i^t}{\|w_i^t\|}, \frac{x_k}{\|x_k\|} \rangle \leq \tdelta_4$ for all $k\in\cK$.
Assume in the following the first case and consider $k\in\cK$ such that $\langle \frac{w_i^t}{\|w_i^t\|}, \frac{x_k}{\|x_k\|} \rangle \leq \tdelta_4$. The symmetric case is dealt with similarly.

Denote\footnote{This can be defined, as $w_i^t\neq\mathbf{0}$ as long as $i\not\in N_t$.} in the following $\tw_i^t=\frac{w_i^t}{\|w_i^t\|}$. From there, it comes
\begin{align*}
\frac{\df \langle \tw_i^t , \frac{x_k}{\|x_k\|} \rangle}{\df t} & \in \frac{a_i^t}{\|w_i^t\|}\left(\langle \D_i^t, \frac{x_k}{\|x_k\|} \rangle - \langle \D_i^t,\tw_i^t \rangle \langle \tw_i^t, \frac{x_k}{\|x_k\|}\rangle\right). 
\end{align*}
Thanks to \cref{ass:phase3} and the definition of $\tdelta_4$
\begin{align*}
\frac{\df \langle \tw_i^t , \frac{x_k}{\|x_k\|} \rangle}{\df t} & \leq \frac{a_i^t}{\|w_i^t\|}\left(\alpha_4 - \bar{D}\tdelta_4 - \bigO{\varepsilon_3+\varepsilon_4}\right)\\
& \leq\frac{a_i^t}{\|w_i^t\|}\left( \alpha_4 - \bigO{\tdelta_4+\varepsilon_3+\varepsilon_4}\right)
\end{align*}
From there, we can choose $\varepsilon_3^*,\varepsilon_4^*$ and $\tilde{\delta}_4$ small enough (but still $\Theta(1)$)
so that $\langle \tw_i^t , \frac{x_k}{\|x_k\|} \rangle$ decreases at a rate $\frac{a_i^t}{\|w_i^t\|}\Theta(\alpha_4)$. This reasoning actually holds for any $k\in\argmax_{k\in\cK}\langle \tw_i^t , \frac{x_k}{\|x_k\|} \rangle$ as long as $i\in N_t$ and $w_i^t\not\in A^{-1}_{\delta_4}(u_i^t)$. 
As a consequence, if all $\langle \tw_i^t,\frac{x_k}{\|x_k\|} \rangle$ are smaller than $\delta_4$ with at least one of them positive for $k\in\cK$, then their maximal value (among $\cK$) decreases at the above rate, so that after a time at most $\Theta(1)$, all the values for $k\in\cK$ become non-negative for $k\in\cK$.
This would then also imply that all values $\langle \tw_i^t,\frac{x_k}{\|x_k\|} \rangle$ for $k\in[n]$ become non-negative. Indeed, if $\{k' \mid \langle w,x_{k'} \rangle>0\}$ is non-empty, we can, thanks to \cref{ass:noconvergence}, substract vectors of the form $\alpha_{k'}x_{k'}$ to $w$. This would decrease all values $\langle w,x_{k'} \rangle$. We can proceed this strategy until all the scalar products are non-positive and at least one of them is\footnote{The vector resulting from these substractions cannot be $\mathbf{0}$, since it would imply that $w$ is positively correlated with all data points.} $0$. This would then imply that the $k'$ for which it is $0$ are in $\cK$, i.e., $\{k' \mid \langle w,x_{k'} \rangle>0\}$ is either empty, or intersects $\cK$.

So if at some point $i\in \cN_t$ and $w_i^t\not\in A^{-1}_{\tdelta_4}(u_i^t)$, then after a time $\Delta t$ , $i\not\in \cN_{t+\Delta t}$ where $\Delta t$ satisfies
\begin{equation}\label{eq:Deltat}
\int_{t}^{t+\Delta t} \frac{|a_i^s|}{\|w_i^s\|}\df s =\Theta\left( 1\right).
\end{equation}
Moreover, during this whole time,
\begin{align}
\frac{\df \|w_i^t\|^2}{\df t} &= \bigO{\frac{|a_i^t|}{\|w_i^t\|} \|w_i^t\|^2}.\label{eq:decreasecase2}
\end{align}
Also, the above argument yields that the sets $\cN_t$ are non-increasing over time on $[\tau_3,\tau_4]$. Indeed, a neuron $i$ cannot enter in the set $\cN_t$, as it would either have $w_i^t\neq\mathbf{0}$ and $\langle \tw_i^t, x_k\rangle=0$ for some $k\in\cK$ at entrance, or $w_i^t=\mathbf{0}$ at entrance. 
In the first case, $w_i^t\not\in A^{-1}_{\tdelta_4}(u_i^t)$ and is thus immediately ejected out of $\cN_t$ (and is hence not entering). 
In the second case, note that $\frac{a_i^t}{\|w_i^t\|}$ is arbitrarily large at entrance in $\cN_t$. Thus, if $w_i^t$ enters such that $w_i^t\not\in A^{-1}_{\tdelta_4}(u_i^t)$, it is also immediately ejected out\footnote{A simpler argument is here to directly consider the non-scaled version $\frac{\df \langle w_i^t, x_k\rangle}{\df t}$.} of $\cN_t$. If instead $w_i^t$ enters such that $w_i^t\in A^{-1}_{\tdelta_4}(u_i^t)$, its norm decreases following \cref{eq:decreasecase1}, while it is actually $0$ at entrance. Hence in all the cases, it cannot enter $\cN_t$.

\medskip

For the neuron $i\in\cN_{\tau_3}$, we can now partition $[\tau_3,\tau_4)$ into three successive disjoint intervals such that
\begin{itemize}
\item \cref{eq:decreasecase1} holds in $[\tau_3,t_0]$;
\item the second interval is of the form $[t_0,t_0+\Delta t]$ with \cref{eq:Deltat,eq:decreasecase2};
\item $i\not\in \cN_t$ in $[t_0+\Delta t,\tau_4)$.
\end{itemize}
Using this partition, a Grönwall argument then yields for any $t\in[\tau_3,\tau_4)$ and $\varepsilon_3^*,\varepsilon_4^*$ small enough,
\begin{align*}
\|w_i^t\|^2 \iind{i\in \cN_t} &\leq  \iind{t\leq t_0+\Delta t}\|w_i^{\tau_3}\|^2 e^{-\frac{c_4}{2}(\min(t,t_0)-\tau_3)}\exp\left(\bigO{\int_{t_0}^{t_0+\Delta t} \frac{|a_i^s|}{\|w_i^s\|} \df s }\right)\\
&=\bigO{\|w_i^{\tau_3}\|^2 e^{-\frac{c_4}{2}(t-\tau_3)}}.
\end{align*}
From there, using the fact that $R_{\tau_3}\leq\lambda^{2(1-\varepsilon)}$, a simple summation yields
\begin{align*}
R_t & = \bigO{\lambda^{2(1-\varepsilon)} e^{-\frac{c_4}{2}(t-\tau_3)}} \text{ for any }t \in[\tau_3, \tau_4).
\end{align*}
\end{proof}

\begin{lem}\label{lemma:localPL}
Under \cref{ass:noconvergence,ass:cK,ass:phase3}, if $\lambda,\varepsilon_3,\varepsilon_4$ and $\delta_4$ satisfy the conditions of \cref{lemma:phase3}  for small enough constants $\tilde\lambda,\varepsilon_3^*, \varepsilon_4^*, \delta^*_4$, then for any $t\in (\tau_3,\tau_4)$:
\begin{equation*}
\forall g_t \in \partial L(\theta^t), \|g_t\|^2 \geq \sigma_{\min}(H)\|\beta^*\|\left(L(\theta^t)-L_{\beta^*}  \right)+\bigO{\lambda^{2(1-\varepsilon)}e^{-\frac{c_4}{4}(t-\tau_3)}},
\end{equation*}
where $L_{\beta^*} = \frac{1}{2n}\sum_{k=1}^n \left(\langle\beta^*, x_k\rangle - y_k\right)^2$ and $\sigma_{\min}(H)$ is the smallest eigenvalue of $H=\frac{1}{n}X^{\top}X$.
\end{lem}
\begin{proof}
For any $t\in (\tau_3, \tau_4)$, it comes by definition of $\tau_4$ that $\langle w_i^t, x_k \rangle >0$ for any $i\in \cI$ and $k\in [n]$. It then holds for any $g_t \in \partial L(\theta^t)$, when considering only the norm of its components corresponding to the derivatives along $w_i^t$ for $i\in\cI$:
\begin{align}\label{eq:g1}
\|g_t\|^2_2 \geq \sum_{i\in \cI} (a_i^t)^2 \|D^t\|^2.
\end{align} 
Note that by definition of $\beta^t$ and $R_t$:
\begin{align*}
h_{\theta^t}(x_k) & = \sum_{i=1}^m a_i^t \langle w_i^t, x_k\rangle_+ \\
& = \sum_{i=1}^m a_i^t (\langle w_i^t, x_k\rangle + \langle w_i^t, x_k\rangle_-) \\
& = \langle\beta^t,x_k\rangle +\sum_{i\in \cN_t}a_i^t\langle w_i^t, x_k\rangle_- \ .
\end{align*}
So that
\begin{align*}
|h_{\theta^t}(x_k) - \langle \beta^t, x_k\rangle| & \leq \sum_{i\in \cN_t} |a_i^t|\|w_i^t\|\|x_k\|\\
&\leq \sqrt{\sum_{i\in \cN_t} (a_i^t)^2}\sqrt{\sum_{i\in \cN_t} \|w_i^t\|^2}\|x_k\|\\
&\leq \sqrt{2R_t+\lambda^2}\sqrt{2R_t}\|x_k\|\\
&=\bigO{\lambda^{2(1-\varepsilon)}e^{-\frac{c_4}{4}(t-\tau_3)}}.
\end{align*}
The last inequality here comes from \cref{lemma:Rt}.
By definition of $D^t$ and $\beta^*$,
\begin{align*}
D^t & = -\frac{1}{n}\sum_{k=1}^n (h_{\theta^t}(x_k)-y_k)x_k\\
& = -\frac{1}{n}\sum_{k=1}^n (h_{\theta^t}(x_k)-\langle \beta^*,x_k\rangle)x_k\\
&  = -H(\beta^t-\beta^*) - \frac{1}{n}\sum_{k=1}^n (h_{\theta^t}(x_k)-\langle \beta^t,x_k\rangle)x_k\\
& =  -H(\beta^t-\beta^*)  +\bigO{\lambda^{2(1-\varepsilon)}e^{-\frac{c_4}{4}(t-\tau_3)}}.
\end{align*}
In the following, we note $\sigma_{\min}$ instead of $\sigma_{\min}(H)$ for simplicity. \cref{eq:g1} now becomes for $t\in(\tau_3, \tau_4)$
\begin{align}
\|g_t\|^2_2 & \geq \sum_{i\in \cI} (a_i^t)^2 \left(\sqrt{\sigma_{\min}}\|\beta^t-\beta^*\|_H - \bigO{\lambda^{2(1-\varepsilon)}e^{-\frac{c_4}{4}(t-\tau_3)}}\right)^2\notag\\
 &\geq \sigma_{\min}\sum_{i\in \cI} (a_i^t)^2 \|\beta^t-\beta^*\|_H^2 - \bigO{\lambda^{2(1-\varepsilon)}e^{-\frac{c_4}{4}(t-\tau_3)}}\notag\\
&\geq (\|\beta^*\|_2-\bigO{\varepsilon_3+\varepsilon_4})\sigma_{\min} \|\beta^t-\beta^*\|_H^2  -\bigO{\lambda^{2(1-\varepsilon)}e^{-\frac{c_4}{4}(t-\tau_3)}}\label{eq:PL1},
\end{align}
where the second inequality uses a bound $\|\beta^t-\beta^*\|_2=\bigO{1}$, which can be proved similarly to \cref{lemma:hphase3}. Note that we again have $\sum_{i\in \cI} (a_i^t)^2=\bigO{1}$ in this phase, thanks to \cref{lemma:boundA} and the definition of $\tau_4$.

On the other hand, we also have for $\tilde{R}_t(x) \coloneqq h_{\theta^t}(x)-\langle\beta^t,x\rangle$
\begin{align*}
L(\theta^t)-L_{\beta^*}& = \frac{1}{2n} \sum_{k=1}^n (h_{\theta^t}(x_k)-y_k)^2-(\langle \beta^*, x_k\rangle-y_k)^2\\
& = \frac{1}{2n} \sum_{k=1}^n (\langle\beta^t,x_k\rangle-y_k)^2-(\langle \beta^*, x_k\rangle-y_k)^2 + \frac{1}{2n} \sum_{k=1}^n (2(\langle\beta^t,x_k\rangle-y_k) \tilde{R}_t(x_k)+\tilde{R}_t(x_k)^2) \\
& = \frac{1}{n}(\beta^t-\beta^*)^{\top}X^{\top}(X\beta^*-Y) +  \frac{1}{2n}(\beta^t-\beta^*)^{\top}X^{\top}X(\beta^t-\beta^*) \\&\phantom{=}+ \frac{1}{2n} \sum_{k=1}^n \left(2(h_{\theta^t}(x_k)-y_k)\tilde{R}_t(x_k)-\tilde{R}_t(x_k)^2\right)\\
& \leq \frac{1}{2}\|\beta^t-\beta^*\|_H^2 + \frac{1}{n}\sum_{k=1}^n (h_{\theta^t}(x_k)-y_k)\tilde{R}_t(x_k).
\end{align*}
The last inequality comes from the definition of $\beta^*$ that implies $X^{\top}(X\beta^*-Y)=\mathbf{0}$. Recall that $|\tilde{R}_t(x_k)|=\bigO{\lambda^{2(1-\varepsilon)}e^{-\frac{c_4}{4}(t-\tau_3)}}$. 
Also, note that $\frac{1}{n}\sum_{k=1}^n (h_{\theta^t}(x_k)-y_k)=\bigO{1}$ by monotonicity of the loss. It yields with \cref{lemma:Rt}:
\begin{align}
L(\theta^t)-L_{\beta^*}
& \leq \frac{1}{2}\|\beta^t-\beta^*\|_H^2 + \bigO{\lambda^{2(1-\varepsilon)}e^{-\frac{c_4}{4}(t-\tau_3)}}.\label{eq:PL2}
\end{align}
Combining \cref{eq:PL1,eq:PL2}; it finally yields:
\begin{align*}
\|g_t\|^2_2 & \geq \sigma_{\min}\left(\|\beta^*\|-\bigO{\varepsilon_3+\varepsilon_4}\right) \|\beta^t-\beta^*\|_H^2  - \bigO{\lambda^{2(1-\varepsilon)}e^{-\frac{c_4}{4}(t-\tau_3)}}\\
& \geq 2\sigma_{\min}\left(\|\beta^*\|-\bigO{\varepsilon_3+\varepsilon_4}\right)\left(L(\theta^t)-L_{\beta^*}  \right)-\bigO{\lambda^{2(1-\varepsilon)}e^{-\frac{c_4}{4}(t-\tau_3)}}.
\end{align*}
\cref{lemma:localPL} then follows for small enough $\varepsilon^*_3$ and $\varepsilon^*_4$, so that the first term is larger than $ \sigma_{\min}\|\beta^*\|\left(L(\theta^t)-L_{\beta^*}  \right)$.
\end{proof}
\begin{lem}\label{lemma:lossdecrease}
Under \cref{ass:noconvergence,ass:cK,ass:phase3}, if $\lambda,\varepsilon_3,\varepsilon_4$ and $\delta_4$ satisfy the conditions of \cref{lemma:phase3}  for small enough constants $\tilde\lambda,\varepsilon_3^*, \varepsilon_4^*, \delta^*_4$, then for any $t\in[\tau_3, \tau_4)$ and $a\coloneqq\min\left(\frac{\sigma_{\min}(H)\|\beta^*\|}{2},\frac{c_4}{4}\right)$: 
\begin{equation*}
L(\theta^t) - L_{\beta^*} = \bigO{\varepsilon_3^2e^{-a(t-\tau_3)}}.
\end{equation*}
\end{lem}
\begin{proof}
By definition of the gradient flow, there is some $g_t \in \partial L(\theta^t)$ such that $\frac{\df \theta^t}{\df t} = - g_t$ a.e. It then comes from \cref{lemma:localPL} that a.e.\footnote{The fact that the chain rule can be applied to $\frac{\df L(\theta^t) }{\df t}$ here is not straightforward, but is possible a.e. \citep[see e.g.,][Lemma 2, Corollary 1 and Proposition 2]{bolte2021conservative}}
\begin{align*}
\frac{\df (L(\theta^t)-L_{\beta^*})}{\df t} & = -\|g_t\|^2 \\
& \leq -\sigma_{\min}\|\beta^*\|( L(\theta^t)-L_{\beta^*}) + \bigO{\lambda^{2(1-\varepsilon)}e^{-\frac{c_4}{4}(t-\tau_3)}}\\
& \leq -\sigma_{\min}\|\beta^*\|( L(\theta^t)-L_{\beta^*}) + \bigO{\lambda^{2(1-\varepsilon)}e^{-a(t-\tau_3)}},
\end{align*}
where again we write $\sigma_{\min}$ for $\sigma_{\min}(H)$. 
Solutions of the ODE $f'(t) = -c f(t) + be^{-a t}$ are of the form
\begin{equation*}
f(t) = \frac{b}{c-a}(e^{-at}-e^{-ct}) + f(0)e^{-ct} \quad \text{ if } a\neq c.
\end{equation*}
Since $a<\sigma_{\min}\|\beta^*\|$, a Gr\"onwall comparison then yields for any $t\in[\tau_3, \tau_4)$,
\begin{equation*}
L(\theta^t)-L_{\beta^*} \leq \left(L(\theta^{\tau_3})-L_{\beta^*}\right)e^{-\sigma_{\min}\|\beta^*\|(t-\tau_3)}+ \bigO{\lambda^{2(1-\varepsilon)}e^{-a(t-\tau_3)}}.
\end{equation*}

Moreover, it comes from \cref{eq:PL2} that
\begin{align*}
L(\theta^{\tau_3})-L_{\beta^*} & \leq \frac{1}{2}\|\beta^{\tau_3}-\beta^*\|_H^2+\bigO{\lambda^{2(1-\varepsilon)}}\\
& \leq \frac{\sigma_{\max}}{2}\|\beta^{\tau_3}-\beta^*\|_2^2+\bigO{\lambda^{2(1-\varepsilon)}}\\
&\leq \frac{\sigma_{\max}}{2}\varepsilon_3^2+\bigO{\lambda^{2(1-\varepsilon)}},
\end{align*}
where the last inequality comes from \cref{lemma:phase2b}. This allows to conclude as $\lambda^{1-\varepsilon}=\bigO{\varepsilon_3}$.
\end{proof}

\begin{lem}\label{lemma:Dt}
Under \cref{ass:noconvergence,ass:cK,ass:phase3}, if $\lambda,\varepsilon_3,\varepsilon_4$ and $\delta_4$ satisfy the conditions of \cref{lemma:phase3}  for small enough constants $\tilde\lambda,\varepsilon_3^*,\varepsilon_4^*,\delta^*_4$, then
\begin{equation*}
\int_{\tau_3}^{\tau_4} \|D^t\|\df t =\bigO{\varepsilon_3}.
\end{equation*}
\end{lem}
\begin{proof}
We already proved in the proof of \cref{lemma:localPL} for any $t\in[\tau_3,\tau_4)$:
\begin{align*}
\|D_t\|& \leq \sqrt{\sigma_{\max}}\|\beta^t-\beta^*\|_H+\bigO{\lambda^{2(1-\varepsilon)} e^{-\frac{c_4}{4} (t-\tau_3)}}.
\end{align*}
Moreover, we also showed
\begin{align*}
\|\beta^t - \beta^*\|_H^2 & = 2(L(\theta^t)-L_{\beta^*})+\frac{1}{n}\sum_{k=1}^n\left( (y_k-h_{\theta^t}(x_k))\tilde{R}_t(x_k) + \tilde{R}_t(x_k)^2\right)\\
\|\beta^t - \beta^*\|_H & \leq \sqrt{2(L(\theta^t)-L_{\beta^*})} + \bigO{\lambda^{(1-\varepsilon)}e^{-\frac{c_4}{4}(t-\tau_3)}}.
\end{align*}
Thanks to \cref{lemma:lossdecrease}, it comes for our choice of $\lambda$ and $\varepsilon_3$:
\begin{equation*}
\|D_t\|\leq \bigO{\varepsilon_3e^{-\frac{a}{2}(t-\tau_3)}},
\end{equation*}
where the constants hidden in $\mathcal{O}$ do not depend on $t$ but only the dataset. 
Integrating the exponential then concludes the proof.
\end{proof}

\begin{lem}\label{lemma:boundeddyn}
Under \cref{ass:noconvergence,ass:cK,ass:phase3},  if $\lambda,\varepsilon_3,\varepsilon_4$ and $\delta_4$ satisfy the conditions of \cref{lemma:phase3}  for small enough constants $\tilde\lambda,\varepsilon_3^*,\varepsilon_4^*,\delta^*_4$, then for any $i\in \cI$ and $t\in [\tau_3,\tau_4)$:
\begin{itemize}
\item $|a_i^t-a_i^{\tau_3}|\leq \left(e^{\bigO{\varepsilon_3}}-1\right) a_i^{\tau_3}$;
\item $\|w_i^t-w_i^{\tau_3}\|\leq \left(e^{\bigO{\varepsilon_3}}+\bigO{\varepsilon_3}-1\right) a_i^{\tau_3} $;
\item $\|\w_i^t-\w_i^{\tau_3}\|\leq \bigO{\varepsilon_3} $.
\end{itemize}
Moreover, for any $i\in \cN$ and $t\in [\tau_3,\tau_4)$,
\begin{itemize}
\item $|a_i^t-a_i^{\tau_3}|\leq \bigO{1}a_i^{\tau_3}$;
\item $\|w_i^t-w_i^{\tau_3}\|\leq \bigO{1} a_i^{\tau_3}$,
\end{itemize}
where again the constants hidden in $\mathcal{O}$ neither depend on $t$ nor $i$, but only on the dataset.
\end{lem}
\begin{proof}
For any $i\in \cI$, the neuron $i$ is activated along all $x_k$ for all $t\in [\tau_3,\tau_4)$ by definition of $\tau_4$. As a consequence, we can bound the derivative of $a_i^t>0$ as:
\begin{align*}
\left|\frac{\df a_i^t}{\df t}\right| & = |\langle w_i^t, D^t\rangle |\\
& \leq a_i^t \|D^t\|.
\end{align*}
Gr\"onwall inequality along with \cref{lemma:Dt} then yield the first item of \cref{lemma:boundeddyn}. We also have the following bound:
\begin{align*}
\left\| \frac{\df \w_i^t}{\df t}\right\| & = \left\| D^t - \langle \w_i^t,D^t\rangle\w_i^t \right\|\\
& \leq \|D^t\|.
\end{align*}
This yields the third point of \cref{lemma:boundeddyn}. Moreover note that 
\begin{align*}
\|w_i^t - w_i^{\tau_3}\| & \leq |a_i^t-a_i^{\tau_3}| + a_i^{\tau_3} \|\w_i^t -\w_i^{\tau_3}\|.
\end{align*}
The second point then follows from the first and third one.

\medskip

For $i\in \cN$, the neuron $i$ might not be activated along all $x_k$. However, we can use the same arguments as in the proof of \cref{lemma:Rt}, so that there are time $t_i$ (potentially $\infty$) and $\Delta t_i$, such that $|a_i^t|$ decreases on $[\tau_3, t_i]$ and $i$ is activated along all (or no) $x_k$ on $[t_i+\Delta t_i, \tau_4)$ with $\Delta t_i = \bigO{1}$. From the first part, it holds for any $t\in[\tau_3, t_i]$
\begin{gather*}
|a_i^{t}-a_i^{\tau_3}| \leq |a_i^{\tau_3}|\\
\|w_i^{t}-w_i^{\tau_3}\| \leq \|w_i^{t}\| + \|w_i^{\tau_3}\| \leq 2|a_i^{\tau_3}|.
\end{gather*} 
And a Gr\"onwall argument also yields for any $t\in[\tau_3, t_i+\Delta t_i]$, since $\Delta t_i =\bigO{1}$:
\begin{gather*}
|a_i^{t}-a_i^{\tau_3}| =\bigO{|a_i^{\tau_3}|}\\
\|w_i^{t}-w_i^{\tau_3}\| =\bigO{|a_i^{\tau_3}|}.
\end{gather*} 

Also, we can bound the difference in the third part similarly to the case $i\in \cI$, \ie for any $t \in [t_i+\Delta t_i, \tau_4)$:
\begin{gather*}
|a_i^{t}-a_i^{t_i+\Delta t_i}| \leq a_i^{t_i+\Delta t_i}e^{\bigO{\varepsilon_3}}\\
\|w_i^t-w_i^{t_i+\Delta t_i}\|\leq \left( e^{\bigO{\varepsilon_3}}+\bigO{\varepsilon_3}\right) a_i^{t_i+\Delta t_i}.
\end{gather*}
Combining these different inequalities then yields the last two points of \cref{lemma:boundeddyn}, using the fact that $\varepsilon_3=\bigO{1}$.
\end{proof}

\begin{coro}\label{coro:tau4}
Under \cref{ass:phase3}, if $\lambda,\varepsilon_3,\varepsilon_4$ and $\delta_4$ satisfy the conditions of \cref{lemma:phase3} for small enough constants $\tilde\lambda,\varepsilon_3^*,\varepsilon_4^*,\delta^*_4$, then $\tau_4=\infty$.
\end{coro}

\begin{proof}
First note that thanks to \cref{lemma:phase2b} and \cref{lemma:boundeddyn}, we have for any $t\in[\tau_3, \tau_4)$, $i\in \cI$ and $k\in[n]$:
\begin{align*}
\langle \w_i^t, x_k\rangle & \geq \langle \w_i^{\tau_3}, x_k\rangle - \|\w_i^{t}-\w_i^{\tau_3}\|\ \|x_k\|\\
& \geq \Omega(1) - \bigO{\varepsilon_3}\\
& > \delta_4 \|x_k\|,
\end{align*}
for a small enough choice of $\varepsilon^*_3$ and $\delta_4^*$. 
This implies that the second condition in the definition of $\tau_4$ does not break first. Moreover, thanks to \cref{lemma:boundeddyn},
\begin{align*}
\|\theta^t - \theta^{\tau_3}\|_2^2 & \leq \sum_{i\in I} (a_i^{t}-a_i^{\tau_3})^2 + \|w_i^t-w_i^{\tau_3}\|^2 +  \sum_{i\in \cN} (a_i^{t}-a_i^{\tau_3})^2 + \|w_i^t-w_i^{\tau_3}\|^2\\
& = \left( e^{\bigO{\varepsilon_3}}+\bigO{\varepsilon_3}-1\right)^2 \sum_{i \in \cI} (a_i^{\tau_3})^2 + \bigO{1} \sum_{i \in \cN} (a_i^{\tau_3})^2 \\
& \leq  \left( e^{\bigO{\varepsilon_3}}+\bigO{\varepsilon_3}-1\right)^2\bigO{1}+ \bigO{\lambda^{2(1-\varepsilon)}} ,\end{align*}
where the last inequality comes from the bounds on the sum of $(a_i^{\tau_3})^2$ from \cref{lemma:phase2b,lemma:phase2bnorm}. Now for any $\varepsilon_4^*=\Theta(1)$ and small enough $\varepsilon^*_3,\tilde\lambda=\Theta(1)$ (depending on $\varepsilon_4^*$), the previous inequality leads to
\begin{align*}
\|\theta^t - \theta^{\tau_3}\|_2^2 & < \varepsilon_4^*.
\end{align*}
This implies that the first condition in the definition of $\tau_4$ does not break first. As a consequence, neither of the conditions in the definition of $\tau_4$ break in finite time, \ie $\tau_4=\infty$.
\end{proof}

\paragraph{Proof of \cref{lemma:phase3}.} First note that we can choose constants $\varepsilon,\tilde\lambda,\varepsilon_2,\varepsilon_3,\varepsilon_4,\delta_4=\Theta(1)$, such that all the lemmas of \cref{app:proofnoconvergence} simultaneously hold for any $\lambda\leq\tilde\lambda$. An easy way to verify so is going backward after fixing $\varepsilon$, i.e., first fix $\varepsilon_4,\delta_4$, then fixing $\varepsilon_3$, then $\varepsilon_2$ and finally $\tilde\lambda$.
Thanks to \cref{lemma:Dt,lemma:boundeddyn,coro:tau4}, the parameters $\theta^t$ and their variations are both bounded on $[\tau_3, \infty)$. As a consequence, $\theta^t$ does have a limit:
\begin{equation*}
\lim_{t\to\infty} \theta^t = \theta^\infty_{\lambda}.
\end{equation*}
Moreover, \cref{lemma:Rt} yields that
\begin{equation*}
\lim_{t\to\infty} R_t = 0.
\end{equation*}
This directly implies the second point of \cref{lemma:phase3}. Now denote $$\beta^\infty_{\lambda} \coloneqq \sum_{\substack{i\in[m]\\\forall k\in[m], \langle  w_{\lambda, i}^\infty, x_k\rangle \geq 0}} a_{\lambda, i}^\infty w_{\lambda, i}^\infty.$$
The second point of \cref{lemma:phase3} that we just proved implies that
\begin{equation*}
h_{\theta^\infty}(x_k) = \langle \beta_{\lambda}^\infty, x \rangle \quad \text{for any }x\in\conex,
\end{equation*}
i.e. $\theta^\infty_{\lambda}$ behaves as a linear regression of parameter $\beta^\infty_{\lambda}$ on the (cone generated by the) convex hull of the training data. As a consequence,
\begin{equation*}
L(\theta^\infty_{\lambda}) = L_{\beta^\infty_{\lambda}} \coloneqq  \frac{1}{2n}\sum_{k=1}^n \left(\langle\beta^\infty_{\lambda}, x_k\rangle - y_k\right)^2.
\end{equation*}
Also, \cref{lemma:lossdecrease} implies that $L(\theta^\infty_{\lambda}) = L_{\beta^\infty_{\lambda}}\leq L_{\beta^*}$. However, $\beta^*$ is the unique minimiser of the least square loss, among the set of linear regression parameters. This implies that $\beta^\infty_{\lambda}=\beta^*$, which concludes the proof of both \cref{lemma:phase3,thm:noconvergence}.

The last point of \cref{thm:noconvergencegeneral} is obtained by using \cref{lemma:hphase3} and noticing that when $\lambda\to0$, we can also choose the constants $\varepsilon_3$ and $\varepsilon_4$ such that they converge to $0$ when $\lambda\to 0$.

\section{Further Dicussion on \citep{glasgow2024sgd}}\label{app:glasgow}

In this section, we discuss in more details how the early alignment phenomenon is related to \citet{glasgow2024sgd}'s work. \citet{glasgow2024sgd} studies the convergence of SGD for XOR-type data with two-layer ReLU network towards four vectors, resulting in a $0$ test loss. 

We argue that the XOR setting under consideration falls within the scope of our framework. Specifically, the associated population loss exhibits exactly four extremal vectors corresponding to these four vectors reached at convergence (see details below). From this, the initial phase of the dynamics described by \citet{glasgow2024sgd} aligns with the early alignment phase analyzed in our work---\cref{thm:alignment} has yet to be applied to gradient flow over the population loss, which requires extending it to the infinite data setting.

\medskip

The primary additional technical challenges addressed in \citet{glasgow2024sgd} involve establishing analogous results for stochastic gradient descent (SGD) rather than gradient flow, and under a finite data regime. Their analysis of early alignment, and the broader training dynamics, for SGD in the XOR setting is achieved by bounding the deviation between SGD and gradient flow on the population loss via concentration inequalities. Controlling the growth of the neurons from the early alignment phase also remains challenging, although simpler to control.

\subsection{Computation of Extremal Vectors in \citealt{glasgow2024sgd} Setting.} In this section, we justify that the population loss only counts four extremal vectors in the XOR setting. For that, we consider the \textit{population loss} with Gaussian input. 

While \citet{glasgow2024sgd} consider data distributed uniformly over the hypercube, their proof of the first alignment phase relies on an approximation of the uniform distribution on the hypercube by a standard Gaussian in high dimension. In that effect, we directly assume here that the dataset is given by
\begin{gather*}
x_k \sim \cN(0,   \mathrm{I}_{d}) \\
  y(x_k) = - \sign(e_1^\top x_k)\sign(e_2^\top x_k)
\end{gather*}
and consider the limit of infinite data. We thus denote the (population) loss of the model as 
\begin{align*}
L(\theta) \coloneqq \E_{x,y}\left[\ell(h_{\theta}(x),y(x))\right],
\end{align*}
where $\ell(\hat{y},y)=\ln(1+e^{-\hat{y}y})$ is the logistic loss. In this population loss setting, the function $G(w)$ can still be defined, and its gradient is given by
\begin{align*}
D(w, \mathbf{0}) = \E_{x,y}[y x \iind{w^\top x \geq 0}].
\end{align*}
In that infinite data setting, the definition of extremal vector also extends to $D\neq \mathbf{0}$ such that both hold \citep[see][]{boursier2024simplicity}
\begin{gather*}
1. \quad D = D(w, \mathbf{0})\\
2. \quad \frac{D}{\|D\|} = \pm \frac{w}{\|w\|}.
\end{gather*}
Using computations (see \cref{app:glasgow1,app:glasgow2} below) similar to Lemma D.4 \citep{glasgow2024sgd}, we can then show that for any $w\in\R^d\setminus\{0\}$:
\begin{equation}\label{eq:glasgow1}
\begin{aligned}
\sign(w_2) = -\sign(\langle D(w, \mathbf{0}),e_1 \rangle),\\
\sign(w_1) = -\sign(\langle D(w, \mathbf{0}),e_2 \rangle)
\end{aligned}
\end{equation}
and\footnote{In this section, $w_i$ is the $i$-th coordinate of $w$ and $w_{i:j}$ is the projection of $w$ onto $\Span(e_i, \ldots, e_j)$.}
\begin{equation}\label{eq:glasgow2}
\sign(\langle D(w, \mathbf{0}),w_{3:d}\rangle) = \sign(w_1 w_2)\iind{w_{3:d}\neq 0}.
\end{equation}
From there, assume that $D(w, \mathbf{0})$ is an extremal vector. If $w_1=0$, then the above imply that $D(w, \mathbf{0}) \perp e_j$ for any $j\geq 2$, so that it is not an extremal vector. In consequence, we necessarily have $w_1\neq 0$ and $w_2 \neq 0$ for similar reasons.

From there assume w.l.o.g. that $w_1 w_2>0$ and also assume that $w_{3:d}\neq 0$. We then have by \cref{eq:glasgow1,eq:glasgow2}:
\begin{gather*}
\langle D(w, \mathbf{0}), w_{1:2})<0, \\
\langle D(w, \mathbf{0}), w_{3:d})>0.
\end{gather*}
However, these two inequalities contradict the fact that $D(w, \mathbf{0}) \propto \pm w$, i.e., that it is an extremal vector. Thus, any extremal vector $D(w, \mathbf{0})$ satisfies $w_{3:d}=0$. Necessarily, it must be of the form $D(w, \mathbf{0}) = \alpha_1 e_1 + \alpha_2 e_2$. We can now be more precise and show that (see \cref{app:glasgow3} below) for any $w$ such that $w_{3:d}=0$,
\begin{equation}\label{eq:glasgow3}
\begin{gathered}
\text{if } \quad |w_1|> |w_2| \quad \text{then}\quad |\langle D(w, \mathbf{0}), e_2 \rangle|> |\langle D(w, \mathbf{0}), e_1 \rangle|, \\
\text{if } \quad |w_1|< |w_2| \quad \text{then}\quad |\langle D(w, \mathbf{0}), e_2 \rangle|< |\langle D(w, \mathbf{0}), e_1 \rangle|.
\end{gathered}
\end{equation}
In consequence, any extremal vector must be such that $|w_1|=|w_2|$, i.e., of the form $D(w, \mathbf{0}) = \alpha( e_1 \pm e_2)$ for some $\alpha\in\R$. In consequence, there exist at most $4$ extremal vectors and it is easy to check that for a good choice of $\alpha$ (both $>0$ and $<0$), this indeed yields extremal vectors.

We have thus shown that in the XOR setting with population logistic loss and Gaussian data, there are only $4$ extremal vectors, which are proportional to the vectors $e_1+e_2$, $e_1-e_2$, $-e_1+e_2$, $e_1+e_2$. 

\subsection{Proof of \cref{eq:glasgow1}.}\label{app:glasgow1}

Similarly to \citet{glasgow2024sgd}, we note $x=z+\xi$ where $z$ is the projection of $x$ on the first two coordinates and $\xi$ on the last $d-2$ coordinates. We also note $y(z) = - \sign(z_1 z_2)$. Denoting by $\sw_1(x) = (-x_1, x_2, \ldots, x_d)$ the flip operator on the first coordinate, we have using the symmetries of the distribution:
\begin{align}
\langle D(w, \mathbf{0}), e_2 \rangle & = \E[  y(z) \iind{w^\top x \geq 0} x_2 ]\notag\\
& = \frac{1}{2}\E[ (\iind{w^\top x \geq 0}-\iind{w^\top \sw_1(x) \geq 0})  y(z)  x_2]\notag\\
& = -\frac{1}{2}\E[(\iind{w^\top x \geq 0}-\iind{w^\top \sw_1(x) \geq 0}) \sign(x_1)|x_2|].\label{eq:glasgowaux}
\end{align}
Moreover, note that if $w_1\geq 0$, we necessarily have $(\iind{w^\top x \geq 0}-\iind{w^\top \sw_1(x) \geq 0})\,  \sign(\langle x, e_1\rangle)\geq 0$, so that
\begin{align*}
w_1\geq 0 \implies \langle D(w, \mathbf{0}), e_2 \rangle \leq 0.
\end{align*}
Moreover if $w_1> 0$, the above expectation is non-zero since there is at least a non-zero measure subset of $\R^d$ for which $(\iind{w^\top x \geq 0}-\iind{w^\top \sw_1(x) \geq 0})  |x_2|> 0$. 
Symmetric arguments allow to derive \cref{eq:glasgow1}.

\subsection{Proof of \cref{eq:glasgow2}.}\label{app:glasgow2}
Similar computations as above yield
\begin{align*}
\langle D(w, \mathbf{0}), w_{3:d} \rangle & = \E[ \iind{w^\top x \geq 0}  y(z) \langle \xi, w_{3:d}\rangle] \\
& = \frac{1}{2} \E[ ( \iind{w^\top (z+\xi) \geq 0} - \iind{w^\top (z-\xi) \geq 0})  y(z) \langle \xi, w_{3:d}\rangle]\\
& =  \frac{1}{2} \E[ y(z) \iind{|\xi^\top w_{3:d}| \geq |z^\top w|}   |\langle \xi, w_{3:d}\rangle|]\\
& = \frac{1}{2} \E_\xi \E_z[ \iind{y(z)=1} \iind{|\xi^\top w_{3:d}| \geq |z^\top w|}   |\langle \xi, w_{3:d}\rangle|]\\
& \phantom{=} -\frac{1}{2} \E_\xi \E_z[ \iind{y(z)=-1} \iind{|\xi^\top w_{3:d}| \geq |z^\top w|}   |\langle \xi, w_{3:d}\rangle|]
\end{align*}
Assume w.l.o.g. that $w_1, w_2\geq 0$. For a fixed $\xi$ and since $y(z)=-\sign(z_1 z_2)$, the region $\{\iind{y(z)=-1} \iind{|\xi^\top w_{3:d}| \geq |z^\top w|}\geq 0\}$ is smaller (w.r.t. the Gaussian measure) than the region $\{\iind{y(z)=1} \iind{|\xi^\top w_{3:d}| \geq |z^\top w|}\geq 0\}$. It is indeed a consequence of the following observation: if $ y(z)=-1$ then $|z^\top w| \geq |\sw_1(z)^\top w|$, i.e.,
\begin{gather*}
\text{if }\quad  \iind{y(z)=-1} \iind{|\xi^\top w_{3:d}| \geq |z^\top w|}=1 \\
\text{then }\quad  \iind{y(\sw_1(z))=1} \iind{|\xi^\top w_{3:d}| \geq |\sw_1(z)^\top w|}=1,
\end{gather*}
and $\sw_1(z)$ also follows a standard Gaussian distribution.

As a consequence, if $w_1, w_2\geq 0$, the previous inequality implies that $\langle D(w, \mathbf{0}), w_{3:d} \rangle \geq 0$. Moreover, if both $w_1, w_2> 0$ and $w_{3:d}\neq 0$, the above difference becomes positive, so that $\langle D(w, \mathbf{0}), w_{3:d} \rangle > 0$. More generally, we have shown by symmetric argument that
\begin{equation*}
\sign(D(w, \mathbf{0}), w_{3:d} \rangle =  \sign(w_1 w_2) \iind{w_{3:d}\neq 0}
\end{equation*}

\subsection{Proof of \cref{eq:glasgow3}.}\label{app:glasgow3}
Assume that $w_{3:d}=0$, \cref{eq:glasgowaux} then yields
\begin{align*}
\langle D(w, \mathbf{0}), e_2 \rangle  &=  -\frac{1}{2}\E_x[(\iind{w^\top x \geq 0}-\iind{w^\top \sw_1(x) \geq 0}) \sign(x_1)|x_2|] \\
& = -\frac{\sign(w_1)}{2}\E_x[\iind{|w_1 x_1|\geq |w_{2} x_{2}|}\sign(w_1 x_1) \sign(x_1)|x_2|]\\
& =-\frac{\sign(w_1)}{2}\E_x[\iind{|w_1 x_1|\geq |w_{2} x_{2}|} |x_2|].
\end{align*}
And similarly $\langle D(w, \mathbf{0}), e_1 \rangle = -\frac{\sign(w_2)}{2}\E_x[\iind{|w_2 x_2|\geq |w_{1} x_{1}|} |x_1|]$.
Now note that the transformation $\tilde{x} = (x_2, x_1, x_3, \ldots, x_d)$ does not change the data distribution and thus yields the following inequalities
\begin{align*}
\E_x[\iind{|w_1 x_1|\geq |w_{2} x_{2}|} |x_2|]&  = \E_x[\iind{|w_1 x_2|\geq |w_{2} x_{1}|} |x_1|]\\
& \begin{cases}
\leq \E_x[\iind{|w_2 x_2|\geq |w_{1} x_{1}|} |x_1|] \quad \text{if } |w_2|\geq |w_1|\\
\geq \E_x[\iind{|w_2 x_2|\geq |w_{1} x_{1}|} |x_1|] \quad \text{if } |w_2|\leq |w_1|.
\end{cases}
\end{align*}
Moreover, the inequalities are strict if $|w_2|> |w_1|$ or $|w_2|< |w_1|$, so that 
\begin{gather*}
|\langle D(w, \mathbf{0}), e_2 \rangle| > |\langle D(w, \mathbf{0}), e_1 \rangle|  \quad \text{if } |w_1|>|w_2|,\\
|\langle D(w, \mathbf{0}), e_2 \rangle| < |\langle D(w, \mathbf{0}), e_1 \rangle|  \quad \text{if } |w_1|<|w_2|.
\end{gather*}
\end{document}